\documentclass{article}

\usepackage[main,final]{neurips_2026}
\usepackage{microtype}
\usepackage{graphicx}
\usepackage{subcaption}
\usepackage{bm}
\usepackage{mathtools}

\usepackage[utf8]{inputenc} 
\usepackage[T1]{fontenc}    
\usepackage{hyperref}       
\usepackage{url}            
\usepackage{booktabs}       
\usepackage{amsfonts}       
\usepackage{nicefrac}       
\usepackage{microtype}      
\usepackage{xcolor}         

\usepackage{amsmath}
\usepackage{amssymb}
\usepackage{amsthm}
\usepackage{physics}
\usepackage{etoc}

\DeclareMathOperator*{\argmin}{arg\,min}

\usepackage[capitalize,noabbrev]{cleveref}

\theoremstyle{plain}
\newtheorem{theorem}{Theorem}[section]

\newtheorem{lemma}[theorem]{Lemma}
\newtheorem{corollary}[theorem]{Corollary}
\theoremstyle{definition}
\newtheorem{definition}[theorem]{Definition}

\theoremstyle{remark}
\newtheorem{remark}[theorem]{Remark}

\title{Sharpness of Minima in Deep Matrix Factorization}

\author{%
  \textbf{Anıl Kamber}$^{1}$ \quad
  \textbf{Rahul Parhi}$^{1}$ \\[4pt]
  $^{1}$Department of Electrical and Computer Engineering\\
  University of California, San Diego \\
  \texttt{akamber@ucsd.edu},\;
  \texttt{rahul@ucsd.edu}
}

\begin{document}

\maketitle

\begin{abstract}

Understanding the loss-landscape geometry near a minimum is key to explaining the implicit bias of gradient-based methods in non-convex optimization problems such as deep neural network training and deep matrix factorization. A central quantity for characterizing this geometry is the maximum eigenvalue of the Hessian of the loss. Currently, its precise role has been obfuscated because no exact expressions for this sharpness measure were known in general settings. In this paper, we present the first exact expression for the maximum eigenvalue of the Hessian of the squared-error loss at any minimizer in deep matrix factorization problems, resolving an open question posed by~\citet{mulayoff}. This expression reveals a fundamental property of the loss landscape in deep matrix factorization: It is sufficient, but not necessary, for flatness that, across all layers, the product of the spectral norms of the left and right intermediate factors equals a constant determined by the depth and the spectral norm of the target matrix. Most notably, in both depth-$2$ matrix factorization and deep overparameterized scalar factorization, we show that this condition is both necessary and sufficient for flatness, which implies that flat minima are spectral-norm balanced and not necessarily balanced in depth-$2$ matrix factorization. To complement our theory, we provide empirical verification, relying on our exact expression, of an escape phenomenon during gradient-based training near a minimizer of a deep matrix factorization problem.
\end{abstract}

\section{Introduction}
Decades of research in learning theory suggest limiting model complexity to prevent overfitting. However, modern deep learning is heavily overparameterized and has nonetheless achieved unprecedented success in practice over the past decade~\citep{alexnet,attentionisallyouneed}. Generally, in overparameterized settings, the loss function has infinitely many global minima that achieve zero training error (interpolation regime), yet these models still perform well. This phenomenon has been explored in various settings such as nonparametric regression~\citep{belkin2018doesdatainterpolationcontradict}, training two-layer neural networks with logistic loss~\citep{frei2025benignoverfittinglinearityneural}, and linear regression~\citep{Bartlett_2020}.

The observation of this phenomenon naturally raises the question: Why do optimization algorithms find minima that generalize well? In the literature, the propensity of neural network training dynamics to converge to \emph{good minima} is generally attributed to the ability of gradient-based optimization algorithms to avoid \emph{bad minima}~\citep{exploringgeneralizationindeeplearning,understandingdeeplearningrequires}. This is related to the \emph{implicit bias} of gradient descent (GD)~\citep{neyshabur2015}, and a large body of work has focused on its understanding~\citep{gunasekar2017matrixfactorization,linearconv,soudry2024,implicitregularizationindeepmatrixfactorization,deeplinbias,homoneural}.

\begin{table}[t]
    \centering
    \caption[Closed-form expressions for the maximum Hessian eigenvalue in the literature.]{Closed-form expressions for the maximum Hessian eigenvalue in the literature. $\Omega$ denotes the set of \emph{all} global minimizers, $\Omega_F \subseteq \Omega$ denotes the set of \emph{flat} global minimizers, and $\Omega_B \subseteq \Omega$ denotes the set of \emph{balanced} global minimizers (see \Cref{balanceddefinition}). The \emph{Layers} column specifies the form of the layers (factors) considered in each work: scalar-valued ($\mathbb{R}$), vector-valued ($\mathbb{R}^{a}$), or matrix-valued ($\mathbb{R}^{a\times b}$).}
    \label{table1}
    \resizebox{\textwidth}{!}{%
    \begin{tabular}{|c|c|c|c|c|c|}
    \hline
     Related Work & Depth & Input Dim. & Output Dim. & $\lambda_{\max}(\nabla^2\mathcal{L}(\mathbf{w}))$ & Layers \\
    \hline
    \citet[Theorem~1]{mulayoff} & $L$ & $d_0$ & $d_L$ & $\mathbf{w} \in \Omega_F$ & $\mathbb{R}^{a \times b}$   \\
    \citet[Appendix~B.1]{identical4eq2} & $2$ & $1$ & $1$ & $\mathbf{w} \in \mathbb{R}^{N}$ & $\mathbb{R}$  \\
    \citet[Theorem~1]{ethpaper} & $2$ & $1$ & $1$ & $\mathbf{w} \in \mathbb{R}^{N}$ & $\mathbb{R}^{a}$  \\
    \citet[Lemma~1]{ghosh2025learning} & $L$ & $d_0$ & $d_L$ & $\mathbf{w} \in \Omega_B$ & $\mathbb{R}^{a \times b}$  \\
    Theorem~\ref{theorem:deepmatrixfactorization} (This Paper) & $L$ & $d_0$ & $d_L$ & $\mathbf{w} \in \Omega$ & $\mathbb{R}^{a \times b}$  \\
    \hline
    \end{tabular}}
\end{table}

It has been observed that \emph{dynamical stability} of GD near a minimum is a key factor in characterizing its implicit bias toward particular solutions~\citep{Wustab,nar2018stepsizemattersdeep}. Conceptually, dynamical stability refers to the ability of GD to \emph{stably converge} to a minimum, and it is closely related to the sharpness of the loss landscape in its vicinity~\citep{mulayoff2}. This topic has been investigated in numerous works~\citep{nar2018stepsizemattersdeep,Wustab,ma2021linearstabilitysgdinputsmoothness,mulayoff2,nacson2023implicitbiasminimastability,qiao2024stable,liang2025stable} within the framework of the classical notion of \emph{linear stability} in dynamical systems~\citep{strogatz2018nonlinear}.

Ultimately, this understanding boils down to understanding the \emph{geometry} of the loss landscape near a minimum. The maximum eigenvalue of the Hessian of the loss serves as a key measure to quantify the \emph{sharpness} of the landscape near a minimum. Despite its significance, its precise role is not well-understood, particularly because closed-form expressions are generally unknown, outside a few particular cases. We summarize the current state of understanding as well as the contributions of our paper in Table~\ref{table1}.

Most notably, the seminal work of \citet{mulayoff} derives a closed-form expression for the maximum eigenvalue of the Hessian at \emph{flat} global minima of deep linear networks/deep matrix factorization with squared-error loss.  However, obtaining a closed-form expression for \emph{all} global minima in deep linear networks/deep matrix factorization was an open problem. In particular,~\citet{mulayoff} claim that finding a closed-form expression for arbitrary global minima is intractable. In this paper, we refute this claim and give a positive answer to the following fundamental question.

{\centering
\emph{Does a closed-form expression for the maximum eigenvalue of the Hessian exist for overparameterized deep matrix factorization problems?}\par}

In particular, in Theorem~\ref{theorem:deepmatrixfactorization}, we provide a closed-form expression for the maximum Hessian eigenvalue at \emph{arbitrary global minima} of depth-$L$ overparameterized deep matrix factorization problems with squared-error loss. We consider a deep matrix factorization setting in which no layer is narrower than both the input and output layers. This guarantees that the factorization is feasible for \emph{every target matrix} in $\mathbb{R}^{d_L \times d_0}$ and hence that the global minima operate in the interpolation regime. Note that the squared-error loss is non-negative, so the existence of an interpolating solution means that the optimal training loss is zero; consequently, every global minimizer is an interpolating solution. Without the above architectural constraint, the global minimizers need not achieve zero training error for an arbitrary target matrix, and our closed-form expression applies only to those target matrices that admit an interpolating solution, a set determined by the dimensions of the factors.  

To the best of our knowledge, our analysis provides the first exact expression of the maximum eigenvalue for deep matrix factorization/deep linear neural network training problems. In the case of deep overparameterized scalar factorization (Corollary~\ref{theorem:warmupdeepoverparameterizedscalarfactorization}) and depth-$2$ matrix factorization (Corollary~\ref{remark:twolayer}), our closed-form expression simplifies considerably. In Section~\ref{sec:flatness}, we present new necessary and sufficient conditions for flat minima in various matrix factorization settings. In Section~\ref{sec:discussion}, we demonstrate that different notions of sharpness can lead to misleading conclusions, and we illustrate this with an explicit example.

With our closed-form expression in hand, we then empirically observe in Section~\ref{sec:exp} the \emph{escape phenomenon}, which was introduced by \citet{Wustab}. This phenomenon has been observed in numerous works; however, it has been verified empirically only in extremely simplified settings. For instance, \citet[see Fig.~1]{lewkowycz2020large} observed this phenomenon for a one-hidden layer linear network trained with squared-error loss on a single scalar data point;  \citet{cohen2021gradient} analyzed it for a convex quadratic loss function; most recently, \citet[see Fig.~4]{ghosh2025learning} characterized it empirically for the two-layer scalar factorization loss. Although their analysis is not directly related to the escape phenomenon, \citet{liang2025gradientdescentlargestep} recently studied the convergence regions of GD in deep matrix factorization problems. 

In contrast, we provide the first empirical verification of the escape phenomenon during gradient-based training near a minimizer of a deep matrix factorization problem. We find that the escape phenomenon also occurs for overparameterized deep matrix factorization problems. We verify this phenomenon through the lens of \emph{dynamical stability} introduced by \citet{Wustab}.

Consider a twice continuously differentiable loss function $\mathcal{L}: \mathbb{R}^{N} \rightarrow \mathbb{R}$, and let $\mathbf{w}^*$ be a minimizer of $\mathcal{L}$. For $\delta >0$, denote the open ball $
    B_{\delta}(\mathbf{w}^*) \coloneqq \left \{\mathbf{w} \in \mathbb{R}^N : \norm{\mathbf{w} - \mathbf{w}^*}_2 < \delta \right \}
$. There exists a $\delta >0$ such that, for all $\mathbf{w} \in B_{\delta}(\mathbf{w}^*)$, 
\begin{equation} \label{eq:taylor}
\mathcal{L}(\mathbf{w})
= \mathcal{L}(\mathbf{w}^*)
+ (\mathbf{w} - \mathbf{w}^*)^\top \nabla \mathcal{L}(\mathbf{w}^*)
+ \frac{1}{2} (\mathbf{w} - \mathbf{w}^*)^\top \nabla^2 \mathcal{L}(\mathbf{w}^*) (\mathbf{w} - \mathbf{w}^*)
+ o(\norm{\mathbf{w}- \mathbf{w}^*}_2^2).
\end{equation}
\begin{definition}
    Let $\mathcal{L}: \mathbb{R}^{N} \rightarrow \mathbb{R}$ be a twice continuously differentiable loss function, and let $\mathbf{w}^*$ be a minimizer of $\mathcal{L}$. The linearized GD update rule for $\mathcal{L}$ is given by 
\begin{equation}
    \Tilde{\mathbf{w}}_{t+1} = \Tilde{\mathbf{w}}_t - \eta \nabla^2 \mathcal{L}(\mathbf{w}^*)(\Tilde{\mathbf{w}}_t-\mathbf{w}^*),
\end{equation}
where $\eta > 0$ is the step size. Given $\delta > 0$ such that (\ref{eq:taylor}) holds, if there exists an initial point $\mathbf{w}_0 \in B_{\delta}(\mathbf{w}^*)$ such that the residuals $\Tilde{\bm{\epsilon}}_t = \Tilde{\mathbf{w}}_t - \mathbf{w}^*$ diverge, i.e.,
\begin{equation}
    \lim_{t \rightarrow \infty} \norm{\Tilde{\bm{\epsilon}}_t}_2 = \infty,
\end{equation}
then $\mathbf{w}^*$ is said to be \emph{dynamically unstable}\footnote{Note that Definition~\ref{linearinstability} describes the linearized dynamics in a region where the loss is well approximated by its quadratic approximation. Once GD leaves this region, the linearized iterates $\{\Tilde{\mathbf{w}}_t\}_{t \geq 0}$ no longer capture the true GD iterates $\{\mathbf{w}_t\}_{t \geq 0}$.}. 
\label{linearinstability}
\end{definition}

A necessary and sufficient condition for a minimum to be dynamically unstable is that $\lambda_{\max}(\nabla^2 \mathcal{L}(\mathbf{w}^*)) > 2/\eta$ (cf.~\citealt{Wustab,mulayoff2,chemnitz2025characterizing}). Thus, we use our exact expression for the maximum eigenvalue to verify the escape phenomenon empirically in \Cref{sec:exp}.

\subsection{Contributions}

In this paper, we present the first exact expression for the maximum eigenvalue of the Hessian of the squared-error loss at any minimizer in general overparameterized deep matrix factorization problems. Our results lead to the following remarkable observations:
\begin{itemize}
    \item A minimizer of the deep overparameterized scalar factorization loss is flat \textbf{if and only if}, across all layers, the product of the spectral norms of the left and right intermediate factors equals a constant determined by the depth and the magnitude of the target (Corollary~\ref{corollarydeeposcalerspectralbalanced}).
    
    \item Flat minima are spectral-norm balanced in depth-$2$ matrix factorization (Corollary \ref{twolayerspectralbalanced}) and not necessarily balanced, contrary to claims made in \citep{ghosh2025learning}. We further discuss this in Section \ref{sec:discussion}.

    \item A minimizer of the deep matrix factorization loss is flat \textbf{if}, across all layers, the product of spectral norms of the left and right intermediate factors equals a constant determined by the depth and the spectral norm of the target matrix (Corollary~\ref{corollary:deepmatrix}). We further show, via an explicit counterexample, that the converse does not hold in general (Remark~\ref{countexmp}).

\end{itemize}
We also highlight a recent work of \citet{singh2026cracking} that provides empirical evidence, in two-layer linear networks, that the largest eigenvalue of the Hessian of the loss is correlated with the spectral norm of the weight matrices. Furthermore, for deep linear networks, they observe that this largest Hessian eigenvalue correlates with the spectral norm of the products of the weight matrices. Thus, our results in Theorem~\ref{theorem:deepmatrixfactorization}, Corollary~\ref{theorem:warmupdeepoverparameterizedscalarfactorization}, and Corollary~\ref{remark:twolayer} provide the first theoretical explanation for their empirical observations.
\subsection{Related Work}

\textbf{Flat Minima.} \citet{mulayoff} derived a closed-form expression for the maximum eigenvalue of the Hessian at flat minima for deep linear neural networks. They also showed that when the input and output dimensions are the same, the Hessian at a global minimum is rank-deficient by at least the order of $1 - 1/L$, where $L$ is the depth of the network. Moreover, they showed that the sharpness of the flat minima increases approximately linearly with $L$ if $L \gg 1$. \citet{ethpaper} provided a full characterization of the Hessian spectrum at a point in parameter space for linear and ReLU networks in the \emph{scalar regression} case. They observed that the eigenvalues scale in proportion to the input variance in scalar linear networks with one hidden layer. More recently, \citet{josz2025geometryflatminima} has shown that locally flat minima are globally flat in depth-$2$ matrix factorization problems.

\textbf{Balanced Minima.} \citet{ghosh2025learning} provided a full characterization of the Hessian spectrum at balanced minima in deep matrix factorization problems. Furthermore, they showed that the maximum eigenvalue of the Hessian at flat minima is equal to that of the balanced minima. \citet{ding2024flat} showed that \emph{norm-minimal}, \emph{balanced} and \emph{flat} solutions coincide in depth-$2$ matrix factorization problems, where the sharpness is measured by the \emph{scaled trace} of the Hessian matrix of the loss function. We further discuss the results of \citet{ghosh2025learning} and \citet{ding2024flat} in Section \ref{sec:discussion}. Finally, \citet{wang2022large} showed that large step size GD training induces a \emph{Frobenius-norm balancing effect} between factors in depth-$2$ matrix factorization problems.

\textbf{Dynamical Stability.} 
In dynamical systems theory, the local stability of a critical point plays a central role in characterizing the convergence behavior of nearby trajectories \citep{strogatz2018nonlinear}. In the seminal work of \citet{Wustab} on the dynamical stability analysis of GD training, it was shown that a global minimum is \emph{dynamically stable} for GD if and only if the step size does not exceed $2/\lambda_{\max}$, where $\lambda_{\max}$ denotes the maximum eigenvalue of the Hessian of the loss at the minimum. \citet{mulayoff2} investigated this mechanism in the space of learned functions for two-layer overparameterized univariate ReLU networks in the interpolation regime. This was then extended to multivariate ReLU networks by \citet{nacson2023implicitbiasminimastability}. The interpolation assumption was then removed by \citet{qiao2024stable} and \citet{liang2025stable}.

\textbf{Edge-of-Stability.}
\citet{cohen2021gradient} observed that neural networks trained with GD typically operate in a regime called \emph{edge of stability}, in which the maximum eigenvalue of the Hessian of the loss hovers just
above the value $2/\eta$, where $\eta$ is the step size, and argued that classical optimization theory fails to explain this phenomenon. Recently, \citet{liang2025stable} empirically observed that explicit regularization seems to break the edge-of-stability phenomenon. 

\textbf{Sharpness and Generalization.} It is widely recognized in the literature that flat minima are associated with better generalization \citep{flatminima,keskar2017on}. In a large-scale empirical investigation, \citet{Jiang*2020Fantastic} examined different complexity measures for deep networks and found that a sharpness-based measure exhibited the strongest correlation with generalization. There is also theoretical evidence for this phenomenon in low-rank matrix recovery~\citep{ding2024flat}. On the other hand, \citet{sharpminimacangeneralize} showed that \emph{good minima} can be arbitrarily sharp in deep neural networks.

\section{Notation and Problem Setup} \label{sec:notation}

We denote the \emph{Kronecker product} by $\otimes$, the \emph{Frobenius inner product} by $\langle \cdot , \cdot \rangle$, the \emph{spectral norm} by $\sigma_{\max}(\cdot)$, and the \emph{Frobenius norm} by $\norm{\cdot}_F$. We denote by $[L]$ the set of natural numbers up to $L$, that is, $[L] \coloneqq \{1,2,\ldots,L\}$, and by $\mathbf{I}_{n \times n}$ the $n \times n$ identity matrix. We denote a sequence by $\{x_i\}_{i \geq 0}$.

To simplify the notation for subsequent derivations, we define 
\begin{equation}
    \prod_{j = n \\ 
    }^ m \mathbf{W}_j \coloneqq \begin{cases}
        \mathbf{W}_m \mathbf{W}_{m-1} \dots {\mathbf{W}_{n}} \quad &1 \leq n \leq m \leq L, \\
        \mathbf{I}_{d_m \times d_m} &0 \leq m \leq L \: \: \text{and} \: \: n = m+1,
    \end{cases}
\end{equation}
where $\mathbf{W}_i\in \mathbb{R}^{d_i \times d_{i-1}}$ for any $i \in [L]$.

In this paper, we study the sharpness of the loss landscape of deep matrix factorization problems near any global minimizer. We consider the following optimization problem
\begin{equation}
    \min_{\mathbf{w} \in \mathbb{R}^{N}}\mathcal{L}\left(\mathbf{w}\right) \coloneqq \norm {\mathbf{M}-\mathbf{W}_L \mathbf{W}_{L-1}\cdots \mathbf{W}_1}_F^2,
    \label{oop}
\end{equation}
where $ \mathbf{w} = [\operatorname{vec}(\mathbf{W}_1)^\top \: \: \operatorname{vec}(\mathbf{W}_2)^\top \: \: \ldots  \: \: \operatorname{vec}(\mathbf{W}_L)^\top]^\top$ denotes the collection of all parameters, and $N$ is the total number of parameters in the model. $\mathbf{M} \in \mathbb{R}^{d_L \times d_0} \setminus \{\mathbf{0}\}$ denotes the target matrix to be factorized, $L \geq 2$ denotes the depth of factorization and  $\mathbf{W}_i \in \mathbb{R}^{d_i \times d_{i-1}}$ is the $i$th factor (layer). This objective is analogous to that of deep linear neural networks. To guarantee the feasibility of factorization for all points in $\mathbb{R}^{d_L \times d_0} \setminus \{\mathbf{0}\}$, we require
\begin{equation}
    d_i \geq \min \{d_0,d_L\} \quad  \forall i \in [L-1],
    \label{overparametcond}
\end{equation}
which follows directly from the fact that
\begin{equation}
    \operatorname{rank}\left(\mathbf{W}_L \cdots \mathbf{W}_1\right) \leq \min \left \{\operatorname{rank}\left(\mathbf{W}_1\right), \dots,\operatorname{rank}\left(\mathbf{W}_L\right) \right\}.
\end{equation}
Linear networks satisfying \eqref{overparametcond} are called \emph{filling architectures}, and those that do not are called \emph{non-filling architectures} \citep{trager2019pure}. We adopt the same terminology in deep matrix factorization.

We define the set of global minima of $\mathcal{L}(\mathbf{w})$ as
\begin{equation}
    \Omega \coloneqq \argmin_{\mathbf{w} \in \mathbb{R}^{N}} \mathcal{L}(\mathbf{w}) = \left \{\mathbf{w} \in \mathbb{R}^{N}: \prod_{i=1}^L \mathbf{W}_i = \mathbf{M}\right\}
\end{equation}
and define the set of flat minima as
\begin{equation}
    \Omega_F \coloneqq \argmin_{\mathbf{w}^* \in \Omega}\lambda_{\max}\left(\nabla^2\mathcal{L}(\mathbf{w}^*)\right),
\end{equation}
where sharpness is measured by the maximum Hessian eigenvalue. This sharpness measure is also known as \emph{worst-case sharpness}. 

We further define \emph{balanced} minima that have been investigated in several works as follows.
\begin{definition} \label[definition]{balanceddefinition}
    A minimizer $\mathbf{w}^* \in \Omega$ is balanced if $({\mathbf{W}^*_{k+1}})^\top\mathbf{W}^*_{k+1} = \mathbf{W}^*_k ({\mathbf{W}^*_{k}})^\top
    $ for all $k \in [L-1]$. Therefore, we define the set of balanced minimizers as 
    \begin{equation}
        \Omega_B \coloneqq \{\mathbf{w}^* \in \Omega: ({\mathbf{W}^*_{k+1}})^\top\mathbf{W}^*_{k+1} = \mathbf{W}^*_k ({\mathbf{W}^*_{k}})^\top \: \:\forall k \in [L-1] \}
    \end{equation}
\end{definition}
Additionally, we define the notion of \emph{norm-balanced minima} as follows.
\begin{definition}
A minimizer $\mathbf{w}^* \in \Omega$ is $\mathcal{S}^p$-norm balanced if 
\begin{equation}
    \norm{\mathbf{W}^*_1}_{\mathcal{S}^p} =\norm{\mathbf{W}^*_2}_{\mathcal{S}^p} = \cdots = \norm{\mathbf{W}^*_L}_{\mathcal{S}^p},
\end{equation}
where $\|\cdot\|_{\mathcal{S}^p}$ is the Schatten-$p$ norm for $p \in [1, \infty]$. When $p=2$, the minimizer is \emph{Frobenius-norm balanced}, and when $p = \infty$, it is \emph{spectral-norm balanced}. 
\end{definition}
\begin{remark}
    \label[remark]{balancedfrobeniusnormbalanced}
    A balanced minimizer $\mathbf{w}^* \in \Omega_B$ is Frobenius-norm balanced. Indeed, this is a direct consequence of the identity $\operatorname{tr}(\mathbf{X}^\top \mathbf{X}) = \operatorname{tr}(\mathbf{X}\mathbf{X}^\top) = \|\mathbf{X}\|_F^2$.
\end{remark}
\subsection{Benign Landscape} \citet{laurent2018deep} showed that all local minima of deep linear networks with convex and differentiable loss are global if the network is filling. In particular, they proved a more general theorem concerning real-valued functions that take as input a product of matrices. Thus, a corollary of \citet[Theorem~1]{laurent2018deep} is that the overparameterized deep matrix factorization problem with squared-error loss \emph{does not have spurious minima}. Thus, in this paper, we focus on global minima, without loss of generality. Furthermore, under the squared-error loss, deep linear networks have no spurious minima even when the architecture is non-filling \citep{baldi1989neural,kawaguchi2016deep, trager2019pure}.

\subsection{Deep Linear Neural Network Training}
For a training dataset $\mathcal{D} \coloneqq \{(\mathbf{x}_k, \mathbf{y}_k)\}_{k=1}^n$, the deep
linear neural network training problem considered by \citet{mulayoff} is as follows:
\begin{equation}\label{dln}
    \min_{\mathbf{w} \in \mathbb{R}^N} \mathcal{L}_{\text{DLN}}(\mathbf{w})
    \coloneqq \frac{1}{n} \sum_{k=1}^n \Big\| \mathbf{y}_k - \prod_{i=1}^L \mathbf{W}_i \mathbf{x}_k \Big\|_2^2
    = \frac{1}{n} \Big\| \mathbf{Y} - \prod_{i=1}^L \mathbf{W}_i \mathbf{X} \Big\|_F^2,
\end{equation}
where $\mathbf{X} \coloneqq [\mathbf{x}_1 \;\; \mathbf{x}_2 \;\; \cdots \;\; \mathbf{x}_n] \in \mathbb{R}^{d_0 \times n}$ and
$\mathbf{Y} \coloneqq [\mathbf{y}_1 \;\; \mathbf{y}_2 \;\; \cdots \;\; \mathbf{y}_n] \in \mathbb{R}^{d_L \times n}$.
Whenever $\mathbf{\widehat{\Sigma}}_x \coloneqq \frac{1}{n} \mathbf{X}\mathbf{X}^\top$ has full rank and $d_i \geq \min\{d_0, d_L\}$ for
all $i \in [L]$, we can construct a deep matrix factorization problem whose global
minimizers coincide with those of~\eqref{dln} \citep[Equation~4]{chou2024gradient}. For instance,
the deep matrix factorization problem considered in~\eqref{oop} with
$\mathbf{M} = \mathbf{\widehat{\Sigma}}_{yx} \mathbf{\widehat{\Sigma}}_x^{-1}$, where $\mathbf{\widehat{\Sigma}}_{yx} \coloneqq \frac{1}{n} \mathbf{Y}\mathbf{X}^\top$,
has the same set of global minimizers as~\eqref{dln}. Furthermore, when
$\mathbf{\widehat{\Sigma}}_x = \mathbf{I}_{d_0 \times d_0}$, the \emph{Hessian matrices} of the two squared-error losses
coincide at every point of this common set of global minimizers \citep[Equation~21]{mulayoff}. We therefore emphasize that overparameterized deep matrix factorization and the deep linear neural
network setting of~\citet{mulayoff} are equivalent for investigating the worst-case sharpness around a minimizer \emph{whenever $\mathbf{\widehat{\Sigma}}_x = \mathbf{I}_{d_0 \times d_0}$ and $d_i \geq \min\{d_0, d_L\}$ for
all $i \in [L]$}. 

Our analysis relies on matrix calculus and the formulation of second directional derivatives. We defer the relevant technical preliminaries to Appendix \ref{appendix:prelim}.

\section{Overparameterized Deep Matrix Factorization} \label{sec:main}
We now consider the general deep matrix factorization problem. In this section, we prove our main result, which is a closed-form expression for the maximum eigenvalue of the Hessian for any global minimum of the objective (\ref{oop}).

\begin{theorem}
Consider the problem setup in \Cref{sec:notation}. If $\mathbf{w}^* \in \Omega$ then
\begin{equation}
 \lambda_{\max} \left( \nabla^2 \mathcal{L}\left(\mathbf{w}^*\right) \right) = 2\sigma_{\max} \left (\sum_{i=1}^L \mathbf{B}_i^\top \mathbf{B}_i \otimes \mathbf{A}_i \mathbf{A}_i^\top \right),
\end{equation}
where $\mathbf{A}_k = \prod_{i = k+1}^L \mathbf{W}_i^*$ and $\mathbf{B}_k = \prod_{i = 1}^{k-1} \mathbf{W}_i^*$.
\label{theorem:deepmatrixfactorization}
\end{theorem}
The proof relies on an explicit expression for the second directional derivative of $\mathcal{L}$ that simplifies at an arbitrary global minimizer, and the full proof appears in Appendix \ref{appendix:deepmatrixfactorization}. Note that the deep overparameterized scalar factorization is a special case of deep matrix factorization where both $\mathbf{B}_i^\top \mathbf{B}_i$ and $\mathbf{A}_i \mathbf{A}_i^\top$ reduce to scalars. In that special case, we have Corollary~\ref{theorem:warmupdeepoverparameterizedscalarfactorization} below. Another corollary of Theorem~\ref{theorem:deepmatrixfactorization} is the maximum Hessian eigenvalue for the classical (depth-$2$) matrix factorization problem. This result may be of independent interest as the expression simplifies considerably.
\begin{corollary} \label{cor:warmup}
\label{theorem:warmupdeepoverparameterizedscalarfactorization}    
Consider the following objective function
\begin{equation}
    \mathcal{L}(\mathbf{w}) \coloneqq ({m-\mathbf{w}_L^\top \mathbf{W}_{L-1}\cdots \mathbf{W}_2\mathbf{w}_1})^2,
    \label{oopscalar}
\end{equation}
where $m \in \mathbb{R} \setminus \{0\}$, $d_0 = d_L = 1$, $\mathbf{w}_L \in \mathbb{R}^{d_{L-1}}$ and $\mathbf{w}_1 \in \mathbb{R}^{d_1}$. For hidden factors (layers), i.e., for all $i \in \{2,3,\cdots,L-1\}$, we have $\mathbf{W}_i \in \mathbb{R}^{d_i \times d_{i-1}}$. Then, for all $\mathbf{w}^* \in \Omega$, 
\begin{equation}
     \lambda_{\max}\left(\nabla^2 \mathcal{L} \left(\mathbf{w}^*\right)\right) =  2\sum_{i=1}^{L} \sigma_{\max} \left(\prod_{j=i+1}^{L}  \mathbf{W}_j^* \right )^2 \sigma_{\max} \left (\prod_{j=1}^{i-1} \mathbf{W}_j^* \right)^2.
\end{equation}
\end{corollary}
We also provide a self-contained proof of this corollary in Appendix~\ref{warmupproofappendix}.
\begin{corollary}
Consider the following depth-$2$ matrix factorization objective
\begin{equation}
    \mathcal{L}(\mathbf{L},\mathbf{R}) = \norm{\mathbf{M} - \mathbf{L}\mathbf{R}^\top}_F^2,
\end{equation}
where $\mathbf{M} \in \mathbb{R}^{d_L \times d_0} \setminus\{\mathbf{0}\}$ is the target matrix, $\mathbf{L} \in \mathbb{R}^{d_L \times k}$, and $\mathbf{R} \in \mathbb{R}^{d_0 \times k}$. To ensure the feasibility of the factorization for every point in $\mathbb{R}^{d_L \times d_0} $, we choose $k \geq \min \{d_0,d_L\}$. We define the set of minimizers as follows:
\begin{equation}
    \Omega \coloneqq \argmin_{\mathbf{L},\mathbf{R}} \mathcal{L}\left(\mathbf{L},\mathbf{R}\right) =   \left \{ \left(\mathbf{L}, \mathbf{R}\right) : \mathbf{M} = \mathbf{L}\mathbf{R}^\top  \right \}.
\end{equation}
If $(\mathbf{L}^*,\mathbf{R}^*) \in \Omega$, then 
\begin{equation}
    \lambda_{\max} (\nabla^2 \mathcal{L}(\mathbf{L}^*,\mathbf{R}^*)) = 2(\sigma_{\max}(\mathbf{L}^*)^2 + \sigma_{\max}(\mathbf{R}^*)^2).
\end{equation}
\label{remark:twolayer}
\end{corollary}
\begin{proof}
We have from Theorem~\ref{theorem:deepmatrixfactorization} that $ \lambda_{\max} (\nabla^2 \mathcal{L}(\mathbf{L},\mathbf{R})) $
\begin{align}
    &= 2\sigma_{\max} \left( \mathbf{I} \otimes \mathbf{L}\mathbf{L}^\top + \mathbf{R}\mathbf{R}^\top \otimes \mathbf{I}\right ) \label{kron1}\\
    &= 2\left( \sigma_{\max}\left(\mathbf{I} \otimes \mathbf{L}\mathbf{L}^\top \right)+ \sigma_{\max} \left(\mathbf{R}\mathbf{R}^\top \otimes \mathbf{I} \right) \right ) \label{kron2} \\
    &=2\left(\sigma_{\max}\left(\mathbf{L}\right)^2 + \sigma_{\max}\left(\mathbf{R}\right)^2\right) \label{kron3}
\end{align}
Using the fact from \citet[Theorem~4.4.5]{horn1994topics}, we can obtain (\ref{kron2}) from (\ref{kron1}). 
Note that for any matrices $\mathbf{A}$ and $\mathbf{B}$, $\sigma_{\max}(\mathbf{A} \otimes \mathbf{B} ) = \sigma_{\max}(\mathbf{A}) \sigma_{\max}(\mathbf{B})$, and $\sigma_{\max}(\mathbf{A}\mathbf{A}^\top) = \sigma_{\max}(\mathbf{A}^\top\mathbf{A}) = \sigma_{\max}(\mathbf{A})^2$. Hence, we rewrite (\ref{kron2}) as (\ref{kron3}).
\end{proof}
We also provide an alternative, self-contained proof of this corollary in Appendix \ref{appendix:remarktwolayer}. This result was also independently observed by \citet{josz2025geometryflatminima}.

\section{Flatness}
\label{sec:flatness}
In this section, we reveal remarkable aspects of the loss-landscape geometry of general deep matrix factorization problems. First, we show that a minimizer of the deep overparameterized scalar factorization problem is flat if and only if, across all layers, the product of spectral norms of left and right intermediate networks equals $|m|^{1-\frac{1}{L}}$. This is a direct consequence of Corollary~\ref{theorem:warmupdeepoverparameterizedscalarfactorization}.

\begin{corollary}
\label{corollarydeeposcalerspectralbalanced}
Consider the following objective function
\begin{equation}
\mathcal{L}\left(\mathbf{w}\right) \coloneqq ({m-\mathbf{w}_L^\top \mathbf{W}_{L-1}\cdots \mathbf{W}_2\mathbf{w}_1})^2,
\end{equation}
where $m \in \mathbb{R} \setminus \{0\}$, $d_0 = d_L = 1$, $\mathbf{w}_L \in \mathbb{R}^{d_{L-1}}$ and $\mathbf{w}_1 \in \mathbb{R}^{d_1}$. For hidden factors (layers), i.e., for all $i \in \{2,3,\cdots,L-1\}$, we have $\mathbf{W}_i \in \mathbb{R}^{d_i \times d_{i-1}}$. Then, for any $\mathbf{w}^* \in \Omega$, we have $\mathbf{w}^* \in \Omega_F$ if and only if
\begin{equation}
    \sigma_{\max}(\mathbf{A}_k)\sigma_{\max}(\mathbf{B}_k) = |m|^{1-\frac{1}{L}} \quad \forall k \in [L],
\end{equation}
where $\mathbf{A}_k = \prod_{i = k+1}^L \mathbf{W}_i^*$ and $\mathbf{B}_k = \prod_{i = 1}^{k-1} \mathbf{W}_i^*$.
\end{corollary}
\begin{proof}
    For a minimizer $\mathbf{w}^*$, let us assume that $\sigma_{\max}(\mathbf{A}_k) \sigma_{\max}(\mathbf{B}_k) = |m|^{1-\frac{1}{L}}$ for all $k \in [L]$. Then, by Corollary~\ref{theorem:warmupdeepoverparameterizedscalarfactorization}, $\lambda_{\max}(\nabla^2\mathcal{L}(\mathbf{w}^*)) = 2L \times |m|^{2(1-1/L)}$. This implies that $\mathbf{w}^*$ is flat \citep[also see~Theorem~1]{mulayoff}. Now, assume that $\mathbf{w}^*$ is flat. By using the fact from \citet[Theorem~2]{mulayoff}, for any flat minimum $\mathbf{w}^*$ and $k \in [L]$, we have
    \begin{equation}
        \sigma_{\max}(\mathbf{A}_k) = |m|^{1- \frac{k}{L}} \:\:\text{and} \:\:\sigma_{\max}(\mathbf{B}_k) = |m|^{\frac{k-1}{L}}.
    \end{equation}
    Hence, $ \sigma_{\max}(\mathbf{A}_k) \sigma_{\max}(\mathbf{B}_k) = |m|^{1-\frac{1}{L}}$ for all $k \in [L]$.
\end{proof}
Second, we show that in depth-$2$ matrix factorization, a minimizer is flat if and only if the spectral norms of its factors equal $\sqrt{\sigma_{\max}(\mathbf{M})}$. This result is a direct consequence of Corollary \ref{remark:twolayer} and the AM-GM inequality \citep[see also][Lemma~5]{josz2025geometryflatminima}. 
\begin{corollary} \label{twolayerspectralbalanced}
Consider the following depth-$2$ matrix factorization objective
\begin{equation}
    \mathcal{L}(\mathbf{L},\mathbf{R}) = \norm{\mathbf{M} - \mathbf{L}\mathbf{R}^\top}_F^2,
\end{equation}
where $\mathbf{M} \in \mathbb{R}^{d_L \times d_0} \setminus \{\mathbf{0}\}$ is the target matrix, $\mathbf{L} \in \mathbb{R}^{d_L \times k}$, and $\mathbf{R} \in \mathbb{R}^{d_0 \times k}$ such that $k \geq \min \{d_0,d_L\}$. Then, for any $(\mathbf{L}^*, \mathbf{R}^*) \in \Omega$, we have $(\mathbf{L}^*, \mathbf{R}^*) \in \Omega_F$ if and only if 
\begin{equation}
    \sigma_{\max}(\mathbf{L}^*) = \sigma_{\max}(\mathbf{R}^*) = \sqrt{\sigma_{\max}(\mathbf{M})}.
\end{equation}
\end{corollary}
\begin{remark}
We discuss further implications of this equivalence in Section~\ref{sec:discussion}.
\end{remark}

Now, we show that a minimizer of deep matrix factorization loss is flat if, across all layers, the product of spectral norms of the left and right intermediate factors equals $\sigma_{\max}(\mathbf{M})^{1-\frac{1}{L}}$.
\begin{corollary}
\label{corollary:deepmatrix}
    Consider the optimization objective (\ref{oop}). If $\mathbf{w}^* \in \Omega$, then 
    \begin{equation}
        \lambda_{\max} \left( \nabla^2 \mathcal{L}(\mathbf{w}^*) \right) \leq 2 \sum_{k=1}^{L}\sigma_{\max}(\mathbf{A}_k)^2\sigma_{\max}(\mathbf{B}_k)^2,
        \label{inequality}
    \end{equation}
    where $\mathbf{A}_k = \prod_{i = k+1}^L \mathbf{W}_i^*$ and $\mathbf{B}_k = \prod_{i = 1}^{k-1} \mathbf{W}_i^*$ (see Appendix \ref{warmupproofappendix}). This implies that if $\mathbf{w}^*$ satisfies $\sigma_{\max}(\mathbf{A}_k)\sigma_{\max}(\mathbf{B}_k) = \sigma_{\max}(\mathbf{M})^{1-\frac{1}{L}}$ for all $k \in [L]$, then $\mathbf{w}^*$ is flat. 
\end{corollary}
\begin{proof}
\citet{mulayoff} showed that for any $\mathbf{w}^* \in \Omega$,
\begin{equation}
    \lambda_{\max} \left( \nabla^2 \mathcal{L}(\mathbf{w}^*) \right) \geq 2L\sigma_{\max}(\mathbf{M})^{2(1-1/L)}.
\end{equation}
Now, assume that $\sigma_{\max}(\mathbf{A}_k)\sigma_{\max}(\mathbf{B}_k) = \sigma_{\max}(\mathbf{M})^{1-\frac{1}{L}}$ for all $k \in [L]$. Then by (\ref{inequality}),
\begin{equation}
    \lambda_{\max} \left( \nabla^2 \mathcal{L}(\mathbf{w}^*) \right) = 2L\sigma_{\max}(\mathbf{M})^{2(1-1/L)}.
\end{equation}
 This implies that $\mathbf{w}^*$ is flat \citep[Theorem~1]{mulayoff}.
\end{proof}

\begin{remark}
\label{countexmp}
The constant product of the spectral norms of the left and right intermediate factors across layers is not necessary for flatness. To see this, consider the depth-3 matrix factorization objective
\begin{equation}
    \mathcal{L}(\mathbf{W}_1, \mathbf{W}_2, \mathbf{W}_3) = \left\| \mathbf{M} - \mathbf{W}_3\mathbf{W}_2\mathbf{W}_1 \right\|_F^2, \qquad \mathbf{M} = \mathrm{diag}\left(1, \tfrac{1}{10}\right),
\end{equation}
where all factors are $2 \times 2$ matrices. Let $\mathbf{w}^*$ be the point given by
\begin{equation}
    \mathbf{W}_3^* = \mathrm{diag}\left(1, \tfrac{1}{2}\right), \quad
    \mathbf{W}_2^* = \mathrm{diag}\left(1, \tfrac{2}{11}\right), \quad
    \mathbf{W}_1^* = \mathrm{diag}\left(1, \tfrac{11}{10}\right).
\end{equation}
Since $\mathbf{W}_3^*\mathbf{W}_2^*\mathbf{W}_1^* = \mathbf{M}$, the point $\mathbf{w}^*$ is a global minimizer. The corresponding left and right intermediate factors are
\begin{align}
    \mathbf{B}_1 &= \mathbf{I}_{d_0 \times d_0},
    & \mathbf{B}_2 &= \mathbf{W}_1^* = \mathrm{diag}\left(1, \tfrac{11}{10}\right),
    & \mathbf{B}_3 &= \mathbf{W}_2^*\mathbf{W}_1^* = \mathrm{diag}\left(1, \tfrac{1}{5}\right), \nonumber\\
    \mathbf{A}_1 &= \mathbf{W}_3^*\mathbf{W}_2^* = \mathrm{diag}\left(1, \tfrac{1}{11}\right),
    & \mathbf{A}_2 &= \mathbf{W}_3^* = \mathrm{diag}\left(1, \tfrac{1}{2}\right),
    & \mathbf{A}_3 &= \mathbf{I}_{d_L \times d_L}.
\end{align}
By Theorem~\ref{theorem:deepmatrixfactorization}, $\lambda_{\max}(\nabla^2 \mathcal{L}(\mathbf{w}^*)) = 6$. \citet{mulayoff} showed that for every $\mathbf{w} \in \Omega_F$, $\lambda_{\max}(\nabla^2 \mathcal{L}(\mathbf{w})) = 2L\,\sigma_{\max}(\mathbf{M})^{2(1-1/L)}$, which equals $6$ for $\mathbf{M} = \mathrm{diag}\left(1, \tfrac{1}{10}\right)$ and $L = 3$. Since $\lambda_{\max}(\nabla^2 \mathcal{L}(\mathbf{w}^*))$ attains this value, $\mathbf{w}^*$ is flat. However,
\begin{equation}
    \sigma_{\max}(\mathbf{A}_1)\,\sigma_{\max}(\mathbf{B}_1) = 1, \quad
    \sigma_{\max}(\mathbf{A}_2)\,\sigma_{\max}(\mathbf{B}_2) = \frac{11}{10}, \quad
    \sigma_{\max}(\mathbf{A}_3)\,\sigma_{\max}(\mathbf{B}_3) = 1.
\end{equation}
Therefore, flatness \emph{does not imply} that $\sigma_{\max}(\mathbf{A}_k)\,\sigma_{\max}(\mathbf{B}_k)$ is constant across layers.
\end{remark}

\section{Discussion}
\label{sec:discussion}
In this section, we examine the claims made by \citet{ding2024flat} and \citet{ghosh2025learning} that flat minima are balanced in depth-$2$ matrix factorization (see \Cref{balanceddefinition}). We then discuss how sharpness measures can be misleading in loss landscape analysis. We conclude by addressing the limitations of our analysis.


\paragraph{Flat Minima Need Not Be Balanced.}
First, we examine the latter work, in which \citet{ghosh2025learning} studied the learning dynamics of deep matrix factorization under the squared-error loss beyond the edge of stability. They claim that \emph{flat minima correspond to balanced minima} in deep matrix factorization problems \citep[Proof of Lemma 4]{ghosh2025learning}. Globally, however, we can find minimizers that are \textbf{flat but not Frobenius-norm balanced}. As we showed in \Cref{balancedfrobeniusnormbalanced}, Frobenius-norm balancedness is necessary for being balanced; therefore, \emph{flat minima need not be balanced}. To see this, consider a depth-$2$ matrix factorization problem, i.e.,
\begin{equation}
    \min_{\mathbf{L} \in \mathbb{R}^{m \times r} \\ \mathbf{R} \in \mathbb{R}^{n \times r}} \left\{ \mathcal{L}(\mathbf{L}, \mathbf{R}) \coloneqq \norm{\mathbf{M} - \mathbf{L}\mathbf{R}^\top}_F^2 \right\}.
\end{equation}
Suppose $m=n=r=3$ and $\mathbf{M} = \text{diag}(11,2,1)$. We know that for any flat minimum $(\mathbf{L}',\mathbf{R}') \in \Omega_F$, $\lambda _{\max}(\nabla^{2} (\mathcal{L}(\mathbf{L}',\mathbf{R}'))) = 4 \times \sigma_{\max} (\mathbf{M}) = 44$. By definition, any minimum whose sharpness is equal to $44$ is a flat minimum. Let's investigate a specific minimizer $(\mathbf{L} ^{*}, \mathbf{R} ^ * )$ such that $\mathbf{L} ^ * = \text{diag}(\sqrt{11} , \frac{2}{3}, 1)$ and $\mathbf{R} ^ * = \text{diag}(\sqrt{11} , 3, 1)$. Using Corollary \ref{remark:twolayer}, $\lambda _{\max}(\nabla ^2(\mathcal{L}(\mathbf{L}^ *,\mathbf{R}^*))) = 44$, which means that $(\mathbf{L}^*,\mathbf{R}^*)$ is a flat minimum. On the other hand, by definition, $(\mathbf{L}^{*}, \mathbf{R}^{*})$ is not Frobenius-norm balanced. Therefore, a flat minimum is not necessarily balanced. Moreover, we provide a counterexample in Appendix~\ref{app:svs} to show that the singular vector stationary set characterization of
\citet[Proposition~1]{ghosh2025learning} does not hold in general.

\paragraph{Sharpness Measures Are Delusive.}  \citet{ding2024flat} showed that flat minima recover the ground truth matrix in low-rank matrix recovery. Furthermore, they showed that \emph{norm-minimal}, \emph{balanced} and \emph{flat} solutions coincide in depth-$2$ matrix factorization \citep[Lemma~2.2]{ding2024flat}. They measured the sharpness of a minimum using the \emph{scaled trace}.\footnote{For the deep matrix factorization problem in (\ref{oop}), the scaled trace of the Hessian matrix of the squared-error loss takes the form $ \operatorname{str}(\nabla^2 \mathcal{L}\left(\mathbf{w}^*\right)) = 2\|\prod_{i = 2} ^ L \mathbf{W}^*_i \|_F^2 + 2\left(\sum_{k=2}^{L-1}  \|\prod_{i = k+1} ^L \mathbf{W}^*_i \|_F^2  \|\prod_{i = 1}^{k-1} \mathbf{W}^*_i \|_F^2\right) +  2\|\prod_{i = 1} ^ {L-1} \mathbf{W}^*_i \|_F^2 $ for any $\mathbf{w}^* \in \Omega$.} Under this measure, flat minima coincide with balanced minima; however, when worst-case sharpness is used as the sharpness measure, flat minima need not even satisfy $\|\mathbf{L}^*\|_F = \|\mathbf{R}^*\|_F$. Hence, under the worst-case sharpness measure, \emph{flat minima need not be balanced}. In fact, we showed that flat minima are spectral-norm balanced when the maximum Hessian eigenvalue is used as the sharpness measure (Corollary \ref{twolayerspectralbalanced}). This means that a loss landscape analysis in non-convex optimization problems using a specific sharpness measure might not lead to the same inferences as those made by using a different metric. 

\paragraph{The Necessity of an Absolute Sharpness Measure.} 
There is still no consensus in the literature regarding the definition of flatness. For instance, while several works define flat minima as global minimizers that minimize the maximum eigenvalue of the Hessian of the loss \citep{mulayoff,liu2021noisy,marion2024deep}, others define them as global minimizers that minimize the trace of the Hessian \citep{sharpminimacangeneralize,gatmiry2023inductive}. Even a scaled version of the Hessian trace is designed to define flat minima in depth-$2$ matrix factorization problems \citep{ding2024flat}, as we mentioned before.  As we have shown in Section \ref{sec:flatness} and discussed above, different sharpness measures can lead to contradictory interpretations of balanced minima. Thus, the necessity of a robust sharpness measure that consolidates the interpretations from existing ones is a matter of urgency. Addressing this concern, \citet{josz2025geometryflatminima} showed how minimizing the maximum Hessian eigenvalue over the solution set can disregard the higher-order variation near minimizers and result in misleading interpretations of flat minima. \citet{josz2025geometryflatminima} defined a sharpness measure based on the maximal variation around a point. They showed that, for a smooth loss $\mathcal{L}$, a minimizer $\mathbf{w}^*$ is flat only if it is a local minimum of $\lambda_{\max}(\nabla^2 \mathcal{L}(\mathbf{w}))$ subject to $\mathcal{L}(\mathbf{w}) = \mathcal{L}(\mathbf{w}^*)$, and that, in depth-$2$ matrix factorization with squared-error loss, flat minima are exactly the global minima that minimize the worst-case sharpness measure.

\paragraph{Limitations.}
\label{sec:limitations}
The dynamical stability analysis relies on how accurate the quadratic approximation of the loss function is in the $\delta$-neighborhood of $\mathbf{w}^*$ \citep{Wustab} (also see Definition \ref{linearinstability}). For loss functions whose Hessian is zero at a minimizer, this analysis fails \citep{Wustab}. For instance, \citet{jules2023charting} observed that for a cross-entropy loss, the quadratic approximation of the training loss near a minimizer is fundamentally inadequate due to the exponential nature of the decision boundary. In contrast, for a nonzero target matrix $\mathbf{M}$, the squared-error loss ensures that the Hessian at any minimizer is nonzero (also see \citet[Lemma~3]{mulayoff}). Therefore, our analysis cannot be extended to other loss functions whose Hessians vanish at their global minima. Moreover, the condition required to guarantee the feasibility of factorization for any target matrix $\mathbf{M}$, i.e., Equation (\ref{overparametcond}), limits the applicability of our work in more general settings. Without this condition, our analysis holds only for target matrices $\mathbf{M}$ that admit an interpolating factorization. Finally, when $\mathbf{\widehat{\Sigma}}_x$ is full rank, the deep matrix factorization/deep linear network training problem is not well suited for analyzing the relationship between generalization and the geometric properties of the landscape. This is because, in this case, all minimizers have exactly the same generalization ability, which makes such an analysis inapplicable in our setting.

\section{Experiments} 
\label{sec:exp}

\begin{figure*}[!htb]
    \centering
    \begin{subfigure}[b]{0.32\linewidth}
        \centering
        \includegraphics[width=1\linewidth]{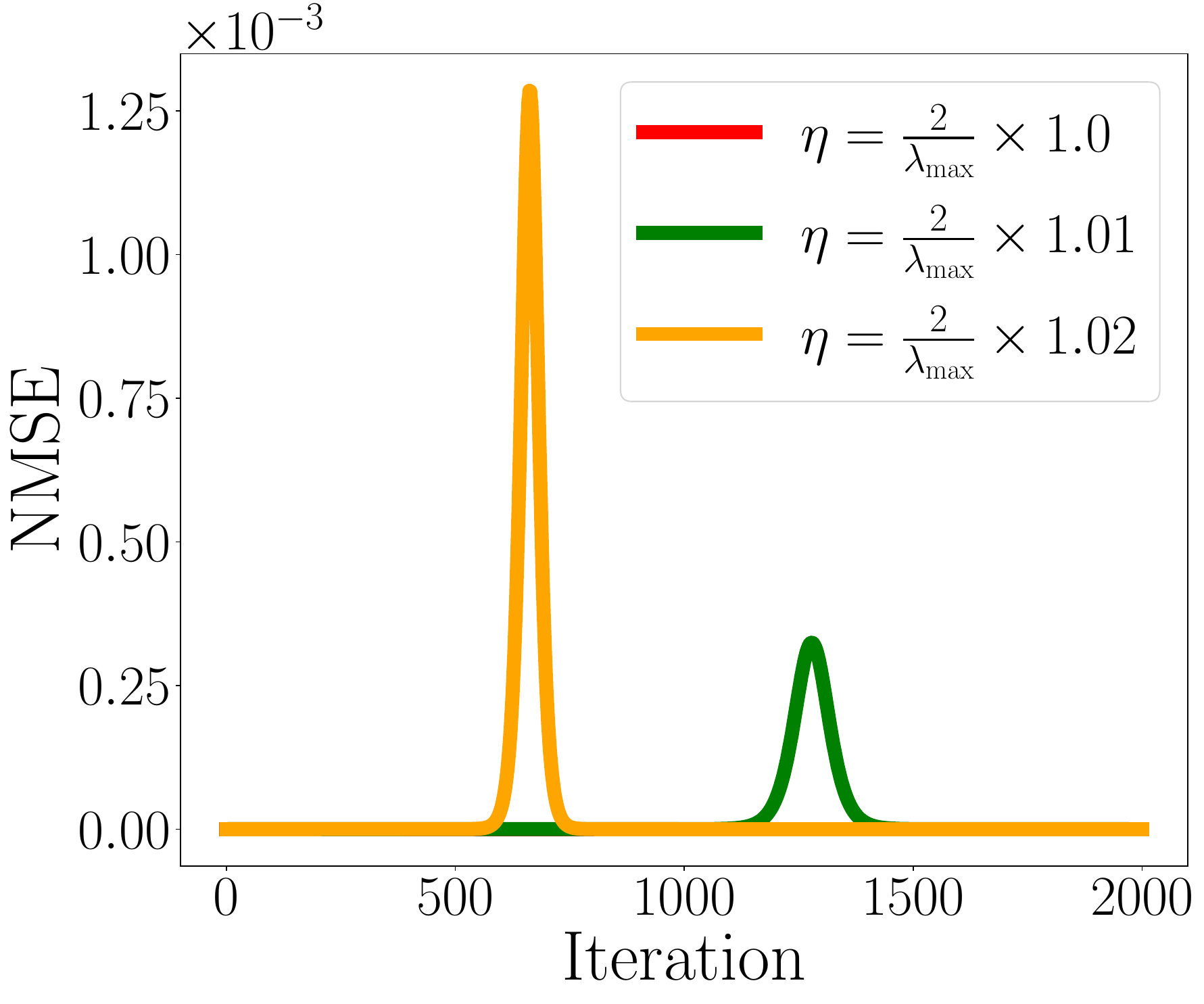}
        \caption{Normalized training error across iterations, i.e., $\mathcal{L}(\mathbf{w}_t)/\lVert \mathbf{M} \rVert_F^2$.}
        \label{fig3:first}
    \end{subfigure}
    \hfill
    \begin{subfigure}[b]{0.32\linewidth}
        \centering
        \includegraphics[width=1\linewidth]{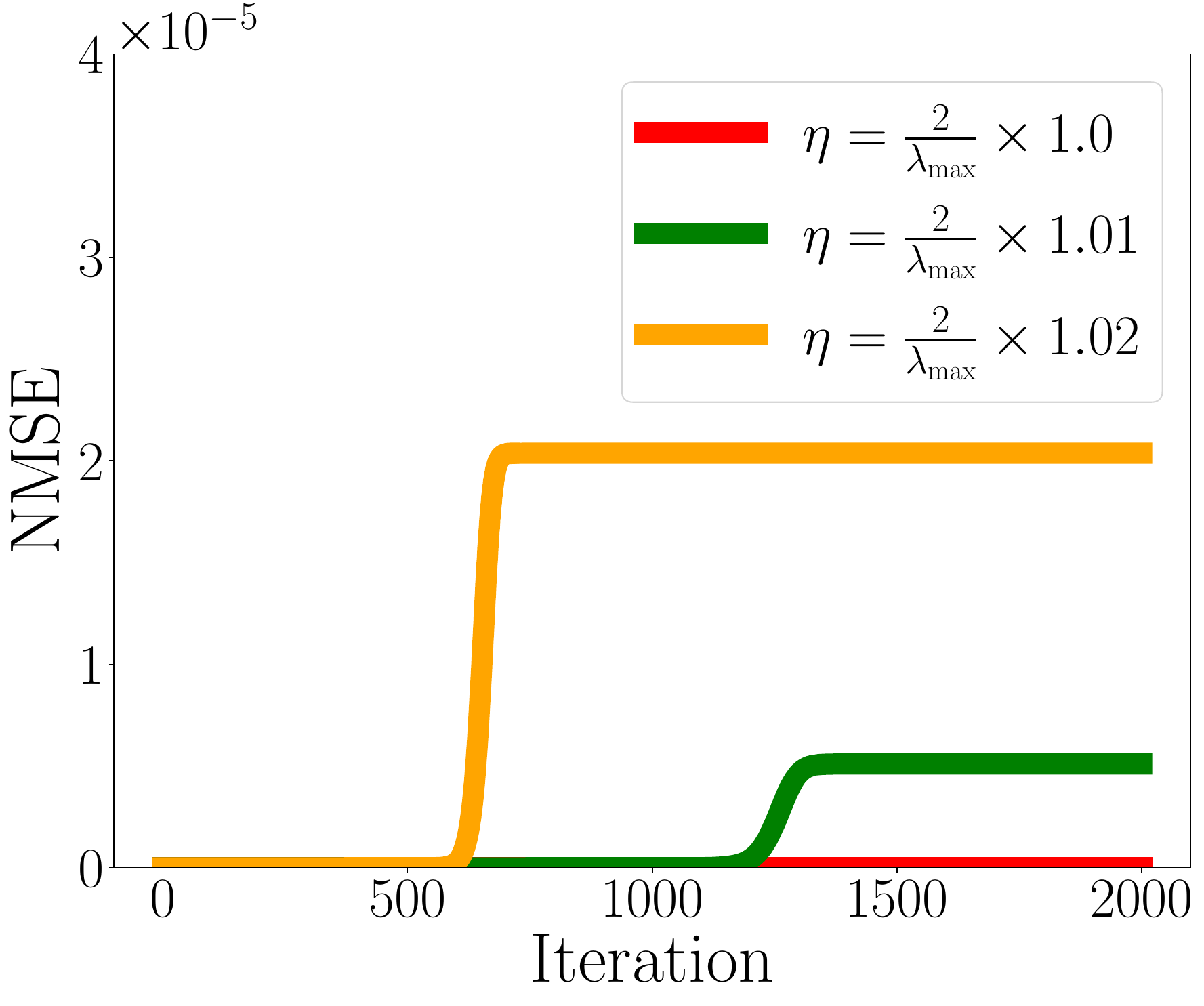}
        \caption{$\| \mathbf{w}_t - \mathbf{w}^*\|_2^2/ \|\mathbf{w}^*\|_2^2$ across iterations.}
        \label{fig3:second}
    \end{subfigure}
    \hfill
    \begin{subfigure}[b]{0.32\linewidth}
        \centering
        \includegraphics[width=1\linewidth]{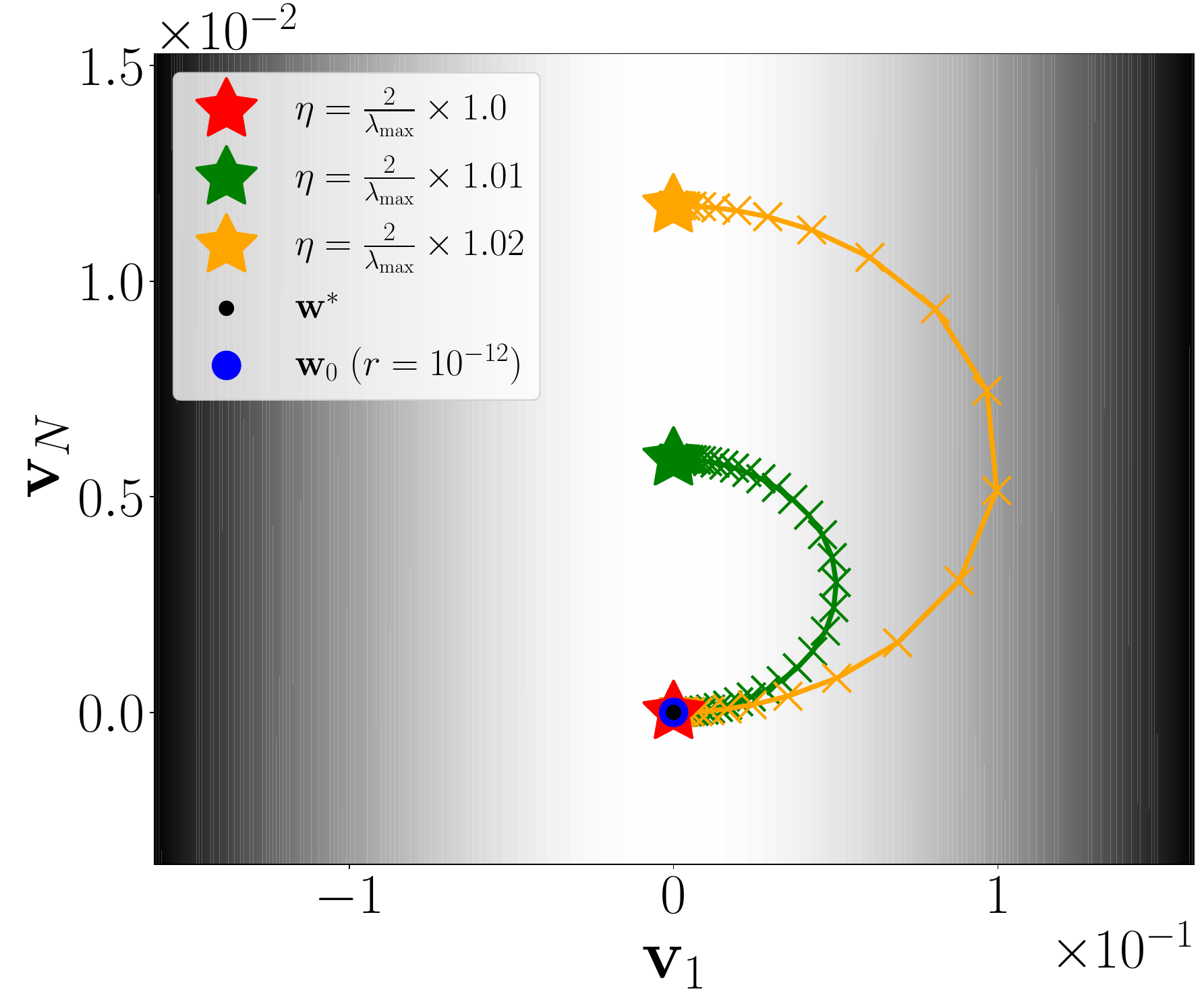}
        \caption{Trajectories of GD on the contour map of the loss landscape.}
        \label{fig3:third}
    \end{subfigure}
    \caption{The setting is a depth-$3$ matrix factorization of a random $10 \times 20$ matrix with Gaussian factors $\mathbf{W}_3 \in \mathbb{R}^{10 \times 20}$ and $\mathbf{W}_1, \mathbf{W}_2 \in \mathbb{R}^{20 \times 20}$. $\mathbf{v}_1$ and $\mathbf{v}_N$ are the eigenvectors of $\nabla^2\mathcal{L}(\mathbf{w}^*)$ that correspond to its maximum and minimum eigenvalues respectively. $\lambda_{\max}$ is computed via the closed-form expression in Theorem~\ref{theorem:deepmatrixfactorization}. GD runs with different step sizes are initialized within $10^{-12}$ of the minimizer $\mathbf{w}^*$ along $\mathbf{v}_1$. The contour plot in (c) projects the parameter space onto $\operatorname{span}\{\mathbf{v}_1, \mathbf{v}_N\}$ (Appendix~\ref{visualization}); darker (brighter) regions indicate higher (lower) loss. See Appendix~\ref{additionalexperiments} for details.}
    \label{figure3:threelayers}
\end{figure*}
We initialize GD in a $\delta$-neighborhood of $\mathbf{w}^*$ with $\delta$ of order $10^{-15}$ to $10^{-9}$ so that the linearized GD dynamics approximate the true GD dynamics well. To observe the escape phenomenon clearly, we choose the perturbation direction for the initial point to be the eigenvector of $\nabla^2 \mathcal{L}(\mathbf{w}^*)$ corresponding to the largest eigenvalue, so as to avoid choosing a direction that is orthogonal to the eigenspace corresponding to the eigenvalues larger than $2/\eta$.

To measure the distance between the iterates and the minimizer, we plot the normalized squared $\ell^2$-norm of $\mathbf{w}_t-\mathbf{w}^*$ at each iteration as shown in Fig.~\ref{fig3:second}. We say that the escape phenomenon occurred if the final distance from the minimizer, i.e., $\|\mathbf{w}_T - \mathbf{w}^*\|_2$, is significantly larger than the initialization radius. In other words, if the training dynamics move the iterates far from the minimizer relative to how close they started, this indicates that the minimizer is unstable and the dynamics have escaped its vicinity. We observe that if $\eta > 2/\lambda_{\max}$, where $\lambda_{\max} \coloneqq \lambda_{\max}(\nabla^2 \mathcal{L}(\mathbf{w}^*))$, GD escapes from the minimum as shown in Fig.~\ref{figure3:threelayers}. On the other hand, if $\eta \leq 2/\lambda_{\max}$, then GD converges. A catapult in the training error shown in Fig. \ref{fig3:first} indicates GD’s escape from the basin of a minimum, after which it eventually converges to another minimum unless the step size is very large \citep{marion2024deep}. To the best of our knowledge, this is the first empirical verification of the escape phenomenon during gradient-based training near a minimizer of a deep matrix factorization problem.
\section{Conclusion}
In this paper, we derived an exact expression for the maximum eigenvalue of the Hessian of the squared-error loss in overparameterized deep matrix factorization at any minimizer, and showed that it simplifies substantially in the depth-$2$ matrix factorization and deep overparameterized scalar factorization settings. We further showed that flat minima in depth-$2$ matrix factorization are spectral-norm balanced, and that a minimizer of the deep overparameterized scalar factorization loss is flat if and only if, across all layers, the product of the spectral norms of the left and right intermediate factors equals a constant determined by the depth and the magnitude of the target. In deep matrix factorization, the same condition is sufficient but not necessary for flatness. To complement our theory, we conducted GD experiments that crucially rely on our exact expression of the sharpness to observe the escape phenomenon during training.  Finally, we discussed how sharpness measures can be misleading for loss landscape analysis, highlighting the need for robust sharpness notions that reconcile existing interpretations. However, we did not compare the computational complexity of our closed-form expressions, particularly those for depth-2 matrix and deep overparameterized scalar factorization, with that of the Lanczos algorithm \citep{lanczos1950iteration}. Although it is clear that our closed-form expressions for these settings do not require forming the Hessian matrix to compute its maximum eigenvalue, their performance against iterative methods based on matrix-vector products has not been investigated in this work. We leave this as future work.

\section*{Acknowledgments}
The authors acknowledge support from a Hellman Fellowship from the University of California. The counterexamples provided in Remark~\ref{countexmp} and Appendix~\ref{app:svs} were suggested by Claude Fable 5 and ChatGPT-6 Astra, respectively, and the authors take full responsibility for their correctness. Beyond this, GPT-4o, GPT-5, GPT-6.1 Sol, and Claude Opus 5.5 were used as spell checkers, grammar suggestion tools, editing aids, and code assistants. 
  
\bibliographystyle{abbrvnat} 
\bibliography{refs}

@inproceedings{mulayoff,
  title={Unique properties of flat minima in deep networks},
  author={Mulayoff, Rotem and Michaeli, Tomer},
  booktitle={International conference on machine learning},
  pages={7108--7118},
  year={2020},
  organization={PMLR}
}

@article{soudry2024,
  title={The implicit bias of gradient descent on separable data},
  author={Soudry, Daniel and Hoffer, Elad and Nacson, Mor Shpigel and Gunasekar, Suriya and Srebro, Nathan},
  journal={Journal of Machine Learning Research},
  volume={19},
  number={70},
  pages={1--57},
  year={2018}
}

@inproceedings{
understandingdeeplearningrequires,
title={Understanding deep learning requires rethinking generalization},
author={Chiyuan Zhang and Samy Bengio and Moritz Hardt and Benjamin Recht and Oriol Vinyals},
booktitle={International Conference on Learning Representations},
year={2017},
}

@article{deeplinbias,
  title={Directional convergence and alignment in deep learning},
  author={Ji, Ziwei and Telgarsky, Matus},
  journal={Advances in Neural Information Processing Systems},
  volume={33},
  pages={17176--17186},
  year={2020}
}

@article{exploringgeneralizationindeeplearning,
  title={Exploring generalization in deep learning},
  author={Neyshabur, Behnam and Bhojanapalli, Srinadh and McAllester, David and Srebro, Nati},
  journal={Advances in neural information processing systems},
  volume={30},
  year={2017}
}

@article{linearconv,
  title={Implicit bias of gradient descent on linear convolutional networks},
  author={Gunasekar, Suriya and Lee, Jason D and Soudry, Daniel and Srebro, Nati},
  journal={Advances in neural information processing systems},
  volume={31},
  year={2018}
}

@inproceedings{
homoneural,
title={A unifying view on implicit bias in training linear neural networks},
author={Chulhee Yun and Shankar Krishnan and Hossein Mobahi},
booktitle={International Conference on Learning Representations},
year={2021}
}

@article{neyshabur2015,
  title={In search of the real inductive bias: On the role of implicit regularization in deep learning},
  author={Neyshabur, Behnam and Tomioka, Ryota and Srebro, Nathan},
  journal={arXiv preprint arXiv:1412.6614},
  year={2014}
}

@article{flatminima,
  title={Flat minima},
  author={Hochreiter, Sepp and Schmidhuber, J{\"u}rgen},
  journal={Neural computation},
  volume={9},
  number={1},
  pages={1--42},
  year={1997},
  publisher={MIT Press One Rogers Street, Cambridge, MA 02142-1209, USA journals-info~…}
}

@article{gunasekar2017matrixfactorization,
  title={Implicit regularization in matrix factorization},
  author={Gunasekar, Suriya and Woodworth, Blake E and Bhojanapalli, Srinadh and Neyshabur, Behnam and Srebro, Nati},
  journal={Advances in neural information processing systems},
  volume={30},
  year={2017}
}

@article{implicitregularizationindeepmatrixfactorization,
  title={Implicit regularization in deep matrix factorization},
  author={Arora, Sanjeev and Cohen, Nadav and Hu, Wei and Luo, Yuping},
  journal={Advances in neural information processing systems},
  volume={32},
  year={2019}
}

@article{Wustab,
  title={How {SGD} selects the global minima in over-parameterized learning: A dynamical stability perspective},
  author={Wu, Lei and Ma, Chao and others},
  journal={Advances in Neural Information Processing Systems},
  volume={31},
  year={2018}
}

@article{attentionisallyouneed,
  title={Attention is all you need},
  author={Vaswani, Ashish and Shazeer, Noam and Parmar, Niki and Uszkoreit, Jakob and Jones, Llion and Gomez, Aidan N and Kaiser, {\L}ukasz and Polosukhin, Illia},
  journal={Advances in neural information processing systems},
  volume={30},
  year={2017}
}

@article{alexnet,
  title={{ImageNet} classification with deep convolutional neural networks},
  author={Krizhevsky, Alex and Sutskever, Ilya and Hinton, Geoffrey E},
  journal={Advances in neural information processing systems},
  volume={25},
  year={2012}
}

@inproceedings{ethpaper,
  title={Closed form of the {Hessian} spectrum for some neural networks},
  author={Singh, Sidak Pal and Hofmann, Thomas},
  booktitle={High-dimensional Learning Dynamics 2024: The Emergence of Structure and Reasoning},
  year={2024}
}

@inproceedings{
identical4eq2,
title={Understanding Edge-of-Stability Training Dynamics with a Minimalist Example},
author={Xingyu Zhu and Zixuan Wang and Xiang Wang and Mo Zhou and Rong Ge},
booktitle={The Eleventh International Conference on Learning Representations },
year={2023}
}

@inproceedings{belkin2018doesdatainterpolationcontradict,
  title={Does data interpolation contradict statistical optimality?},
  author={Belkin, Mikhail and Rakhlin, Alexander and Tsybakov, Alexandre B},
  booktitle={The 22nd international conference on artificial intelligence and statistics},
  pages={1611--1619},
  year={2019},
  organization={PMLR}
}

@inproceedings{frei2025benignoverfittinglinearityneural,
  title={Benign overfitting without linearity: Neural network classifiers trained by gradient descent for noisy linear data},
  author={Frei, Spencer and Chatterji, Niladri S and Bartlett, Peter},
  booktitle={Conference on Learning Theory},
  pages={2668--2703},
  year={2022},
  organization={PMLR}
}

@article{Bartlett_2020,
  title={Benign overfitting in linear regression},
  author={Bartlett, Peter L and Long, Philip M and Lugosi, G{\'a}bor and Tsigler, Alexander},
  journal={Proceedings of the National Academy of Sciences},
  volume={117},
  number={48},
  pages={30063--30070},
  year={2020},
  publisher={National Academy of Sciences}
}

@article{mulayoff2,
  title={The implicit bias of minima stability: A view from function space},
  author={Mulayoff, Rotem and Michaeli, Tomer and Soudry, Daniel},
  journal={Advances in Neural Information Processing Systems},
  volume={34},
  pages={17749--17761},
  year={2021}
}

@book{strogatz2018nonlinear,
  title={Nonlinear dynamics and chaos: with applications to physics, biology, chemistry, and engineering},
  author={Strogatz, Steven H},
  year={2024},
  publisher={Chapman and Hall/CRC}
}

@article{nar2018stepsizemattersdeep,
  title={Step size matters in deep learning},
  author={Nar, Kamil and Sastry, Shankar},
  journal={Advances in Neural Information Processing Systems},
  volume={31},
  year={2018}
}

@article{ma2021linearstabilitysgdinputsmoothness,
  title={On linear stability of {SGD} and input-smoothness of neural networks},
  author={Ma, Chao and Ying, Lexing},
  journal={Advances in Neural Information Processing Systems},
  volume={34},
  pages={16805--16817},
  year={2021}
}

@inproceedings{nacson2023implicitbiasminimastability,
  author={Mor Shpigel Nacson and Rotem Mulayoff and Greg Ongie and Tomer Michaeli and Daniel Soudry},
  title={The Implicit Bias of Minima Stability in Multivariate Shallow {ReLU} Networks},
  year={2023},
  booktitle={ICLR},
}

@article{li2018visualizing,
  title={Visualizing the loss landscape of neural nets},
  author={Li, Hao and Xu, Zheng and Taylor, Gavin and Studer, Christoph and Goldstein, Tom},
  journal={Advances in neural information processing systems},
  volume={31},
  year={2018}
}

@article{goodfellow2014qualitatively,
  title={Qualitatively characterizing neural network optimization problems},
  author={Goodfellow, Ian J and Vinyals, Oriol and Saxe, Andrew M},
  journal={arXiv preprint arXiv:1412.6544},
  year={2014}
}

@inproceedings{sharpminimacangeneralize,
  title={Sharp minima can generalize for deep nets},
  author={Dinh, Laurent and Pascanu, Razvan and Bengio, Samy and Bengio, Yoshua},
  booktitle={International Conference on Machine Learning},
  pages={1019--1028},
  year={2017},
  organization={PMLR}
}

@article{marion2024deep,
  title={Deep linear networks for regression are implicitly regularized towards flat minima},
  author={Marion, Pierre and Chizat, L{\'e}na{\"\i}c},
  journal={Advances in Neural Information Processing Systems},
  volume={37},
  pages={76848--76900},
  year={2024}
}

@inproceedings{
ghosh2025learning,
title={Learning Dynamics of Deep Matrix Factorization Beyond the Edge of Stability},
author={Avrajit Ghosh and Soo Min Kwon and Rongrong Wang and Saiprasad Ravishankar and Qing Qu},
booktitle={The Thirteenth International Conference on Learning Representations},
year={2025},
}

@inproceedings{
cohen2021gradient,
title={Gradient Descent on Neural Networks Typically Occurs at the Edge of Stability},
author={Jeremy Cohen and Simran Kaur and Yuanzhi Li and J Zico Kolter and Ameet Talwalkar},
booktitle={International Conference on Learning Representations},
year={2021}
}

@inproceedings{
wang2022large,
title={Large Learning Rate Tames Homogeneity: Convergence and Balancing Effect},
author={Yuqing Wang and Minshuo Chen and Tuo Zhao and Molei Tao},
booktitle={International Conference on Learning Representations},
year={2022},
}

@article{ding2024flat,
  title={Flat minima generalize for low-rank matrix recovery},
  author={Ding, Lijun and Drusvyatskiy, Dmitriy and Fazel, Maryam and Harchaoui, Zaid},
  journal={Information and Inference: A Journal of the IMA},
  volume={13},
  number={2},
  pages={iaae009},
  year={2024},
  publisher={Oxford University Press}
}

@article{qiao2024stable,
  title={Stable minima cannot overfit in univariate {ReLU} networks: Generalization by large step sizes},
  author={Qiao, Dan and Zhang, Kaiqi and Singh, Esha and Soudry, Daniel and Wang, Yu-Xiang},
  journal={Advances in Neural Information Processing Systems},
  volume={37},
  pages={94163--94208},
  year={2024}
}

@inproceedings{
liang2025stable,
title={Stable Minima of {ReLU} Neural Networks Suffer from the Curse of Dimensionality: The Neural Shattering Phenomenon},
author={Tongtong Liang and Dan Qiao and Yu-Xiang Wang and Rahul Parhi},
booktitle={Advances in Neural Information Processing Systems},
year={2025},
}

@article{josz2025geometryflatminima,
  title={On the geometry of flat minima},
      author={Cédric Josz},
  journal={arXiv preprint arXiv:2509.11386},
  year={2025}
}

@inproceedings{
keskar2017on,
title={On Large-Batch Training for Deep Learning: Generalization Gap and Sharp Minima},
author={Nitish Shirish Keskar and Dheevatsa Mudigere and Jorge Nocedal and Mikhail Smelyanskiy and Ping Tak Peter Tang},
booktitle={International Conference on Learning Representations},
year={2017}
}

@inproceedings{
Jiang*2020Fantastic,
title={Fantastic Generalization Measures and Where to Find Them},
author={Yiding Jiang and Behnam Neyshabur and Hossein Mobahi and Dilip Krishnan and Samy Bengio},
booktitle={International Conference on Learning Representations},
year={2020}
}

@book{horn1994topics,
  title={Topics in matrix analysis},
  author={Horn, Roger A and Johnson, Charles R},
  year={1994},
  publisher={Cambridge university press}
}

@inproceedings{laurent2018deep,
  title={Deep linear networks with arbitrary loss: All local minima are global},
  author={Laurent, Thomas and Brecht, James},
  booktitle={International conference on machine learning},
  pages={2902--2907},
  year={2018},
  organization={PMLR}
}

@misc{
singh2026cracking,
title={Cracking the {Hessian}: Closed-Form {Hessian} Spectra for Fundamental Neural Networks},
author={Sidak Pal Singh and Weronika Ormaniec and Thomas Hofmann},
year={2026},
url={https://openreview.net/forum?id=XfO8npR1fb}
}

@article{chemnitz2025characterizing,
  title={Characterizing dynamical stability of stochastic gradient descent in overparameterized learning},
  author={Chemnitz, Dennis and Engel, Maximilian},
  journal={Journal of Machine Learning Research},
  volume={26},
  number={134},
  pages={1--46},
  year={2025}
}

@inproceedings{liu2021noisy,
  title={Noisy Gradient Descent Converges to Flat Minima for Nonconvex Matrix Factorization},
  author={Liu, Tianyi and Li, Yan and Wei, Song and Zhou, Enlu and Zhao, Tuo},
  booktitle={International Conference on Artificial Intelligence and Statistics},
  year={2021}
}

@article{gatmiry2023inductive,
  title={What is the inductive bias of flatness regularization? a study of deep matrix factorization models},
  author={Gatmiry, Khashayar and Li, Zhiyuan and Ma, Tengyu and Reddi, Sashank and Jegelka, Stefanie and Chuang, Ching-Yao},
  journal={Advances in Neural Information Processing Systems},
  volume={36},
  pages={28040--28052},
  year={2023}
}

@article{chou2024gradient,
  title={Gradient descent for deep matrix factorization: Dynamics and implicit bias towards low rank},
  author={Chou, Hung-Hsu and Gieshoff, Carsten and Maly, Johannes and Rauhut, Holger},
  journal={Applied and Computational Harmonic Analysis},
  volume={68},
  pages={101595},
  year={2024},
  publisher={Elsevier}
}

@article{lewkowycz2020large,
  title={The large learning rate phase of deep learning: the catapult mechanism},
  author={Lewkowycz, Aitor and Bahri, Yasaman and Dyer, Ethan and Sohl-Dickstein, Jascha and Gur-Ari, Guy},
  journal={arXiv preprint arXiv:2003.02218},
  year={2020}
}

@inproceedings{liang2025gradientdescentlargestep,
  title={Gradient descent with large step sizes: Chaos and fractal convergence region},
  author={Liang, Shuang and Mont{\'u}far, Guido},
  booktitle={International Conference on Learning Representations},
  year={2026}
}

@article{jules2023charting,
  title={Charting the topography of the neural network landscape with thermal-like noise},
  author={Jules, Theo and Brener, Gal and Kachman, Tal and Levi, Noam and Bar-Sinai, Yohai},
  journal={arXiv preprint arXiv:2304.01335},
  year={2023}
}

@article{lanczos1950iteration,
  title={An iteration method for the solution of the eigenvalue problem of linear differential and integral operators},
  author={Lanczos, Cornelius},
  journal={Journal of research of the National Bureau of Standards},
  volume={45},
  number={4},
  pages={255--282},
  year={1950}
}

@article{trager2019pure,
  title={Pure and spurious critical points: a geometric study of linear networks},
  author={Trager, Matthew and Kohn, Kathl{\'e}n and Bruna, Joan},
  journal={arXiv preprint arXiv:1910.01671},
  year={2019}
}

@article{baldi1989neural,
  title={Neural networks and principal component analysis: Learning from examples without local minima},
  author={Baldi, Pierre and Hornik, Kurt},
  journal={Neural networks},
  volume={2},
  number={1},
  pages={53--58},
  year={1989},
  publisher={Elsevier}
}

@article{kawaguchi2016deep,
  title={Deep learning without poor local minima},
  author={Kawaguchi, Kenji},
  journal={Advances in neural information processing systems},
  volume={29},
  year={2016}
}

\newpage
\appendix

\section{Matrix Calculus Background}
\label{appendix:prelim}
\subsection{Gâteaux Derivatives} 
Let $f:\mathbb{R}^{K \times L} \to \mathbb{R}$ be a differentiable function with continuous first- and second-order derivatives on $\mathbb{R}^{K \times L}$. Our objective is to derive closed-form expressions for the first- and second-order directional derivatives of $f$ in the direction of $\mathbf{U} \in \mathbb{R}^{K \times L}$, where $\norm{\mathbf{U}}_F < \infty$, denoted respectively by $D_{\mathbf{U}} f\left(\mathbf{X}\right)$ and $D_{\mathbf{U}}^2 f\left(\mathbf{X}\right)$. By the limit definition of the derivative, the first derivative of $f\left(\mathbf{X}\right)$ with respect to each entry of $\mathbf{X}$ can be expressed as follows:

\begin{equation}
    \frac{\partial f\left(\mathbf{X}\right)}{\partial X_{ij}} = \lim _{\Delta t \rightarrow 0} \frac{f\left(\mathbf{X}+ \Delta t \mathbf{e}_i \mathbf{e}_j^\top\right) - f\left(\mathbf{X}\right)}{\Delta t}, 
    \label{lim}
\end{equation}
where $\mathbf{e}_i$ is the $i$th standard basis vector of $\mathbb{R}^K$ and $\mathbf{e}_j$ is the $j$th standard basis vector of $\mathbb{R}^L$. If the limit in (\ref{lim}) exists, then by substitution of variables
\begin{equation}
    \frac{\partial f\left(\mathbf{X}\right)}{\partial X_{ij}} U_{ij}= \lim _{\Delta t \rightarrow 0} \frac{f\left(\mathbf{X}+ \Delta t U_{ij}\mathbf{e}_i \mathbf{e}_j^\top\right) - f\left(\mathbf{X}\right)}{\Delta t}.
    \label{lemma1proofreq}
\end{equation}
By definition, the total change in $f(\mathbf{X})$ in the direction of $\mathbf{U}$ is the sum of changes due to each entry of $\mathbf{X}$. Then 
\begin{align}
D_{\mathbf{U}} f(\mathbf{X}) &= \sum_{(i,j) \in [K] \times [L]}\frac{\partial f(\mathbf{X})}{\partial X_{ij}} U_{ij} \\
&= \sum_{(i,j) \in [K] \times [L]}\lim _{\Delta t \rightarrow 0} \frac{f(\mathbf{X}+ \Delta t U_{ij}\mathbf{e}_i \mathbf{e}_j^\top) - f\left(\mathbf{X}\right)}{\Delta t} \\
&= \lim _{\Delta t \rightarrow 0} \frac{f\left(\mathbf{X}+ \Delta t \mathbf{U}\right) - f\left(\mathbf{X}\right)}{\Delta t} \label{limtrick}.
\end{align}
We can rewrite (\ref{limtrick}) as follows:
\begin{equation}
    \lim _{\Delta t \rightarrow 0} \frac{f\left(\mathbf{X}+ \left(\Delta t+t\right) \mathbf{U}\right) - f\left(\mathbf{X}+t\mathbf{U}\right)}{\Delta t} \Biggl|_{t=0}.
\end{equation}
This expression can be rewritten as
\begin{equation}
    \frac{\partial f\left(\mathbf{X}+t\mathbf{U}\right)}{\partial t} \Biggl |_{t=0}.
\end{equation}
This is known as the \emph{Gâteaux derivative}, which represents the change in $f\left(\mathbf{X}\right)$ under a perturbation in the direction of $\mathbf{U}$. By the same reasoning, we obtain the following result.
\begin{lemma}
The second directional derivative of $f$ at $\mathbf{X}$ in the direction $\mathbf{U} \in \mathbb{R}^{K \times L}$ is given by
\begin{equation}
    D_{\mathbf{U}}^2 f\left(\mathbf{X}\right) = \frac{\partial^2}{\partial t^2} f\left(\mathbf{X} + t\mathbf{U}\right) \Big |_{t = 0}.
\end{equation}
\label{lemma1}
\end{lemma}

\begin{proof}
We can express the derivative of $\frac{\partial f(\mathbf{X})}{\partial X_{ij}}$ with respect to each entry of $\mathbf{X}$, using the limit definition of the derivative, as follows: 
\begin{equation}
    \frac{\partial^2 f(\mathbf{X})}{\partial X_{kl}\partial X_{ij}} = \frac{\partial}{\partial X_{kl}} \left( \frac{\partial f(\mathbf{X})}{\partial X_{ij}} \right) = \lim_{\Delta t \rightarrow 0 } \frac{\partial f(\mathbf{X}+\Delta t \mathbf{e}_i \mathbf{e}_j^\top)- \partial f(\mathbf{X})}{\partial X_{kl} \Delta t} \quad \forall(k,l) \in [K] \times [L]
\end{equation}
which is equal to 
\begin{equation}
\lim_{\Delta h, \Delta t \rightarrow 0} \frac{f(\mathbf{X}+\Delta t \mathbf{e}_i \mathbf{e}_j ^\top + \Delta h \mathbf{e}_k \mathbf{e}_l^\top) - f(\mathbf{X}+\Delta t \mathbf{e}_i \mathbf{e}_j ^\top)  - f(\mathbf{X}+\Delta h \mathbf{e}_k \mathbf{e}_l^\top)+f(\mathbf{X})}{\Delta h \Delta t}.   
\end{equation}
By using the substitution of variables as in (\ref{lemma1proofreq}),
\begin{equation}
   \frac{\partial^2 f(\mathbf{X})}{\partial X_{kl}\partial X_{ij}} U_{ij} U_{kl} =\frac{\partial}{\partial X_{kl}} \left( \frac{\partial f(\mathbf{X})}{\partial X_{ij}} U_{ij}\right) U_{kl} = \lim_{\Delta t \rightarrow 0 } \frac{\partial f(\mathbf{X}+\Delta t U_{ij} \mathbf{e}_i \mathbf{e}_j^\top)- \partial f(\mathbf{X})}{\partial X_{kl} \Delta t} U_{kl}.
   \label{42}
\end{equation}
Equivalently, 
\begin{equation}
\begin{aligned}
\lim_{\Delta h,\Delta t\to 0}
\frac{1}{\Delta h\,\Delta t}\Bigl[
& f\bigl(\mathbf{X}
  + \Delta t U_{ij}\mathbf{e}_i\mathbf{e}_j^\top
  + \Delta h U_{kl}\mathbf{e}_k\mathbf{e}_l^\top\bigr) \\
&- f\bigl(\mathbf{X}
  + \Delta t U_{ij}\mathbf{e}_i\mathbf{e}_j^\top\bigr)
 - f\bigl(\mathbf{X}
  + \Delta h U_{kl}\mathbf{e}_k\mathbf{e}_l^\top\bigr)
 + f(\mathbf{X})
\Bigr],
\end{aligned}
\end{equation}
which can be proved by substitution of variables in (\ref{42}). In turn, the second order differential in the direction $\mathbf{U} \in \mathbb{R}^{K \times L}$ is
\begin{equation}
    D_{\mathbf{U}}^2 f(\mathbf{X}) =\sum_{i,j} \sum_{k,l}  \frac{\partial^2 f(\mathbf{X})}{\partial X_{kl}\partial X_{ij}} U_{ij} U_{kl} = \left<\nabla \left< \nabla f(\mathbf{X}),\mathbf{U} \right>, \mathbf{U} \right>,
\end{equation}
where 
\begin{equation}
    \nabla f(\mathbf{X}) = \begin{bmatrix}
        \frac{\partial f(\mathbf{X})}{\partial X_{11}} &  \frac{\partial f(\mathbf{X})}{\partial X_{12}} & \dots & \frac{\partial f(\mathbf{X})}{\partial X_{1L}} \\ 
        \frac{\partial f(\mathbf{X})}{\partial X_{21}} &  \frac{\partial f(\mathbf{X})}{\partial X_{22}} & \dots & \frac{\partial f(\mathbf{X})}{\partial X_{2L}} \\
        \vdots &  \ddots   & \dots & \vdots \\
        \frac{\partial f(\mathbf{X})}{\partial X_{K1}} & \frac{\partial f(\mathbf{X})}{\partial X_{K2}}   & \dots & \frac{\partial f(\mathbf{X})}{\partial X_{KL}} \\
    \end{bmatrix} \in \mathbb{R}^{K \times L}.
\end{equation}
Equivalently,
\begin{align}
D_{\mathbf{U}}^2 f(\mathbf{X}) =&\sum_{k,l} \lim_{\Delta t \rightarrow 0 } \frac{\partial f(\mathbf{X}+\Delta t \mathbf{U} )- \partial f(\mathbf{X})}{\partial X_{kl} \Delta t} U_{kl} \\
=& \lim_{\Delta t \rightarrow 0} \frac{f(\mathbf{X}+2\Delta t \mathbf{U}) - 2f(\mathbf{X}+\Delta t \mathbf{U}) + f(\mathbf{X})}{\Delta t^2} \\ 
=& \frac{\partial^2}{\partial t^2}
   f(\mathbf{X} + t\mathbf{U})\biggr|_{t=0}\!.
\end{align}
\end{proof}
\subsection{Second Directional Derivatives and Maximum Eigenvalue}
Consider the following objective function for our real-valued matrix-variable function $\mathcal{L}$ such that 
\begin{equation}
    \mathcal{L}\left(\mathbf{W}_1,\cdots, \mathbf{W}_L\right) = \norm {\mathbf{M}-\mathbf{W}_L \mathbf{W}_{L-1}\cdots \mathbf{W}_1}_F^2,
    \label{directderiv}
\end{equation}
where $\mathbf{W}_{i} \in \mathbb{R}^{d_i \times d_{i-1}}$ for all $i \in [L]$, and $\mathbf{M} \in \mathbb{R}^{d_L \times d_0}$. In this setting, we define the maximum eigenvalue of the Hessian  $\nabla^{2} \mathcal{L}\left(\mathbf{W}_1,\cdots, \mathbf{W}_L\right)$ at an arbitrary point in the parameter space as follows: 
\begin{equation}
\lambda_{\max}(\nabla^2 \mathcal{L}\left(\mathbf{W}_1,\cdots, \mathbf{W}_L\right)) = 
    \max_{\substack{\mathbf{U}_1,\mathbf{U}_2,\cdots,\mathbf{U}_L :\\ \sum_{i=1}^{L} \norm{\mathbf{U}_i}_F^2 = 1}} \frac{d^2}{dt^2}\mathcal{L}\left(\mathbf{W}_1+t\mathbf{U}_1,\cdots,\mathbf{W}_L+t\mathbf{U}_L\right) \Bigr|_{t=0}. \label{rayleighh}
\end{equation}
This is the generalization of the Rayleigh quotient to the case where the Hessian is represented as a tensor and its eigenvectors take the form of matrices. This leads to the following lemma.
\begin{lemma}
For any $[\mathbf{W}^*_1,\cdots, \mathbf{W}^*_L]$ such that $\mathbf{M}= \prod_{j=1}^{L} \mathbf{W}^*_j$, the second directional derivative of $\mathcal{L}$ at $[\mathbf{W}^*_1,\cdots, \mathbf{W}^*_L]$ along $[\mathbf{U}_1,\cdots,\mathbf{U}_L]$ is given by
\begin{equation}
\nabla^2\mathcal{L}\left(\mathbf{W}^*_1,\cdots, \mathbf{W}^*_L\right)[\mathbf{U}_1,\cdots,\mathbf{U}_L] = 2\left \|\sum_{i=1}^{L} 
        \left[ \left(\prod_{j=i+1}^{L} \mathbf{W}^*_j \right)\mathbf{U}_i 
        \left(\prod_{j=1}^{i-1} \mathbf{W}_j^* \right) \right] \right \|_F ^2. \label{directionalderiv}
\end{equation}
\label{lemma2}
\end{lemma}
\begin{proof}
We can rewrite (\ref{directderiv}) using the definition of the \emph{Frobenius inner product} 

\begin{equation}
    \mathcal{L}(\mathbf{W}_1, \cdots, \mathbf{W}_L) = \left < \mathbf{M}- \prod_{i=1}^{L} \mathbf{W}_i, \mathbf{M}- \prod_{i=1}^{L} \mathbf{W}_i \right >.
\end{equation}
Then 
\begin{equation}
    \mathcal{L}(\mathbf{W}_1+t\mathbf{U}_1, \cdots, \mathbf{W}_L+t\mathbf{U}_L) = \left < \mathbf{M}- \prod_{i=1}^{L} (\mathbf{W}_i+t\mathbf{U}_i), \mathbf{M}- \prod_{i=1}^{L} (\mathbf{W}_i+t\mathbf{U}_i) \right>.
\end{equation}
Let's define $g:\mathbb{R} \rightarrow \mathbb{R}^{d_L \times d_0}$ such that 
\begin{equation}
    g(t) = \mathbf{M}-\prod_{i=1} ^ L (\mathbf{W}_i + t\mathbf{U}_i), \quad \mathcal{L}(\mathbf{W}_1+t\mathbf{U}_1, \cdots, \mathbf{W}_L+t\mathbf{U}_L) = \left < g(t), g(t)\right >.
\end{equation}
First, we need to differentiate $\mathcal{L}(\mathbf{W}_1+t\mathbf{U}_1, \cdots, \mathbf{W}_L+t\mathbf{U}_L)$ with respect to $t$. Using the fact that $\left <\mathbf{A},\mathbf{B} \right > = \operatorname{tr}(\mathbf{A}^\top \mathbf{B})$, which simplifies the differentiation, 
\begin{equation}
    \frac{d}{dt} \mathcal{L}(\mathbf{W}_1+t\mathbf{U}_1, \cdots, \mathbf{W}_L+t\mathbf{U}_L) = 2\left < g(t), g'(t)\right >.
\end{equation}
\begin{equation}
    \frac{d^2}{dt^2} \mathcal{L}(\mathbf{W}_1+t\mathbf{U}_1, \cdots, \mathbf{W}_L+t\mathbf{U}_L) = 2\left < g'(t), g'(t)\right > + 2\left< g(t), g''(t)\right > .
\end{equation}
Then, the second directional derivative $\nabla^2\mathcal{L}(\mathbf{W}_1, \cdots, \mathbf{W}_L)[\mathbf{U}_1,\cdots,\mathbf{U}_L]$ equals 
\begin{equation}
    \frac{d^2}{dt^2} \mathcal{L}(\mathbf{W}_1+t\mathbf{U}_1, \cdots, \mathbf{W}_L+t\mathbf{U}_L)\Bigr|_{t=0} = 2\left < g'(0), g'(0)\right > + 2\left < g(0), g''(0)\right>.
\end{equation}
Differentiating $g(t)$ gives 
\begin{align}
    g(0)   &= \mathbf{M}-\prod_{i=1}^L \mathbf{W}_i, \\
    g'(0)  &= - \sum_{i=1}^{L} 
        \left[ \left(\prod_{j=i+1}^{L} \mathbf{W}_j \right)\mathbf{U}_i 
        \left(\prod_{j=1}^{i-1} \mathbf{W}_j \right) \right], \\
    g''(0) &= -2 \sum_{1 \leq k < i \leq L} \left[ \left( \prod_{j=i+1}^{L} \mathbf{W}_j \right)\mathbf{U}_i \left ( \prod_{j=k+1} ^{i-1} \mathbf{W}_j \right) \mathbf{U}_k \left (\prod_{j=1}^{k-1} \mathbf{W}_j \right)\right].
\end{align}
Therefore,
\begin{align}
   &\nabla^2\mathcal{L}(\mathbf{W}_1, \cdots, \mathbf{W}_L)[\mathbf{U}_1,\mathbf{U}_2,\cdots, \mathbf{U}_L] =  \\
     &2\left <\sum_{i=1}^{L} 
        \left[ \left(\prod_{j=i+1}^{L} \mathbf{W}_j \right)\mathbf{U}_i 
        \left(\prod_{j=1}^{i-1} \mathbf{W}_j \right) \right],\sum_{i=1}^{L} 
        \left[ \left(\prod_{j=i+1}^{L} \mathbf{W}_j \right)\mathbf{U}_i 
        \left(\prod_{j=1}^{i-1} \mathbf{W}_j \right) \right]\right > \\
        &-4 \left<\mathbf{M}-\prod_{i=1}^L \mathbf{W}_i, \sum_{1 \leq k < i \leq L} \left [ \left( \prod_{j=i+1}^{L} \mathbf{W}_j \right)\mathbf{U}_i \left ( \prod_{j=k+1} ^{i-1} \mathbf{W}_j \right) \mathbf{U}_k \left ( \prod_{j=1}^{k-1} \mathbf{W}_j \right)\right]\right>.
\end{align}
Note that for any $\mathbf{w}^* = \operatorname{vec}([\mathbf{W}^*_1,\cdots, \mathbf{W}^*_L]) \in \Omega$, $\mathbf{M} = \prod_{j=1}^{L} \mathbf{W}^*_j$. Hence, for any global minimum 
\begin{align}
    &\nabla^2\mathcal{L}(\mathbf{W}^*_1,\cdots, \mathbf{W}^*_L)[\mathbf{U}_1,\mathbf{U}_2,\cdots, \mathbf{U}_L] = \hfill \\
    &2\left <\sum_{i=1}^{L} 
        \left[ \left(\prod_{j=i+1}^{L} \mathbf{W}^*_j \right)\mathbf{U}_i 
        \left(\prod_{j=1}^{i-1} \mathbf{W}_j^* \right) \right],\sum_{i=1}^{L} 
        \left[ \left(\prod_{j=i+1}^{L} \mathbf{W}_j^* \right)\mathbf{U}_i 
        \left(\prod_{j=1}^{i-1} \mathbf{W}^*_j \right) \right]\right >.
\end{align}
\end{proof}

\section{Alternative, Self-Contained Proof of Corollary~\ref{cor:warmup}}
\label{warmupproofappendix}
\begin{proof}
According to Lemma \ref{lemma2},
\begin{align}
    \lambda_{\max} \left( \nabla^2 \mathcal{L}(\mathbf{w}^*) \right) &= \max_{\substack{
    \mathbf{U}_1,\mathbf{U}_2,\dots,\mathbf{U}_L: \\
    \sum_{i=1}^{L} \norm{\mathbf{U}_i}_F^2 = 1
}} 2 \norm{\sum_{i=1}^{L} 
        \left[ \left(\prod_{j=i+1}^{L} \mathbf{W}^*_j \right)\mathbf{U}_i 
        \left(\prod_{j=1}^{i-1} \mathbf{W}_j^* \right) \right]}_F^2 \label{hope1}\\
        &\leq
\max_{\substack{
    \mathbf{U}_1,\mathbf{U}_2,\dots,\mathbf{U}_L: \\
    \sum_{i=1}^{L} \norm{\mathbf{U}_i}_F^2 = 1
}}
    2 \left( \sum_{i=1}^{L} 
        \norm{\left(\prod_{j=i+1}^{L} \mathbf{W}^*_j \right)\mathbf{U}_i 
        \left(\prod_{j=1}^{i-1} \mathbf{W}_j^* \right)}_F \right)^2 \label{hope2} \\
        &= \max_{\substack{
    \mathbf{u}_1,\mathbf{u}_2,\dots,\mathbf{u}_L: \\
    \sum_{i=1}^{L} \norm{\mathbf{u}_i}_2^2 = 1
}}
    2 \left( \sum_{i=1}^{L} 
        \norm{ \left[ \left(\prod_{j=1}^{i-1} \mathbf{W}_j^* \right)^\top \otimes \left(\prod_{j=i+1}^{L} \mathbf{W}^*_j \right) \right ]\mathbf{u}_i 
       }_2 \right)^2 \label{hope3} \\
       &\leq \max_{\substack{
    \mathbf{u}_1,\mathbf{u}_2,\dots,\mathbf{u}_L: \\
    \sum_{i=1}^{L} \norm{\mathbf{u}_i}_2^2 = 1
}}
    2 \left( \sum_{i=1}^{L} 
         \sigma_{\max}\left( \left(\prod_{j=1}^{i-1} \mathbf{W}_j^* \right)^\top \otimes \left(\prod_{j=i+1}^{L} \mathbf{W}^*_j \right) \right ) \norm{\mathbf{u}_i}_2
        \right)^2 \label{hope4} \\ 
        &= \max_{\substack{
    \mathbf{u}_1,\mathbf{u}_2,\dots,\mathbf{u}_L: \\
    \sum_{i=1}^{L} \norm{\mathbf{u}_i}_2^2 = 1
}}
    2 \left( \sum_{i=1}^{L} 
          \sigma_{\max}\left(\prod_{j=i+1}^{L} \mathbf{W}^*_j  \right ) \sigma_{\max}\left( \prod_{j=1}^{i-1} \mathbf{W}_j^* \right) \norm{\mathbf{u}_i}_2
        \right)^2, \label{hope5}
\end{align}
where $\mathbf{U}_i \in \mathbb{R}^{d_i \times d_{i-1}}$ and $\mathbf{u}_i \in \mathbb{R}^{d_id_{i-1}}$. We can upper bound the right-hand side of (\ref{hope1}) by using the \emph{triangle inequality}. By applying the \emph{vectorization trick} of the Kronecker product, we can rewrite (\ref{hope2}). Then, noting the fact that for any matrix $\mathbf{A} \in \mathbb{R}^{m \times n}$ and vector $\mathbf{x} \in \mathbb{R}^n$, $\norm{\mathbf{Ax}}_2 \leq \sigma_{\max}(\mathbf{A}) \norm{\mathbf{x}}_2$, we can upper bound the right-hand side of (\ref{hope3}). Note that for any matrices $\mathbf{A}$ and $\mathbf{B}$, $\sigma_{\max}(\mathbf{A} \otimes \mathbf{B} ) = \sigma_{\max}(\mathbf{A}) \sigma_{\max}(\mathbf{B})$. Hence, we can rewrite (\ref{hope4}). Then, by using the Cauchy–Schwarz inequality, 
\begin{equation}
    \lambda_{\max} \left( \nabla^2 \mathcal{L}(\mathbf{w}^*) \right) \leq 2 \sum_{i=1}^{L} \sigma_{\max} \left(\prod_{j=i+1}^{L}  \mathbf{W}_j^* \right ) ^2\sigma_{\max} \left(\prod_{j=1}^{i-1} \mathbf{W}_j^* \right )^2. \label{achievup}
\end{equation}   
Then, it suffices to show that there exists a direction $[\mathbf{U}_1^*, \mathbf{U}_2^*, \dots, \mathbf{U}_L^*]$ along which the bound in~(\ref{achievup}) is achieved. 

Consider the singular value decomposition (SVD) of $\prod_{j=i+1}^{L} \mathbf{W}^*_j $, and denote by $\mathbf{u}^{L}_{i} $ and $\mathbf{v}^{L}_{i} $ the left and right singular vectors of $\prod_{j=i+1}^{L} \mathbf{W}^*_j $ corresponding to its largest singular value, respectively. Note that since $\mathbf{w}_L^\top$ is a row vector, we can choose $\mathbf{u}^{L}_{i} = 1$ for all $i \in [L]$. Moreover, decompose $\prod_{j=1}^{i-1} \mathbf{W}_j^* $ by SVD, and denote by $\mathbf{u}^{R}_{i}$ and $\mathbf{v}^{R}_{i} $ the left and right singular vectors of $\prod_{j=1}^{i-1} \mathbf{W}_j^*$ corresponding to the largest singular value, respectively. Note that since $\mathbf{w}_1$ is a vector, we can choose $\mathbf{v}^{R}_{i}  = 1$ for all $i \in [L]$. Determine a particular direction $[\mathbf{U}_1^*, \mathbf{U}_2^*, \dots, \mathbf{U}_L^*]$ that achieves the upper bound while satisfying the constraint $\sum_{i=1}^L\norm{\mathbf{U}_i^*}_F^2= 1$. Suppose 
\begin{equation}
    \mathbf{U}_i^* = \frac{\sigma_{\max}\left(\prod_{j=i+1}^{L} \mathbf{W}^*_j  \right) \sigma_{\max}\left(\prod_{j=1}^{i-1} \mathbf{W}_j^* \right)}{\sqrt{\sum_{i=1}^{L} 
          \sigma_{\max}\left(\prod_{j=i+1}^{L} \mathbf{W}^*_j  \right)^2 \sigma_{\max}\left( \prod_{j=1}^{i-1} \mathbf{W}_j^* \right) ^2}}\mathbf{v}^{L}_{i}  {\mathbf{u}^{R}_{i}}^\top.
\end{equation}
Then, 
\begin{equation}
    2 \left \|\sum_{i=1}^{L} 
        \left[ \left(\prod_{j=i+1}^{L} \mathbf{W}^*_j \right)\mathbf{U}^*_i 
        \left(\prod_{j=1}^{i-1} \mathbf{W}_j^* \right) \right]\right \|_F^2 = 2 \sum_{i=1}^{L} \sigma_{\max} \left(\prod_{j=i+1}^{L}  \mathbf{W}_j^* \right)^2 \sigma_{\max} \left (\prod_{j=1}^{i-1} \mathbf{W}_j^* \right)^2.
\end{equation}
Since the upper bound is achieved, it implies
\begin{equation}
    \lambda_{\max}(\nabla^2 \mathcal{L}\left(\mathbf{w}^*\right)) = 2 \sum_{i=1}^{L} \sigma_{\max} \left(\prod_{j=i+1}^{L}  \mathbf{W}_j^* \right)^2 \sigma_{\max} \left(\prod_{j=1}^{i-1} \mathbf{W}_j^* \right)^2.
\end{equation}
\end{proof}

\section{Proofs from Section~\ref{sec:main}}

\subsection{Proof of Theorem \ref{theorem:deepmatrixfactorization}}
\label{appendix:deepmatrixfactorization}

\begin{proof}
From (\ref{rayleighh}) and (\ref{directionalderiv}), 
\begin{align}
\lambda_{\max}(\nabla^2 \mathcal{L}(\mathbf{w}^*)) &= \max_{\substack{
    \mathbf{U}_1,\mathbf{U}_2,\dots,\mathbf{U}_L: \\
    \sum_{i=1}^{L} \norm{\mathbf{U}_i}_F^2 = 1}}   2\norm{\sum_{i=1}^{L} 
        \left[ \left(\prod_{j=i+1}^{L} \mathbf{W}^*_j \right)\mathbf{U}_i 
        \left(\prod_{j=1}^{i-1} \mathbf{W}_j^* \right) \right]}_F ^2 \label{def}\\
        &= \max_{\substack{
    \mathbf{U}_1,\mathbf{U}_2,\dots,\mathbf{U}_L: \\
    \sum_{i=1}^{L} \norm{\mathbf{U}_i}_F^2 = 1}} 2\norm{\operatorname{vec} \left(\sum_{i=1}^{L} 
        \left[ \left(\prod_{j=i+1}^{L} \mathbf{W}^*_j \right)\mathbf{U}_i 
        \left(\prod_{j=1}^{i-1} \mathbf{W}_j^* \right) \right] \right)}_2 ^2 \label{defvec}\\
        &=\max_{\substack{
    \mathbf{U}_1,\mathbf{U}_2,\dots,\mathbf{U}_L: \\
    \sum_{i=1}^{L} \norm{\mathbf{U}_i}_F^2 = 1}} 2\norm{\sum_{i=1}^{L} 
        \operatorname{vec} \left( \left(\prod_{j=i+1}^{L} \mathbf{W}^*_j \right)\mathbf{U}_i 
        \left(\prod_{j=1}^{i-1} \mathbf{W}_j^* \right) \right)}_2 ^2 \label{linearity}\\
        &= \max_{\substack{
    \mathbf{u}_1,\mathbf{u}_2,\dots,\mathbf{u}_L: \\
    \sum_{i=1}^{L} \norm{\mathbf{u}_i}_2^2 = 1
}} 2\norm{\sum_{i=1}^L \left [\left(\prod_{j=1}^{i-1} \mathbf{W}_j^* \right)^\top \otimes \left(\prod_{j=i+1}^{L} \mathbf{W}^*_j \right) \right] \mathbf{u}_i}_2^2. \label{vectonew}
\end{align}
Note that $\operatorname{vec}$ is a linear operator. Therefore, (\ref{defvec}) can be rewritten as (\ref{linearity}). Then, by using the vectorization trick of the Kronecker product, we can obtain (\ref{vectonew}). Let's define a block matrix and a vector such that
\begin{equation}
    \mathbf{K} =
    \begin{bmatrix}
        \mathbf{I} \otimes \prod_{j=2}^{L} \mathbf{W}^*_j \big| {\mathbf{W}_1^*}^\top \otimes \left(\prod_{j=3}^{L} \mathbf{W}^*_j \right) \big| & \dots &  \big| \left(\prod_{j=1}^{L-1} \mathbf{W}_j^* \right)^\top \otimes \mathbf{I} 
    \end{bmatrix},
\end{equation}
\begin{equation}
    \mathbf{u} = [\mathbf{u}_1^\top \quad \mathbf{u}_2^\top \quad \cdots \quad \mathbf{u}_L^\top]^\top.
\end{equation}
Then 
\begin{align}
\lambda_{\max}(\nabla ^2 \mathcal{L}(\mathbf{w}^*)) &= \max_{\mathbf{u}: \norm{\mathbf{u}}_2 = 1}   2 \norm{\mathbf{K}\mathbf{u}}_2^2 \\
&= 2\sigma_{\max}(\mathbf{K}^\top \mathbf{K}).
\end{align}
Note that $\sigma_{\max}(\mathbf{K}^\top \mathbf{K}) = \sigma_{\max}(\mathbf{K}\mathbf{K}^\top)$, and for any two block matrices $\mathbf{A}$ and $\mathbf{B}$ such that
\begin{equation}
 \mathbf{A} = 
    \begin{bmatrix}
\mathbf{A}_1 \quad \mathbf{A}_2 \quad \dots \quad \mathbf{A}_L
\end{bmatrix} \in \mathbb{R}^{M_1 \times d}, 
\quad
\mathbf{B} =
\begin{bmatrix}
\mathbf{B}_1 \\
\vdots \\
\mathbf{B}_L \\
\end{bmatrix} \in \mathbb{R}^{d \times M_2} 
\end{equation}
\begin{equation}
    \mathbf{AB} = 
    \sum_{i=1}^L \mathbf{A}_i \mathbf{B}_i, \quad \mathbf{AB}\in \mathbb{R}^{M_1 \times M_2}.
\end{equation}
Furthermore, for any matrices $\mathbf{A}, \mathbf{B}, \mathbf{C}, \mathbf{D}$ such that the matrix products $\mathbf{AB}$ and $\mathbf{CD}$ are well defined, we have
\begin{equation}
(\mathbf{A} \otimes \mathbf{C})(\mathbf{B} \otimes \mathbf{D}) = \mathbf{AB} \otimes \mathbf{CD}.
\end{equation}
Using the fact that $(\mathbf{A} \otimes \mathbf{C})^\top = \mathbf{A}^\top \otimes \mathbf{C}^\top$ together with the previous property, it follows that
\begin{equation}
\lambda_{\max} \left( \nabla^2 \mathcal{L}(\mathbf{w}^*) \right) = 2\sigma_{\max} \left(\sum_{i=1}^L \mathbf{B}_i^\top \mathbf{B}_i \otimes \mathbf{A}_i \mathbf{A}_i^\top\right),
\end{equation}
where $\mathbf{A}_k = \prod_{i = k+1}^L \mathbf{W}_i^*$ and $\mathbf{B}_k = \prod_{i = 1}^{k-1} \mathbf{W}_i^*$.

\end{proof}

\subsection{Proof of Corollary~\ref{remark:twolayer}}
\label{appendix:remarktwolayer}

\begin{proof}
    According to Lemma \ref{lemma2}, for any $(\mathbf{L}^*,\mathbf{R}^*) \in \Omega$,
\begin{equation}
    \nabla^2\mathcal{L}(\mathbf{L}^*,\mathbf{R}^*)[\mathbf{V},\mathbf{U}] = 2 \norm{\mathbf{L}^* \mathbf{U}^{\top} + \mathbf{V}\mathbf{R}^{*\top}}_F^2.
\end{equation}
Then
\begin{align}
  \lambda_{\max} (\nabla^2 \mathcal{L}(\mathbf{L}^*,\mathbf{R}^*)) &= \max_{\substack{\mathbf U,\mathbf V \\ \|\mathbf U\|_F^2+\|\mathbf V\|_F^2=1}} 2 \norm{\mathbf{L}^* \mathbf{U}^{\top} + \mathbf{V}\mathbf{R}^{*\top}}_F^2 \label{LUVRT}    \\
  &\leq \max_{\substack{\mathbf U,\mathbf V \\ \|\mathbf U\|_F^2+\|\mathbf V\|_F^2=1}} 2\left(\norm{\mathbf{L}^*\mathbf{U}^\top}_F + \norm{\mathbf{V}\mathbf{R}^{*\top}}_F  \right)^2 \label{rewritehand}\\ 
  &=\max_{\substack{\mathbf u,\mathbf v \\ \|\mathbf u\|_2^2+\|\mathbf v\|_2^2=1}}
    2\left(\|(\mathbf I\otimes \mathbf L^*)\mathbf u\|_2+\|(\mathbf R^*\otimes \mathbf I)\mathbf v\|_2\right)^2 \label{22}\\
    &\leq \max_{\substack{\mathbf u,\mathbf v \\ \|\mathbf u\|_2^2+\|\mathbf v\|_2^2=1}} 2\left(\sigma_{\max}(\mathbf I \otimes \mathbf L^*)\|\mathbf u\|_2+\sigma_{\max}(\mathbf R^* \otimes \mathbf I)\|\mathbf v\|_2\right)^2.
\end{align}
We can upper bound the right-hand side of (\ref{LUVRT}) using the \emph{triangle inequality}. By applying the \emph{vectorization trick} of the Kronecker product again, we can rewrite (\ref{rewritehand}). Then, noting that for any matrix $\mathbf{A} \in \mathbb{R}^{m \times n}$ and vector $\mathbf{x} \in \mathbb{R}^n$, $\norm{\mathbf{Ax}}_2 \leq \sigma_{\max}(\mathbf{A}) \norm{\mathbf{x}}_2$, we can upper bound the right-hand side of (\ref{22}). Note that for any matrices $\mathbf{A}$ and $\mathbf{B}$, $\sigma_{\max}(\mathbf{A} \otimes \mathbf{B} ) = \sigma_{\max}(\mathbf{A}) \sigma_{\max}(\mathbf{B})$. Hence,
\begin{equation}
     \lambda_{\max} (\nabla^2 \mathcal{L}(\mathbf{L}^*,\mathbf{R}^*)) \leq \max_{\substack{\mathbf u,\mathbf v \\ \|\mathbf u\|_2^2+\|\mathbf v\|_2^2=1}} 2\left(\sigma_{\max}{( \mathbf{L}^*)} \norm{\mathbf{u}}_2 + \sigma_{\max}(\mathbf{R}^{*}) \norm{\mathbf{v}}_2  \right)^2.
\end{equation}
Then, by using the Cauchy-Schwarz inequality,
\begin{equation}
\lambda_{\max}\!\left(\nabla^2 \mathcal{L}(\mathbf L^*,\mathbf R^*)\right)
 \;\leq\; 2\left(\sigma_{\max}(\mathbf L^*)^2+\sigma_{\max}(\mathbf R^*)^2\right).
\end{equation}
Now, we will show that this upper bound is achievable. Let us decompose $\mathbf{L}^*$ as $\mathbf{U}_L \mathbf{\Sigma}_L \mathbf{V}_L ^\top$ by SVD, and denote by $\mathbf{u}_L$ and $\mathbf{v}_L$ the left and right singular vectors corresponding to its largest singular value, respectively. Moreover, decompose $\mathbf{R}^{*}$ as $\mathbf{U}_R \mathbf{\Sigma}_R \mathbf{V}_R^\top$ by SVD, and denote by $\mathbf{u}_R$ and $\mathbf{v}_R$ the left and right singular vectors corresponding to its largest singular value, respectively. We determine a particular $(\mathbf{U}^*,\mathbf{V}^{*})$ such that it achieves the upper bound while satisfying the constraint $\norm{\mathbf{U}^*}_F^2 + \norm{\mathbf{V}^*}_F^2 = 1$. Choose
\begin{equation}
    \mathbf{U}^{*\top} = \frac{\sigma_{\max}(\mathbf{L}^*)}{\sqrt{\sigma_{\max}(\mathbf{L}^*)^2+\sigma_{\max}(\mathbf{R}^*)^2}}\mathbf{v}_L \mathbf{u}_R^\top, \quad \mathbf{V}^* = \frac{\sigma_{\max}(\mathbf{R}^*)}{\sqrt{\sigma_{\max}(\mathbf{L}^*)^2+\sigma_{\max}(\mathbf{R}^*)^2}} \mathbf{u}_L \mathbf{v}_R^\top.
\end{equation}
Using the identity
$\norm{\mathbf{x}\mathbf{y}^{\top}}_F^2
= \norm{\mathbf{x}}_2^2 \norm{\mathbf{y}}_2^2$,
we obtain
\begin{equation}
2\norm{\frac{(\sigma_{\max}(\mathbf{L}^*)^2 + \sigma_{\max}(\mathbf{R}^*)^2) }{\sqrt{\sigma_{\max}(\mathbf{L}^*)^2 + \sigma_{\max}(\mathbf{R}^*)^2}} \mathbf{u}_L \mathbf{u}_R ^\top}_F^2 = 2(\sigma_{\max}(\mathbf{L}^*)^2 + \sigma_{\max}(\mathbf{R}^*)^2).
\end{equation} 
\end{proof}

\section{On the Singular Vector Stationary Set}
\label{app:svs}

In this section, we show that the characterization of the singular vector stationary (SVS) set in \citet[Proposition 1]{ghosh2025learning} does not hold in general. \citet{ghosh2025learning} considered the deep matrix factorization objective in \eqref{oop} with $d_0 = d_1 = \cdots = d_L = d$ and a target matrix $\mathbf{M}$ with distinct singular values.\footnote{Their objective has an additional factor of $1/2$, which does not affect the arguments below.}

Consider gradient flow (GF), $\dot{\mathbf{W}}_i(t) = -\nabla_{\mathbf{W}_i} \mathcal{L}(\mathbf{w}(t))$, and let $\mathbf{M} = \mathbf{U}_{\mathbf{M}} \boldsymbol{\Sigma}_{\mathbf{M}} \mathbf{V}_{\mathbf{M}}^\top$ and $\mathbf{W}_i(t) = \mathbf{U}_{\mathbf{W}_i}(t) \boldsymbol{\Sigma}_{\mathbf{W}_i}(t) \mathbf{V}_{\mathbf{W}_i}(t)^\top$ be SVDs for all $i \in [L]$. \citet{ghosh2025learning} state in their Proposition 1 that, for any $t \geq 0$, if $\dot{\mathbf{U}}_{\mathbf{W}_i}(t) = \dot{\mathbf{V}}_{\mathbf{W}_i}(t) = \mathbf{0}$ for all $i \in [L]$, then
\begin{equation}
    \mathbf{U}_{\mathbf{W}_L}(t) = \mathbf{U}_{\mathbf{M}}, \qquad
    \mathbf{V}_{\mathbf{W}_1}(t) = \mathbf{V}_{\mathbf{M}}, \qquad
    \mathbf{V}_{\mathbf{W}_{i+1}}(t) = \mathbf{U}_{\mathbf{W}_i}(t)
    \label{eq:svs}
\end{equation}
for every $i \in [L-1]$. The set of points that satisfy \eqref{eq:svs} is called the SVS set. 

Suppose that each $\mathbf{W}_i(t)$ has full rank with distinct singular values. Then, by \citet[Equations (22)--(25)]{ghosh2025learning}, $\dot{\mathbf{U}}_{\mathbf{W}_i}(t) = \dot{\mathbf{V}}_{\mathbf{W}_i}(t) = \mathbf{0}$ for all $i \in [L]$ if and only if
\begin{equation}
    \mathbf{U}_{\mathbf{W}_i}(t)^\top \nabla_{\mathbf{W}_i} \mathcal{L}(\mathbf{w}(t)) \mathbf{V}_{\mathbf{W}_i}(t)
    \label{eq:svs-stationary}
\end{equation}
is diagonal for all $i \in [L]$, where for any $\mathbf{w}$, $\nabla_{\mathbf{W}_i} \mathcal{L}(\mathbf{w}) = -2 \mathbf{A}_i^\top \big(\mathbf{M} - \prod_{j=1}^{L} \mathbf{W}_j\big) \mathbf{B}_i^\top$, with $\mathbf{A}_i$ and $\mathbf{B}_i$ as in Theorem~\ref{theorem:deepmatrixfactorization}. The same condition governs GD: If $\nabla_{\mathbf{W}_i} \mathcal{L}(\mathbf{w}) = \mathbf{U}_{\mathbf{W}_i} \mathbf{D}_i \mathbf{V}_{\mathbf{W}_i}^\top$ with $\mathbf{D}_i$ diagonal, then
\begin{equation}
    \mathbf{W}_i - \eta \nabla_{\mathbf{W}_i} \mathcal{L}(\mathbf{w}) = \mathbf{U}_{\mathbf{W}_i} \left( \boldsymbol{\Sigma}_{\mathbf{W}_i} - \eta \mathbf{D}_i \right) \mathbf{V}_{\mathbf{W}_i}^\top,
\end{equation}
so a GD step leaves the singular vectors unchanged, up to sign and permutation ambiguity.

\citet[Proposition 1]{ghosh2025learning} states that \eqref{eq:svs-stationary} implies \eqref{eq:svs}. We give a counterexample with $L = 2$ and $d_0 = d_1 = d_2 = 2$, where both GF and GD are initialized at $\mathbf{w}(0)$. The counterexample was suggested by GPT-6 Astra, and we take full responsibility for its correctness. In the counterexample, every factor has full rank with distinct singular values; therefore, the singular vectors are unique up to sign and permutation ambiguity.

\paragraph{Counterexample.} Consider
\begin{equation}
    \mathbf{W}_1(0) = \begin{pmatrix} 2 & 0 \\ 0 & 1 \end{pmatrix}, \qquad
    \mathbf{W}_2(0) = \begin{pmatrix} 2 & 1 \\ 1 & 2 \end{pmatrix}, \qquad
    \mathbf{M} = \begin{pmatrix} 3 & 2 \\ 5/2 & 0 \end{pmatrix}.
\end{equation}
Then $\mathbf{M} - \mathbf{W}_2(0) \mathbf{W}_1(0) = \begin{pmatrix} -1 & 1 \\ 1/2 & -2 \end{pmatrix} \neq \mathbf{0}$, so $\mathbf{w}(0) \notin \Omega$, and
\begin{equation}
    \nabla_{\mathbf{W}_1} \mathcal{L}(\mathbf{w}(0)) = \begin{pmatrix} 3 & 0 \\ 0 & 6 \end{pmatrix}, \qquad
    \nabla_{\mathbf{W}_2} \mathcal{L}(\mathbf{w}(0)) = \begin{pmatrix} 4 & -2 \\ -2 & 4 \end{pmatrix}.
\end{equation}
Here, $\mathbf{U}_{\mathbf{W}_1}(0) = \mathbf{V}_{\mathbf{W}_1}(0) = \mathbf{I}_{2\times 2}$ with $\boldsymbol{\Sigma}_{\mathbf{W}_1}(0) = \operatorname{diag}(2, 1)$, and
\begin{equation}
    \mathbf{U}_{\mathbf{W}_2}(0) = \mathbf{V}_{\mathbf{W}_2}(0) = \frac{1}{\sqrt{2}} \begin{pmatrix} 1 & 1 \\ 1 & -1 \end{pmatrix}, \qquad
    \boldsymbol{\Sigma}_{\mathbf{W}_2}(0) = \operatorname{diag}(3, 1).
\end{equation}
Thus,
\begin{align}
\mathbf{U}_{\mathbf{W}_1}(0)^\top \nabla_{\mathbf{W}_1} \mathcal{L}(\mathbf{w}(0)) \mathbf{V}_{\mathbf{W}_1}(0) &= \operatorname{diag}(3, 6), \\
    \mathbf{U}_{\mathbf{W}_2}(0)^\top \nabla_{\mathbf{W}_2} \mathcal{L}(\mathbf{w}(0)) \mathbf{V}_{\mathbf{W}_2}(0) &= \operatorname{diag}(2, 6),
\end{align}
so \eqref{eq:svs-stationary} holds at $t = 0$. However, $\mathbf{V}_{\mathbf{W}_2}(0) \neq \mathbf{U}_{\mathbf{W}_1}(0)$. Moreover,
\begin{equation}
    \mathbf{M}^\top \mathbf{M} = \begin{pmatrix} 61/4 & 6 \\ 6 & 4 \end{pmatrix}
\end{equation}
is not diagonal, so $\mathbf{V}_{\mathbf{W}_1}(0) \neq \mathbf{V}_{\mathbf{M}}$. Therefore, $\mathbf{w}(0)$ is not in the SVS set.

We note that the converse implication does hold. Under \eqref{eq:svs}, $\mathbf{A}_i$, $\mathbf{B}_i$, and $\mathbf{M} - \prod_{j=1}^{L} \mathbf{W}_j$ share the aligned singular vectors, so $\nabla_{\mathbf{W}_i} \mathcal{L}(\mathbf{w}) = \mathbf{U}_{\mathbf{W}_i} \mathbf{D}_i \mathbf{V}_{\mathbf{W}_i}^\top$ for some diagonal $\mathbf{D}_i$. Hence, the family obtained by fixing these singular-vector bases and allowing arbitrary diagonal coefficients is invariant under GD, which is what \citet[Propositions 3 and 4]{ghosh2025learning} use. However, it is not a characterization of the points at which the singular vectors are stationary.
\section{Additional Experimental Results}
\label{expappendix}

\subsection{Visualization of the Contour Map of the Loss Landscape}
\label{visualization}
To study the dynamics of deep matrix factorization, we analyze the trajectories of GD. Previous works have visualized neural network loss landscapes to explore their highly non-convex and non-Euclidean structure \citep{goodfellow2014qualitatively,li2018visualizing}. However, the high dimensionality prevents full visualization. As a result, only 1-D (line) or 2-D (surface) visualizations are available. In this paper, we focus on contour maps of the loss landscape in the vicinity of a global minimum.

\textbf{Contour Plots with Random Projections.}
We want to visualize the loss landscape around a global minimum $\mathbf{w}^* \in \mathbb{R}^N$. We select two random vectors, $\mathbf{\zeta}$ and $\mathbf{\gamma}$, from $\mathbb{R}^N$. Then, for any $K \subset \mathbb{R}^2$, we can define the function $p: K \rightarrow \mathbb{R}:$ 
\begin{equation}
    p(x,y) = \mathcal{L}(\mathbf{w}^* + x\mathbf{\zeta} + y\mathbf{\gamma}) \quad \forall (x,y) \in K
\end{equation}
and plot $p$ with a desired resolution.

\textbf{Scale Invariance and Manifolds}. Note that our loss function is \emph{scale-invariant}, which means that for any nonzero scalar $c \in \mathbb{R}$, multiplying one layer by $c$ and the next layer by $1/c$, or vice versa, yields the same end-to-end function. This phenomenon forms a manifold for global minimizers in the loss landscape \citep{sharpminimacangeneralize}. Furthermore, we know that when $d_0 = d_L$, $\nabla^{2} \mathcal{L}(\mathbf{w}^*)$ is rank-deficient by at least the order of $1-1/L$; that is, at least $1-1/L$ of the eigenvalues of $\nabla^{2} \mathcal{L}(\mathbf{w}^*)$ are zero \citep{mulayoff}.

\textbf{Projection onto the Hessian Eigenvectors.}
If we use random projections in visualizations, plots might not be informative for tracking the optimization dynamics of GD due to the phenomenon caused by the scale invariance. To make contour maps as informative as possible, we choose $\mathbf\zeta$ and $\mathbf\gamma$ to be $\mathbf v_1$ and $\mathbf v_N$, respectively — the eigenvectors of the largest and smallest eigenvalues of $\nabla^{2} \mathcal{L}(\mathbf{w}^*)$.

\subsection{Experiment Details and Additional Experiments}

\label{additionalexperiments}
In our experimental setting, we first generate a minimizer $\mathbf{w}^*$ by randomly selecting the layer dimensions and then sampling Gaussian factor matrices of those dimensions. Then we construct $\mathbf{M}$ or $m$ by simply taking the product of the generated Gaussian matrices. To observe the escape phenomenon, we need to chart the loss landscape around the minimizer $\mathbf{w}^*$; however, the parameter space is high-dimensional, and this prevents a full visualization. Therefore, the common practice is to project the loss landscape onto two vectors to chart a 2-D contour plot as we mentioned in Appendix \ref{visualization}. We select these vectors as eigenvectors of $\nabla^2 \mathcal{L}(\mathbf{w}^*)$ that correspond to the maximum and minimum eigenvalues of $\nabla^2 \mathcal{L}(\mathbf{w}^*)$. Note that we could have also considered a random perturbation direction sampled uniformly from $\mathbb{S}^{N-1}$. For instance, \citet{mulayoff} showed that for the deep matrix factorization problem where $d_0 = d_L$, which has the same Hessian structure as that of deep linear networks with a full-rank data matrix, the Hessian is rank-deficient at all minima by at least the order of $1 - (1/L)$. This means that at least $1 - (1/L)$ of the eigenvalues are zero at a minimum. Thus, the probability of choosing a direction that is orthogonal to the eigenspace corresponding to eigenvalues larger than $2/\eta$ is 0. Therefore, random perturbations would lead to the escape phenomenon with probability $1$. On the other hand, for stable minima, it is important to choose the direction as the eigenvector of $\nabla^2 \mathcal{L}(\mathbf{w}^*)$ that corresponds to the largest eigenvalue. The reasoning is the same: if you choose a direction that is not orthogonal to the eigenspace corresponding to the eigenvalues that are zero, then the linearized GD dynamics cannot converge to $\mathbf{w}^*$. This means that if we choose the perturbation direction randomly, then with probability $1$, we choose a direction that is not orthogonal to the eigenspace corresponding to the eigenvalues that are zero. Therefore, for the experiment, it is convenient to choose the perturbation direction as the eigenvector of $\nabla^2 \mathcal{L}(\mathbf{w}^*)$ corresponding to the largest eigenvalue. Note that the convergence we refer to here is the convergence to the specific global minimizer, not to an arbitrary global minimizer.

To track the training dynamics in the projected space, we project the iterates onto the subspace spanned by these two eigenvectors and indicate each iterate on the 2-D contour plot. We then repeat the experiments while varying the depth, dimensions, and initialization distance $r$ as shown in Figs.~\ref{figure3:threelayers}–\ref{fig:figure4}. We note that oscillations occur along the eigenvector corresponding to the maximum eigenvalue of the Hessian. 

For the scalar factorizations of Figures~\ref{fig:figure5} and~\ref{fig:figure3}, the hidden widths are drawn uniformly at random from $[24]$ before the factors are sampled. This yields \((d_0,d_1,d_2,d_3,d_4)=(1,8,4,9,1)\) for the four-layer experiment, i.e., \(N=85\) parameters, and \(d_0=d_{20}=1\) with
\[
(d_1,\dots,d_{19})=(12,5,16,17,18,3,15,3,3,22,3,21,11,17,16,24,8,4,9)
\]
for the twenty-layer experiment, i.e., \(N=2421\) parameters. In all experiments, the entries of the factors $\mathbf{W}_1^*, \mathbf{W}_2^*, \ldots, \mathbf{W}_L^*$ are drawn independently from the standard Gaussian distribution \(\mathcal{N}(0,1)\) without any rescaling. The random instances are generated with NumPy, using \texttt{numpy.random.default\_rng(12)} for Figure~\ref{figure3:threelayers} and \texttt{numpy.random.RandomState(100)} for Figures~\ref{fig:figure2}--\ref{fig:figure4}. All computations are performed in IEEE~754 double precision.

\begin{figure}[!htb]
    \centering
    \begin{subfigure}[t]{0.32\textwidth}
        \centering
        \includegraphics[width=\textwidth]{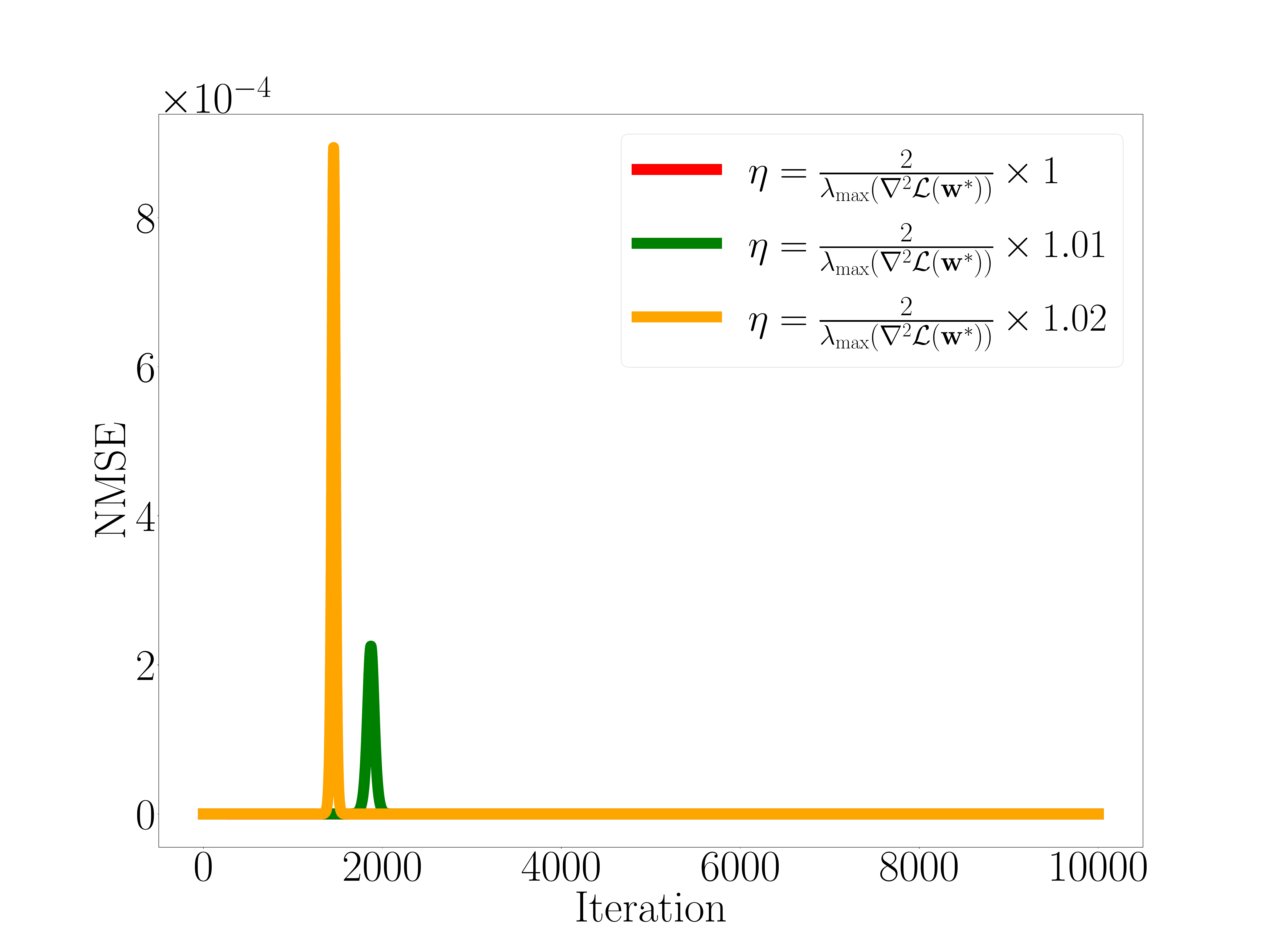}
        \caption{Normalized training error across iterations, i.e., $\mathcal{L}(\mathbf{w}_{t})/\|\mathbf{M}\|_F^2$.}
        \label{fig:figure2a}
    \end{subfigure}
    \hfill
    \begin{subfigure}[t]{0.32\textwidth}
        \centering
        \includegraphics[width=\textwidth]{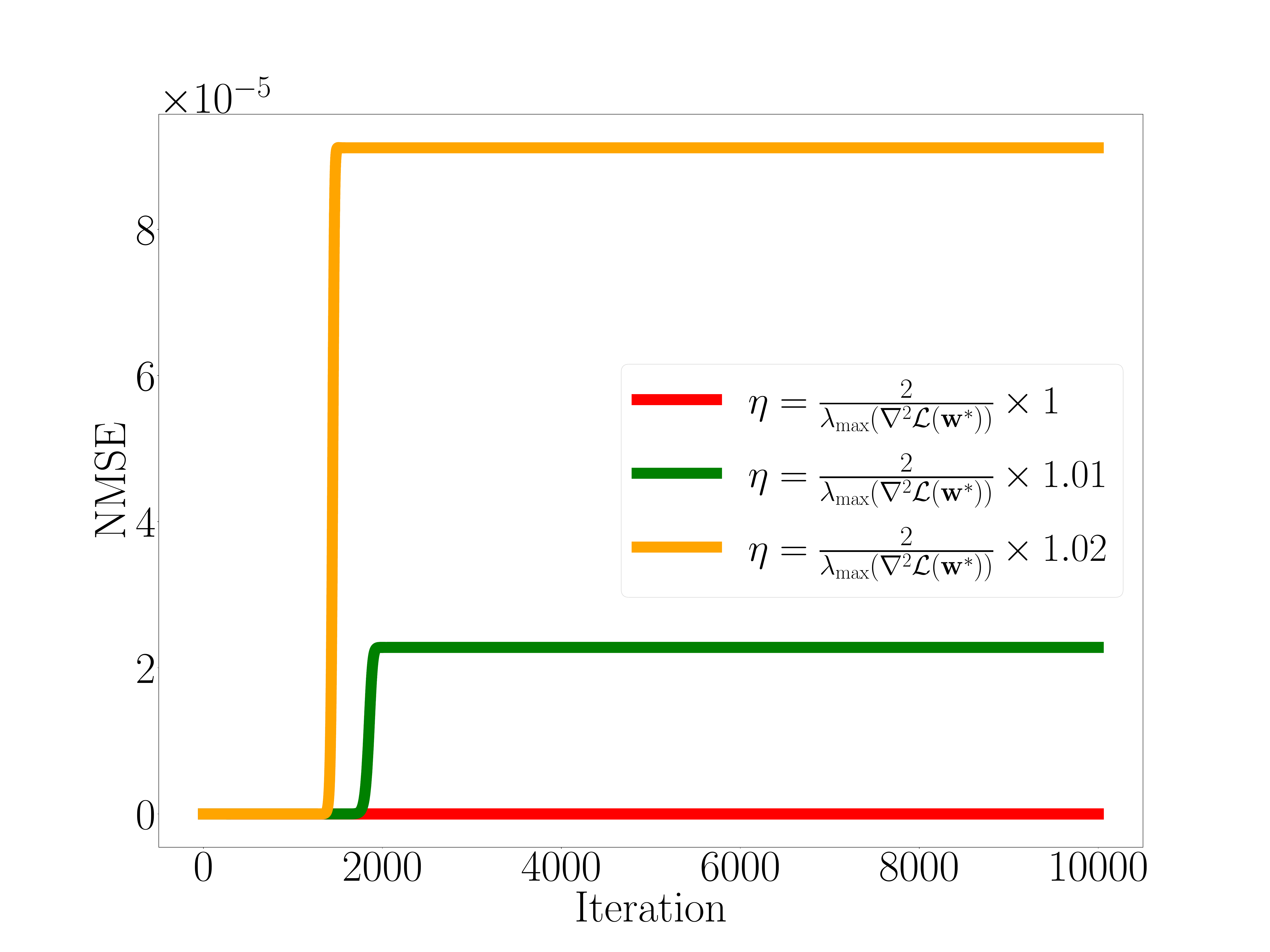}
        \caption{Normalized squared $\ell^2$ distance of $\mathbf{w}_{t}$ from the minimum $\mathbf{w}^*$.}
        \label{fig:figure2b}
    \end{subfigure}
    \hfill
    \begin{subfigure}[t]{0.32\textwidth}
        \centering
        \includegraphics[width=\textwidth]{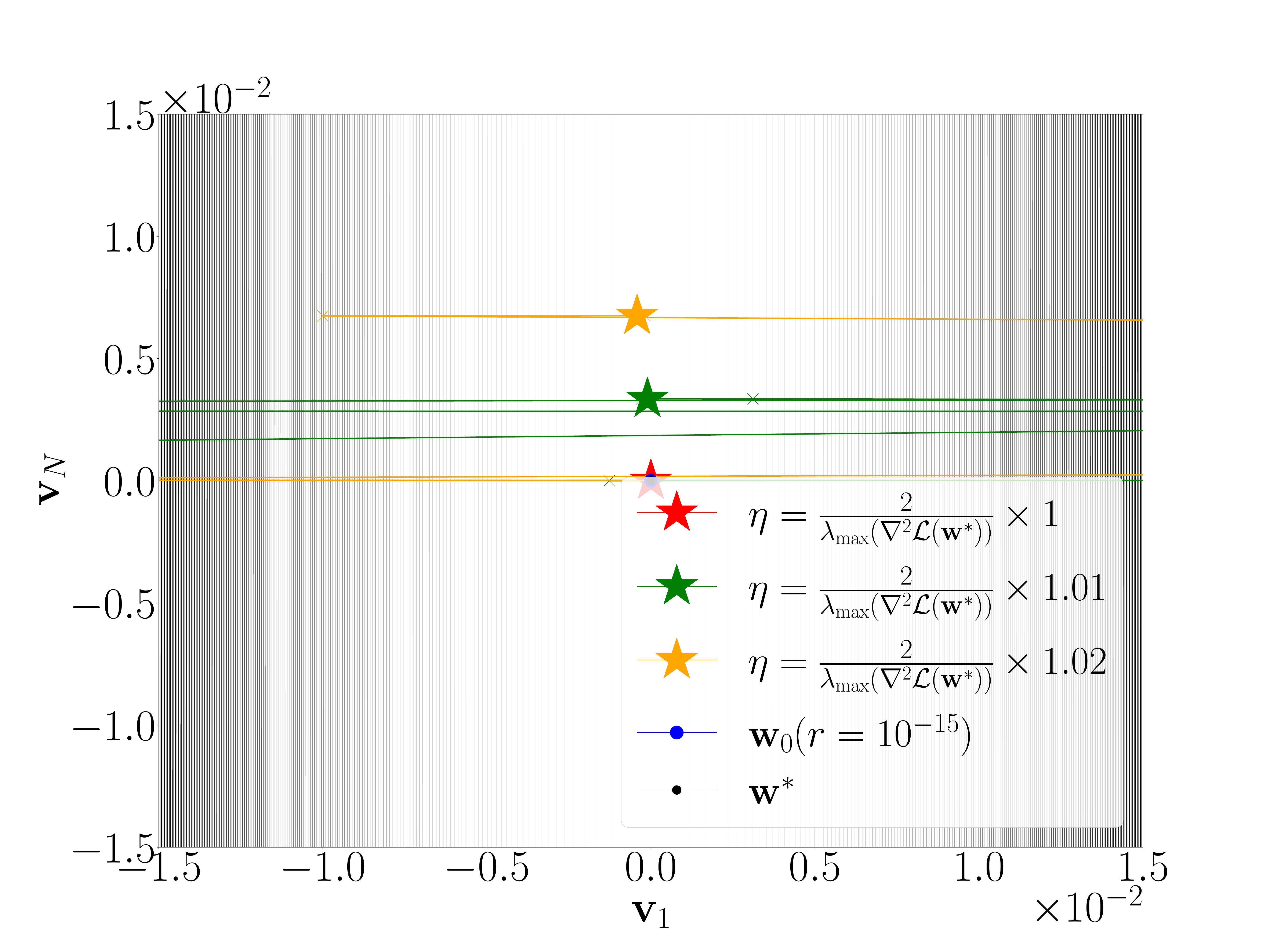}
        \caption{Trajectories of GD on the contour map of the loss landscape.}
        \label{fig:figure2c}
    \end{subfigure}

    \vspace{0.5em}

    \begin{subfigure}[t]{0.32\textwidth}
        \centering
        \includegraphics[width=\textwidth]{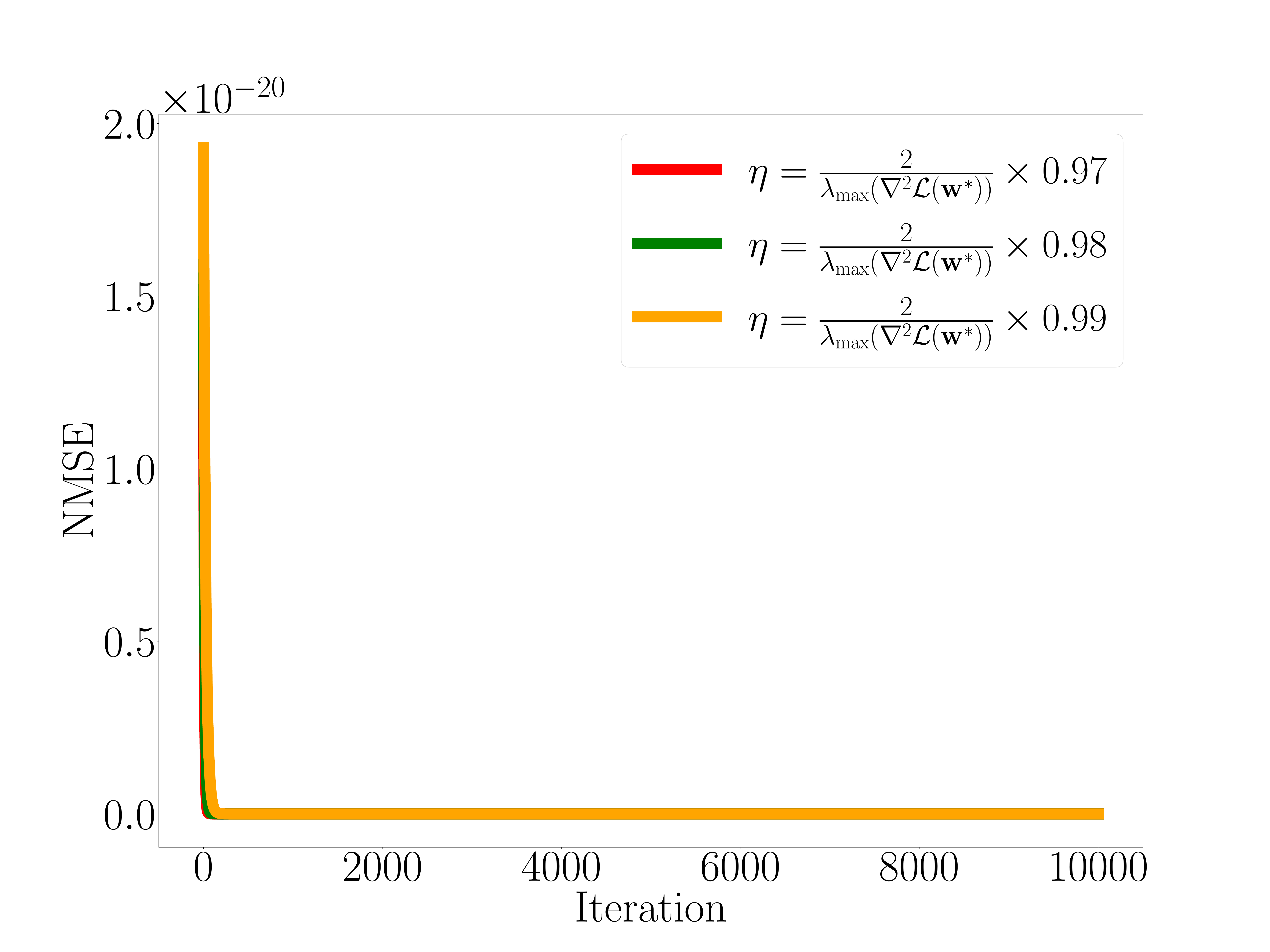}
        \caption{Normalized training error across iterations, i.e., $\mathcal{L}(\mathbf{w}_{t})/\|\mathbf{M}\|_F^2$.}
        \label{fig:figure2d}
    \end{subfigure}
    \hfill
    \begin{subfigure}[t]{0.32\textwidth}
        \centering
        \includegraphics[width=\textwidth]{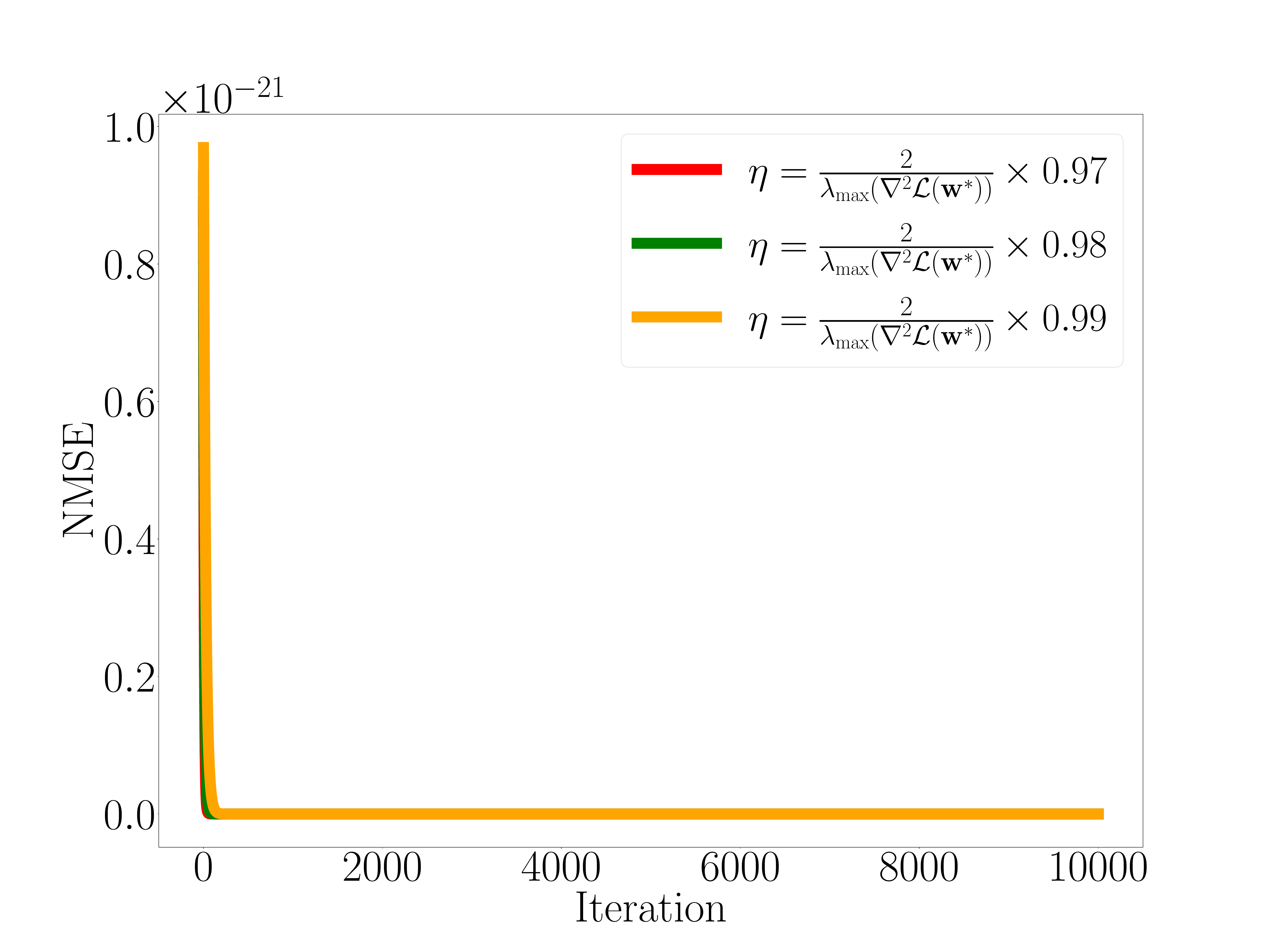}
        \caption{Normalized squared $\ell^2$ distance of $\mathbf{w}_{t}$ from the minimum $\mathbf{w}^*$.}
        \label{fig:figure2e}
    \end{subfigure}
    \hfill
    \begin{subfigure}[t]{0.32\textwidth}
        \centering
        \includegraphics[width=\textwidth]{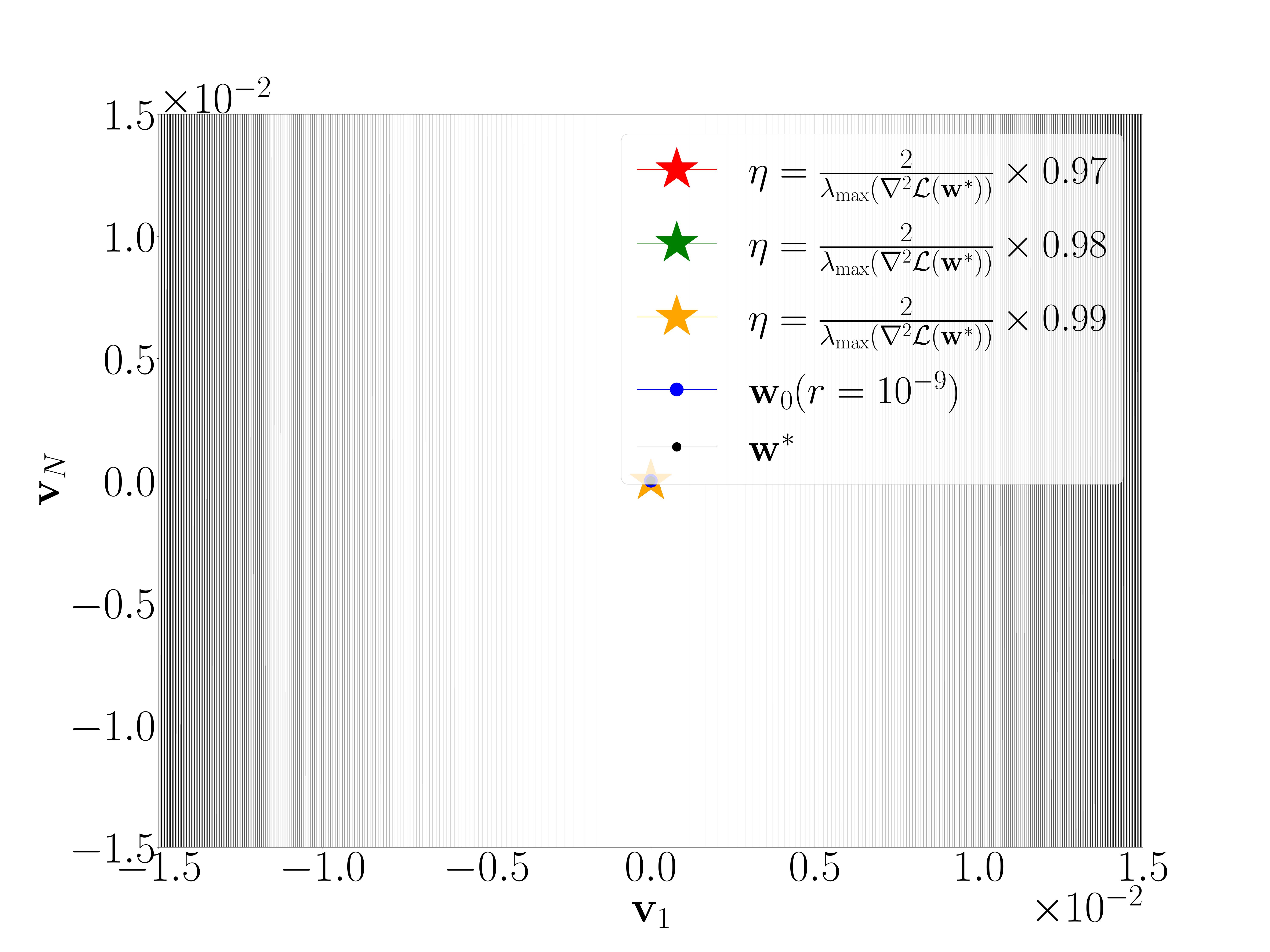}
        \caption{Trajectories of GD on the contour map of the loss landscape.}
        \label{fig:figure2f}
    \end{subfigure}
    \caption{The setting is a depth-2 matrix factorization of a random $10 \times 20$ target matrix, where $\mathbf{L} \in \mathbb{R}^{10 \times 30}$ and $\mathbf{R} \in \mathbb{R}^{20 \times 30}$ are Gaussian factors. GD dynamics with different step sizes, indicated by different colors, are initialized within a radius of $10^{-15}$ in (a)--(c) and $10^{-9}$ in (d)--(f) from the minimizer $\mathbf{w}^*$. The initialization direction is along the eigenvector of $\nabla^2 \mathcal{L}(\mathbf{w}^*)$ that corresponds to its maximum eigenvalue. The value of $\lambda_{\max}(\nabla^2 \mathcal{L}(\mathbf{w}^*))$ is computed using the closed-form expression derived in Corollary~\ref{remark:twolayer}.}
    \label{fig:figure2}
\end{figure}

\begin{figure}[!htb]
    \centering
    \begin{subfigure}[t]{0.32\textwidth}
        \centering
        \includegraphics[width=\textwidth]{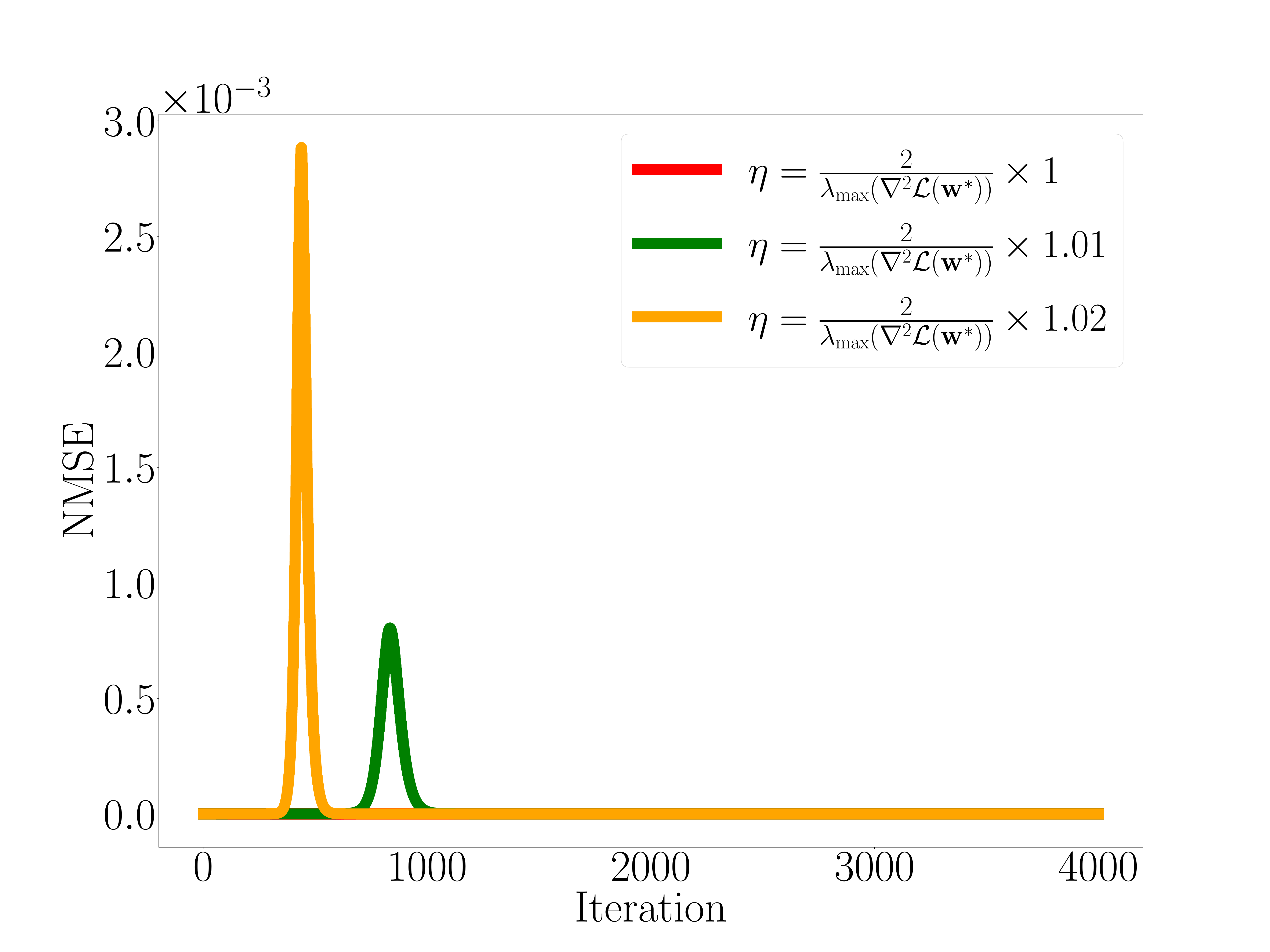}
        \caption{Normalized training error across iterations, i.e., $\mathcal{L}(\mathbf{w}_{t})/m^2$.}
        \label{fig:figure5a}
    \end{subfigure}
    \hfill
    \begin{subfigure}[t]{0.32\textwidth}
        \centering
        \includegraphics[width=\textwidth]{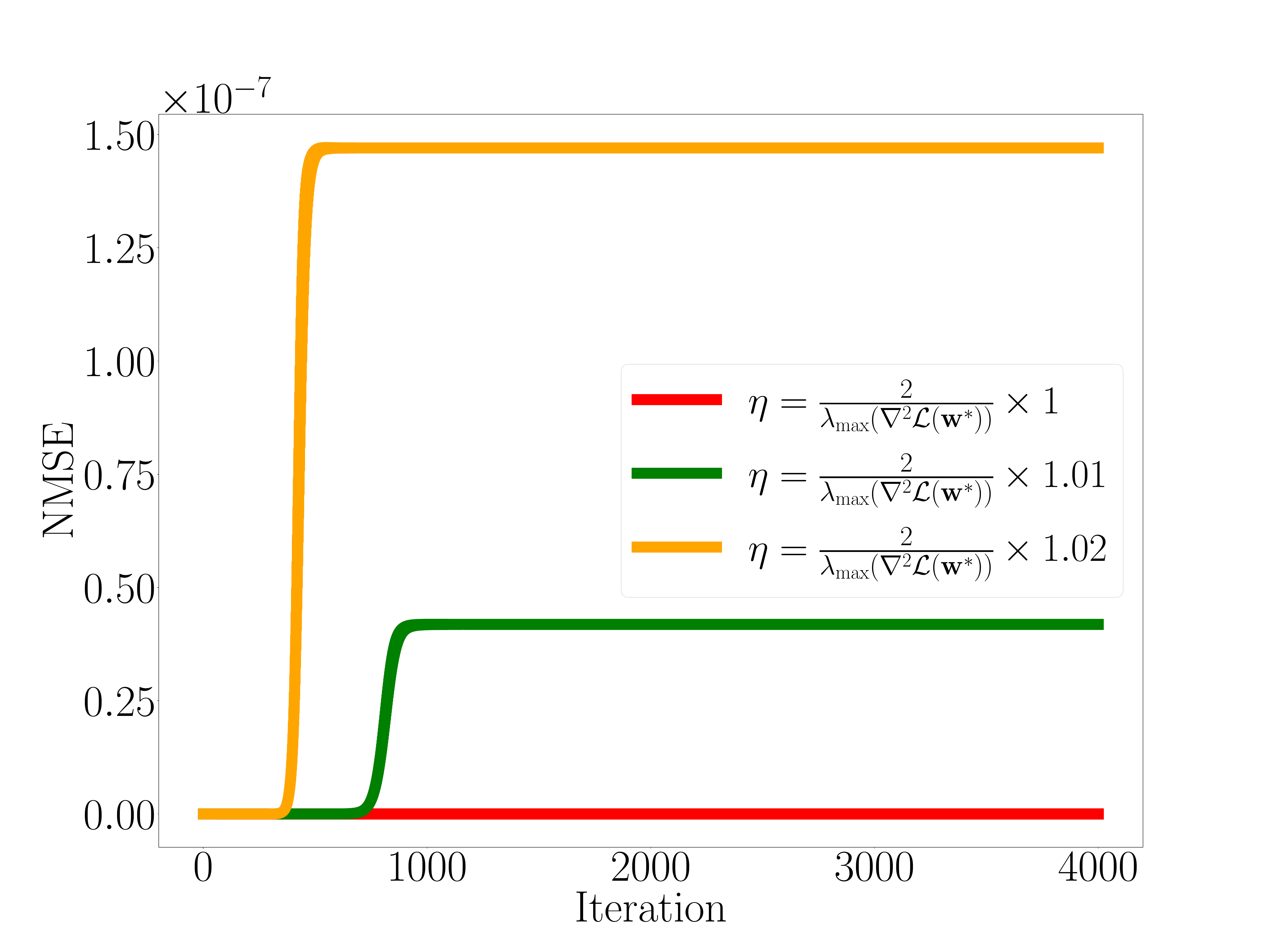}
        \caption{Normalized squared $\ell^2$ distance of $\mathbf{w}_{t}$ from the minimum $\mathbf{w}^*$.}
        \label{fig:figure5b}
    \end{subfigure}
    \hfill
    \begin{subfigure}[t]{0.32\textwidth}
        \centering
        \includegraphics[width=\textwidth]{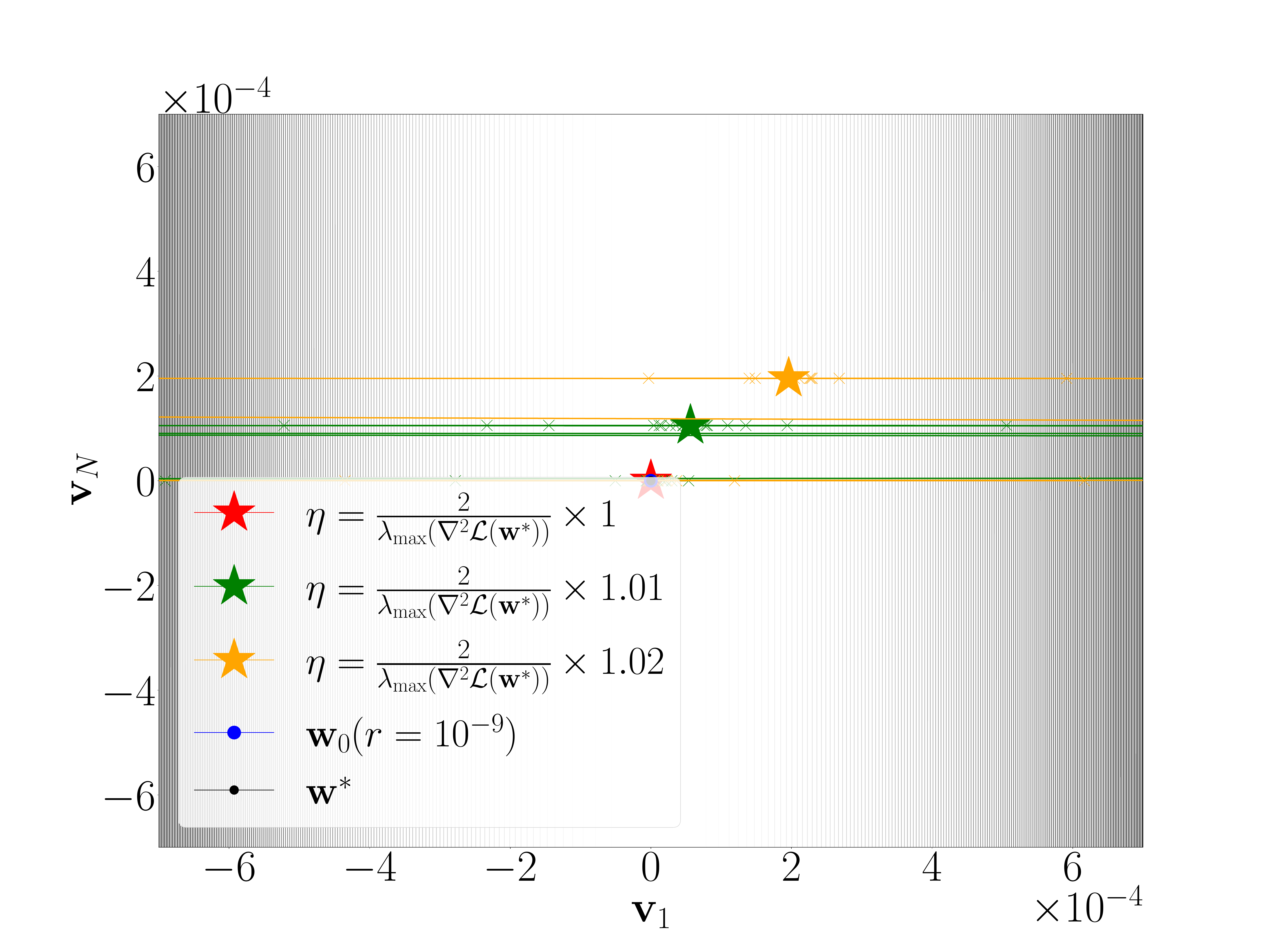}
        \caption{Trajectories of GD on the contour map of the loss landscape.}
        \label{fig:figure5c}
    \end{subfigure}

    \vspace{0.5em}

    \begin{subfigure}[t]{0.32\textwidth}
        \centering
        \includegraphics[width=\textwidth]{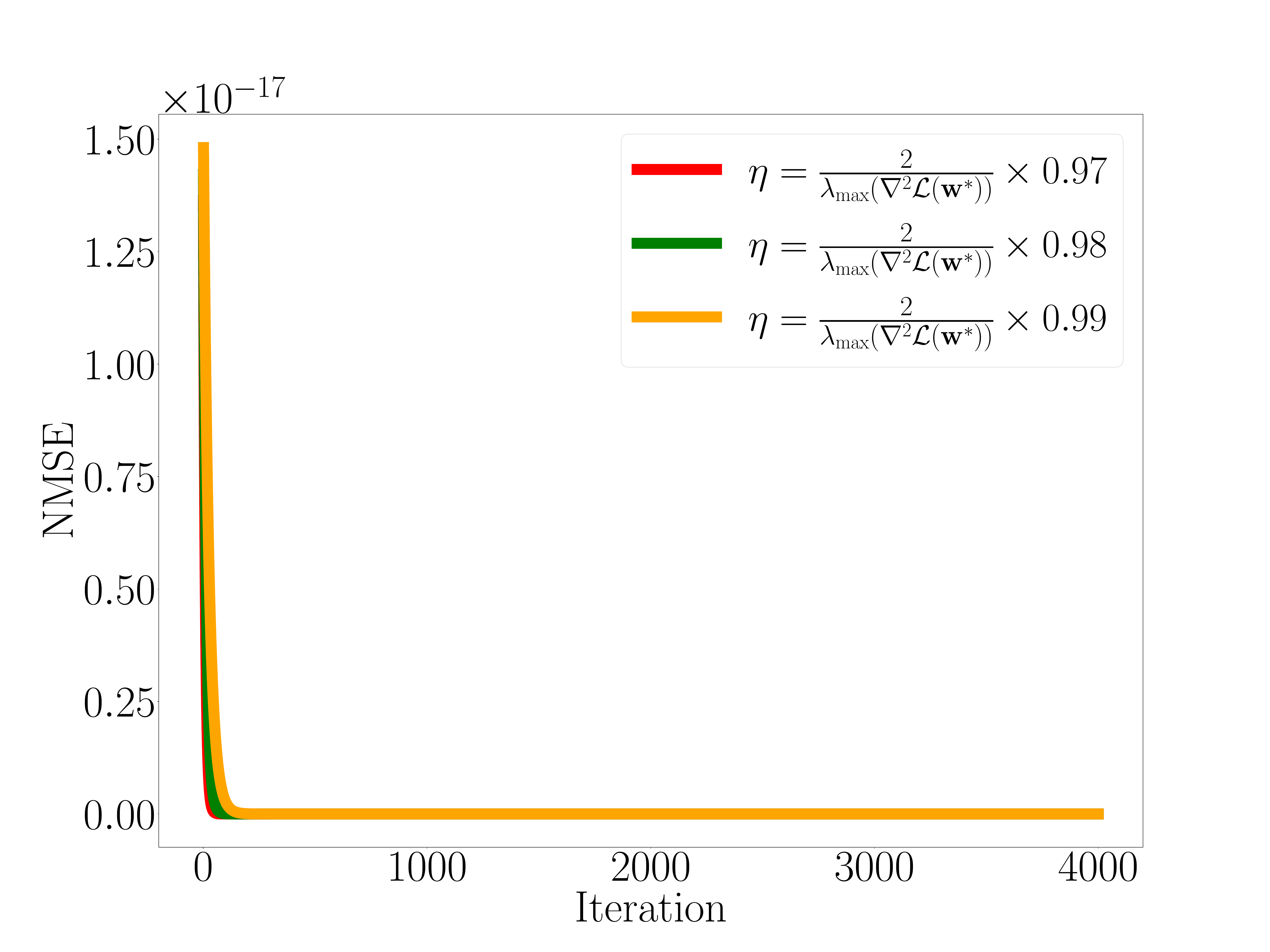}
        \caption{Normalized training error across iterations, i.e., $\mathcal{L}(\mathbf{w}_{t})/m^2$.}
        \label{fig:figure5d}
    \end{subfigure}
    \hfill
    \begin{subfigure}[t]{0.32\textwidth}
        \centering
        \includegraphics[width=\textwidth]{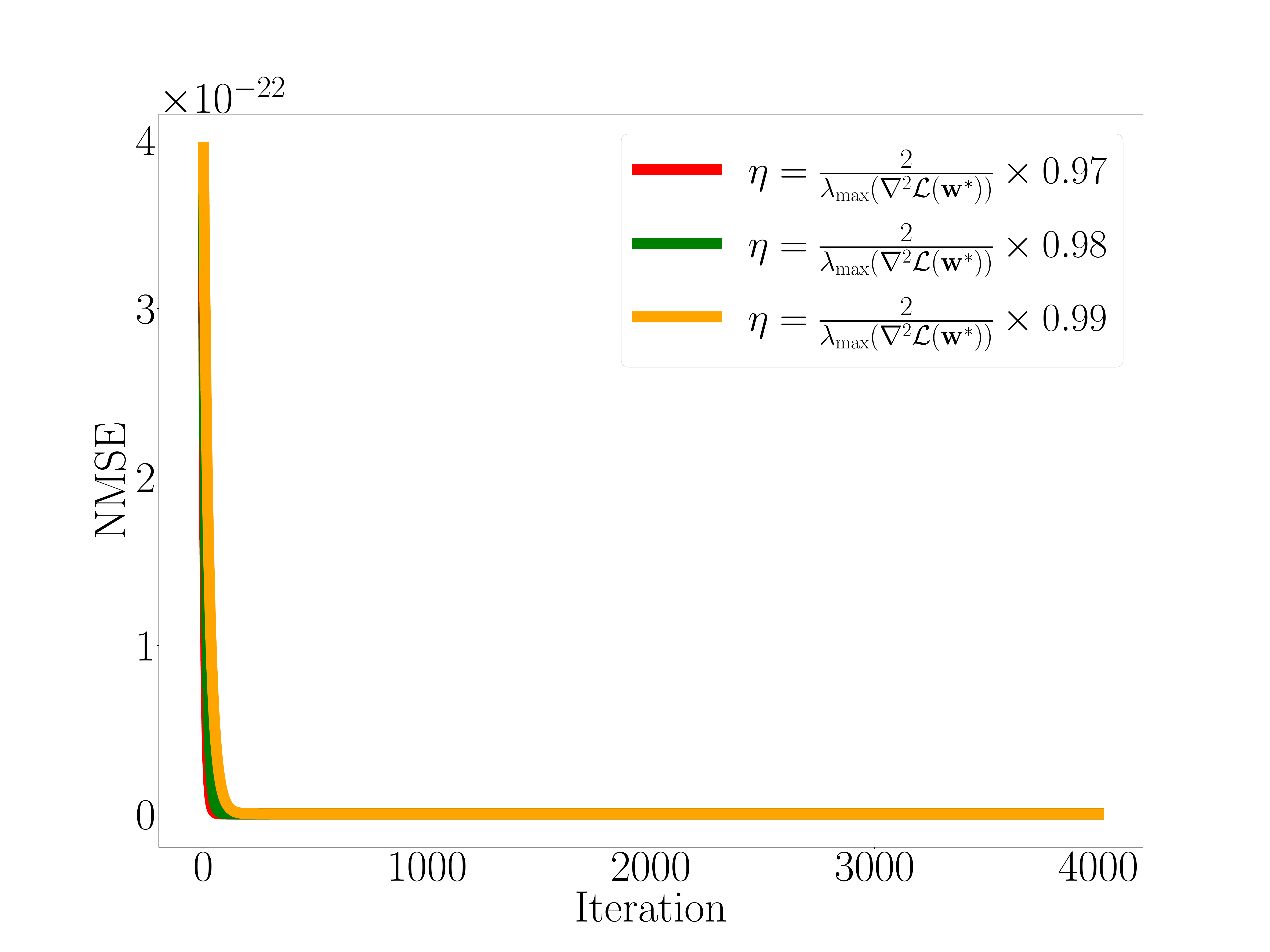}
        \caption{Normalized squared $\ell^2$ distance of $\mathbf{w}_{t}$ from the minimum $\mathbf{w}^*$.}
        \label{fig:figure5e}
    \end{subfigure}
    \hfill
    \begin{subfigure}[t]{0.32\textwidth}
        \centering
        \includegraphics[width=\textwidth]{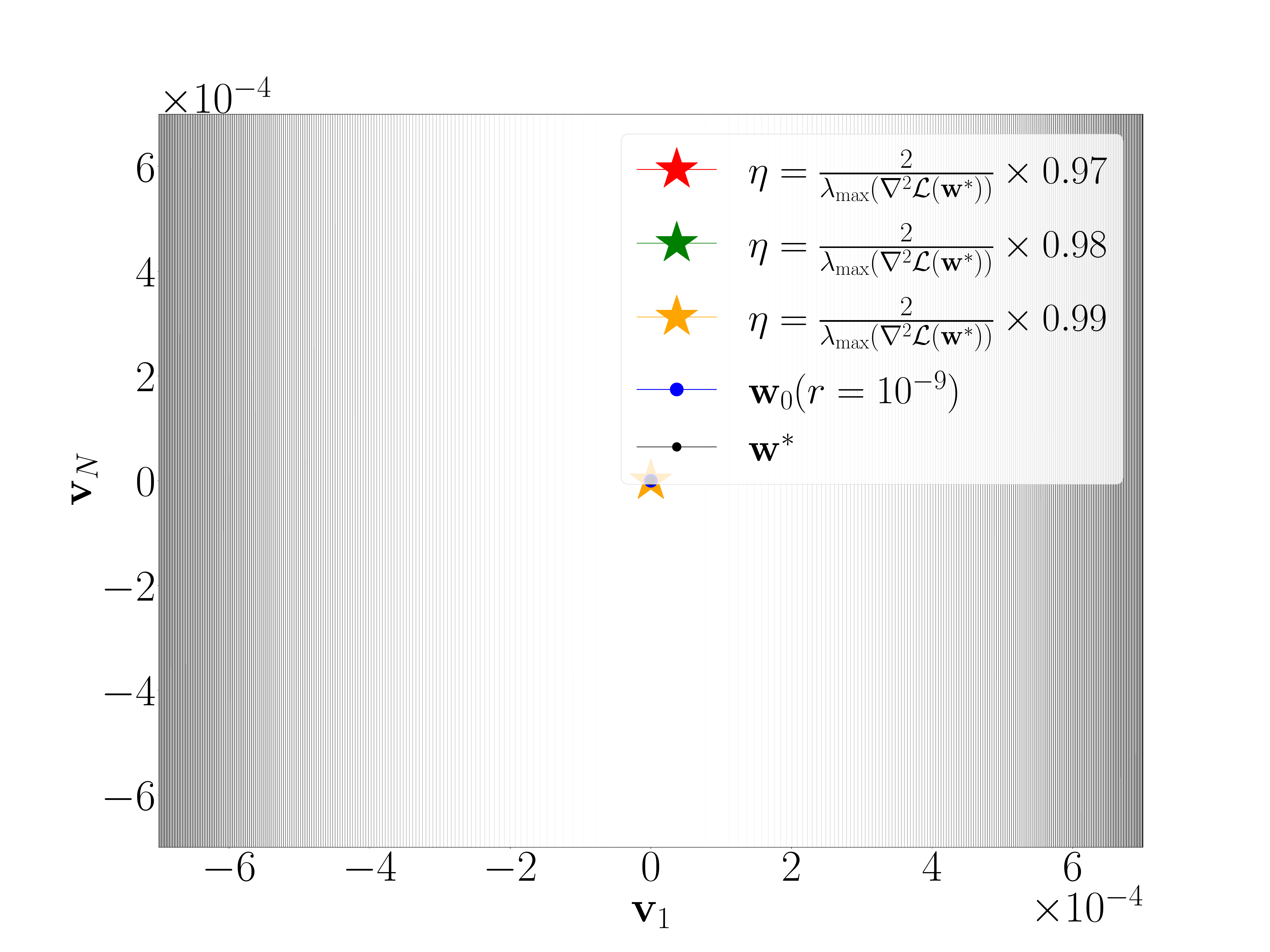}
        \caption{Trajectories of GD on the contour map of the loss landscape.}
        \label{fig:figure5f}
    \end{subfigure}
    \caption{The setting is a 20-layer overparameterized scalar factorization of a random scalar. GD dynamics with different step sizes, indicated by different colors, are initialized within a radius of $10^{-9}$ from the minimizer $\mathbf{w}^*$. The initialization direction is along the eigenvector of $\nabla^2 \mathcal{L}(\mathbf{w}^*)$ that corresponds to its maximum eigenvalue. The value of $\lambda_{\max}(\nabla^2 \mathcal{L}(\mathbf{w}^*))$ is computed using the closed-form expression derived in Corollary~\ref{cor:warmup}.}
    \label{fig:figure5}
\end{figure}

\begin{figure}[!htb]
    \centering
    \begin{subfigure}[t]{0.32\textwidth}
        \centering
        \includegraphics[width=\textwidth]{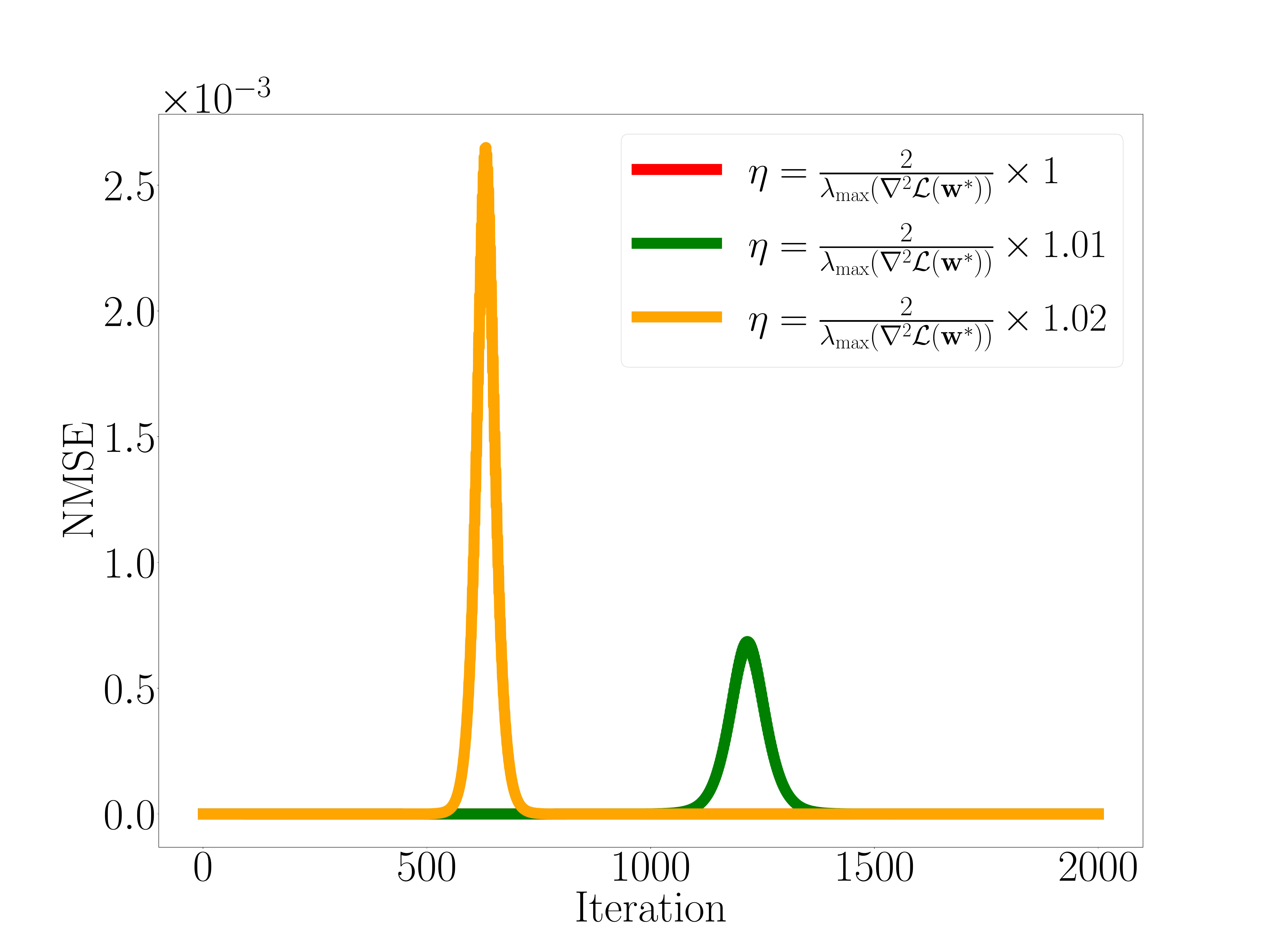}
        \caption{Normalized training error across iterations, i.e., $\mathcal{L}(\mathbf{w}_{t})/m^2$.}
        \label{fig:figure3a}
    \end{subfigure}
    \hfill
    \begin{subfigure}[t]{0.32\textwidth}
        \centering
        \includegraphics[width=\textwidth]{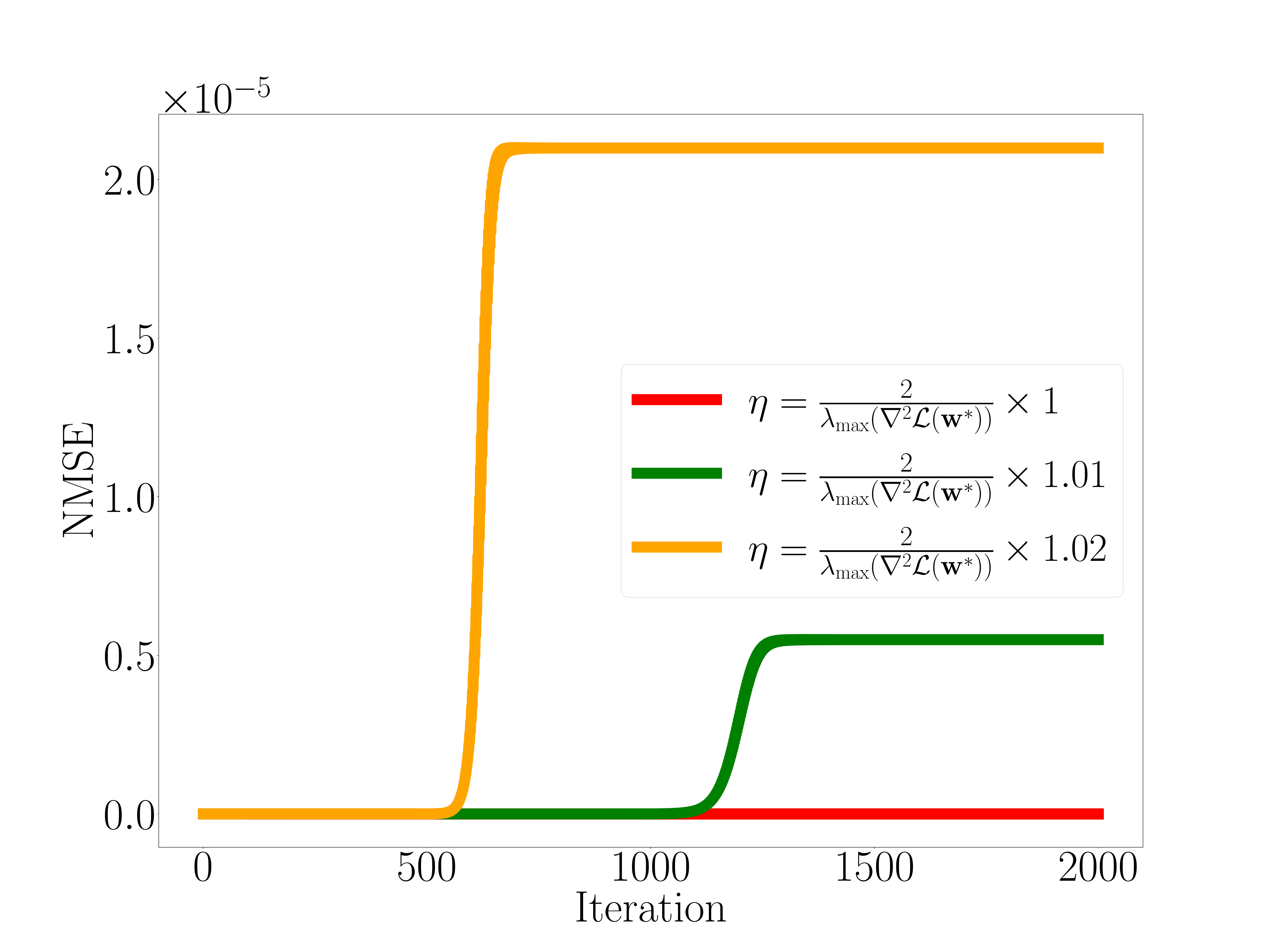}
        \caption{Normalized squared $\ell^2$ distance of $\mathbf{w}_{t}$ from the minimum $\mathbf{w}^*$.}
        \label{fig:figure3b}
    \end{subfigure}
    \hfill
    \begin{subfigure}[t]{0.32\textwidth}
        \centering
        \includegraphics[width=\textwidth]{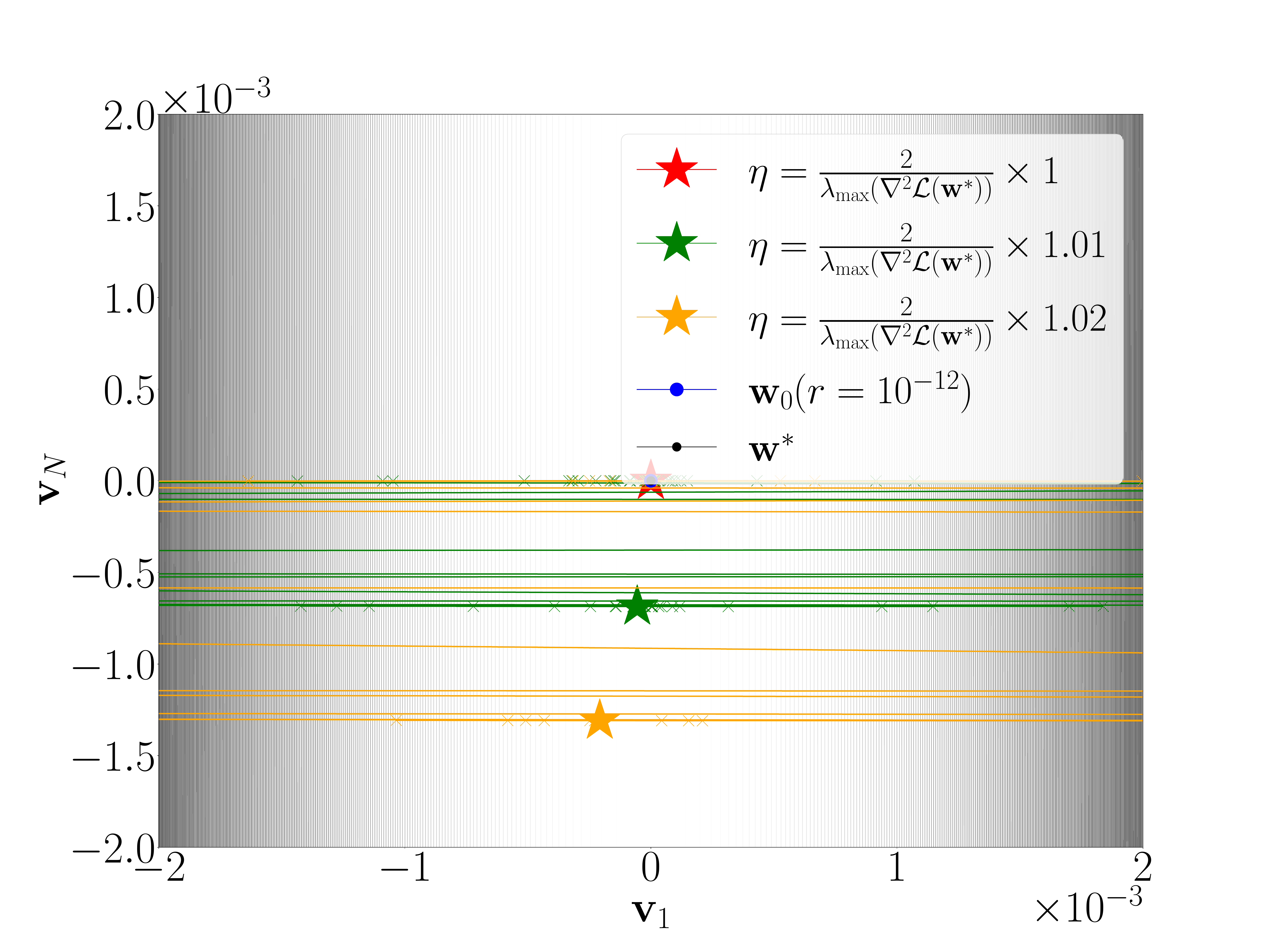}
        \caption{Trajectories of GD on the contour map of the loss landscape.}
        \label{fig:figure3c}
    \end{subfigure}

    \vspace{0.5em}

    \begin{subfigure}[t]{0.32\textwidth}
        \centering
        \includegraphics[width=\textwidth]{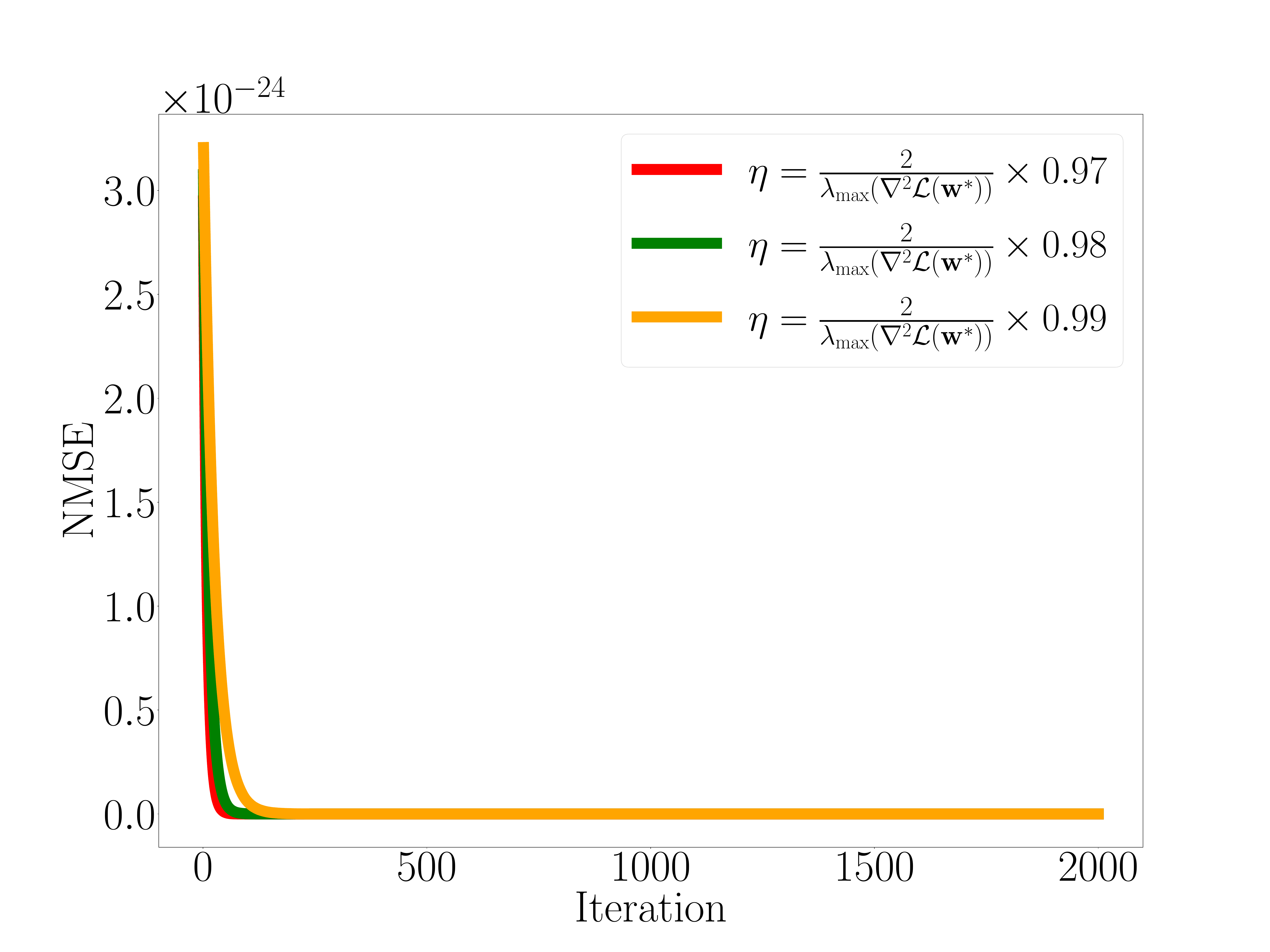}
        \caption{Normalized training error across iterations, i.e., $\mathcal{L}(\mathbf{w}_{t})/m^2$.}
        \label{fig:figure3d}
    \end{subfigure}
    \hfill
    \begin{subfigure}[t]{0.32\textwidth}
        \centering
        \includegraphics[width=\textwidth]{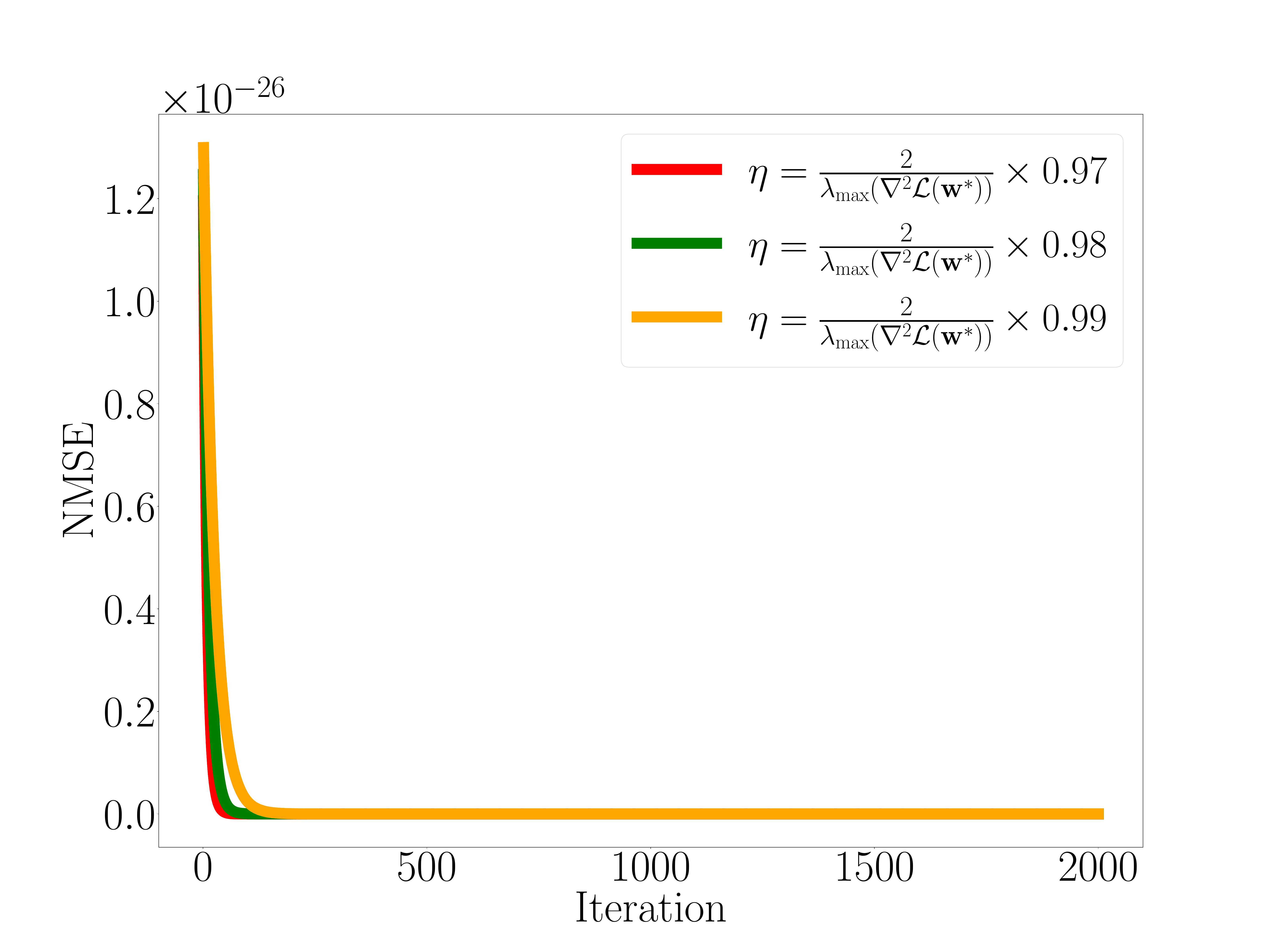}
        \caption{Normalized squared $\ell^2$ distance of $\mathbf{w}_{t}$ from the minimum $\mathbf{w}^*$.}
        \label{fig:figure3e}
    \end{subfigure}
    \hfill
    \begin{subfigure}[t]{0.32\textwidth}
        \centering
        \includegraphics[width=\textwidth]{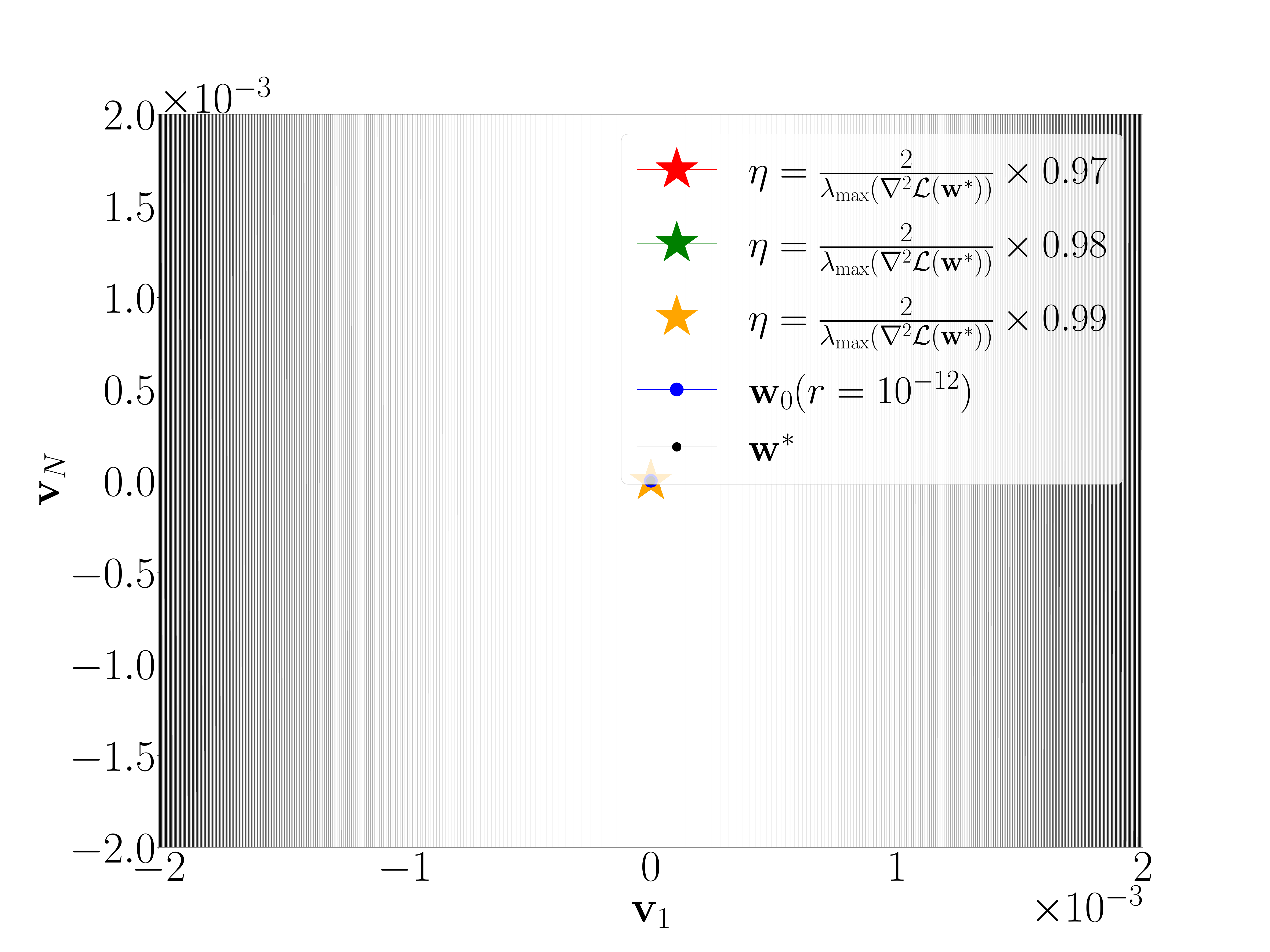}
        \caption{Trajectories of GD on the contour map of the loss landscape.}
        \label{fig:figure3f}
    \end{subfigure}
    \caption{The setting is a 4-layer overparameterized scalar factorization of a random scalar. GD dynamics with different step sizes, indicated by different colors, are initialized within a radius of $10^{-12}$ from the minimizer $\mathbf{w}^*$. The initialization direction is along the eigenvector of $\nabla^2 \mathcal{L}(\mathbf{w}^*)$ that corresponds to its maximum eigenvalue. The value of $\lambda_{\max}(\nabla^2 \mathcal{L}(\mathbf{w}^*))$ is computed using the closed-form expression derived in Corollary~\ref{cor:warmup}.}
    \label{fig:figure3}
\end{figure}

\begin{figure}[!htb]
    \centering
    \begin{subfigure}[t]{0.32\textwidth}
        \centering
        \includegraphics[width=\textwidth]{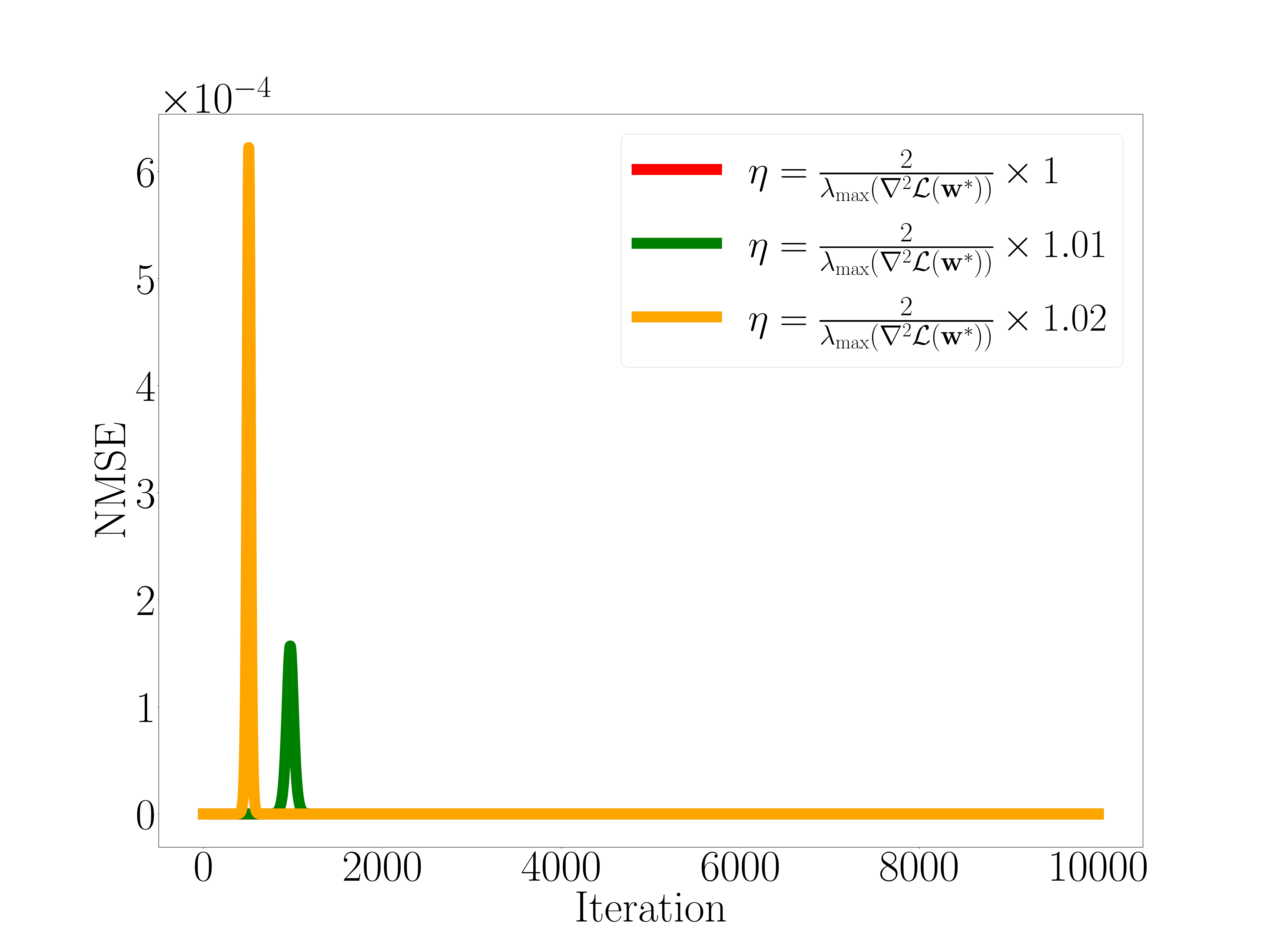}
        \caption{Normalized training error across iterations, i.e., $\mathcal{L}(\mathbf{w}_{t})/\|\mathbf{M}\|_F^2$.}
        \label{fig:figure4a}
    \end{subfigure}
    \hfill
    \begin{subfigure}[t]{0.32\textwidth}
        \centering
        \includegraphics[width=\textwidth]{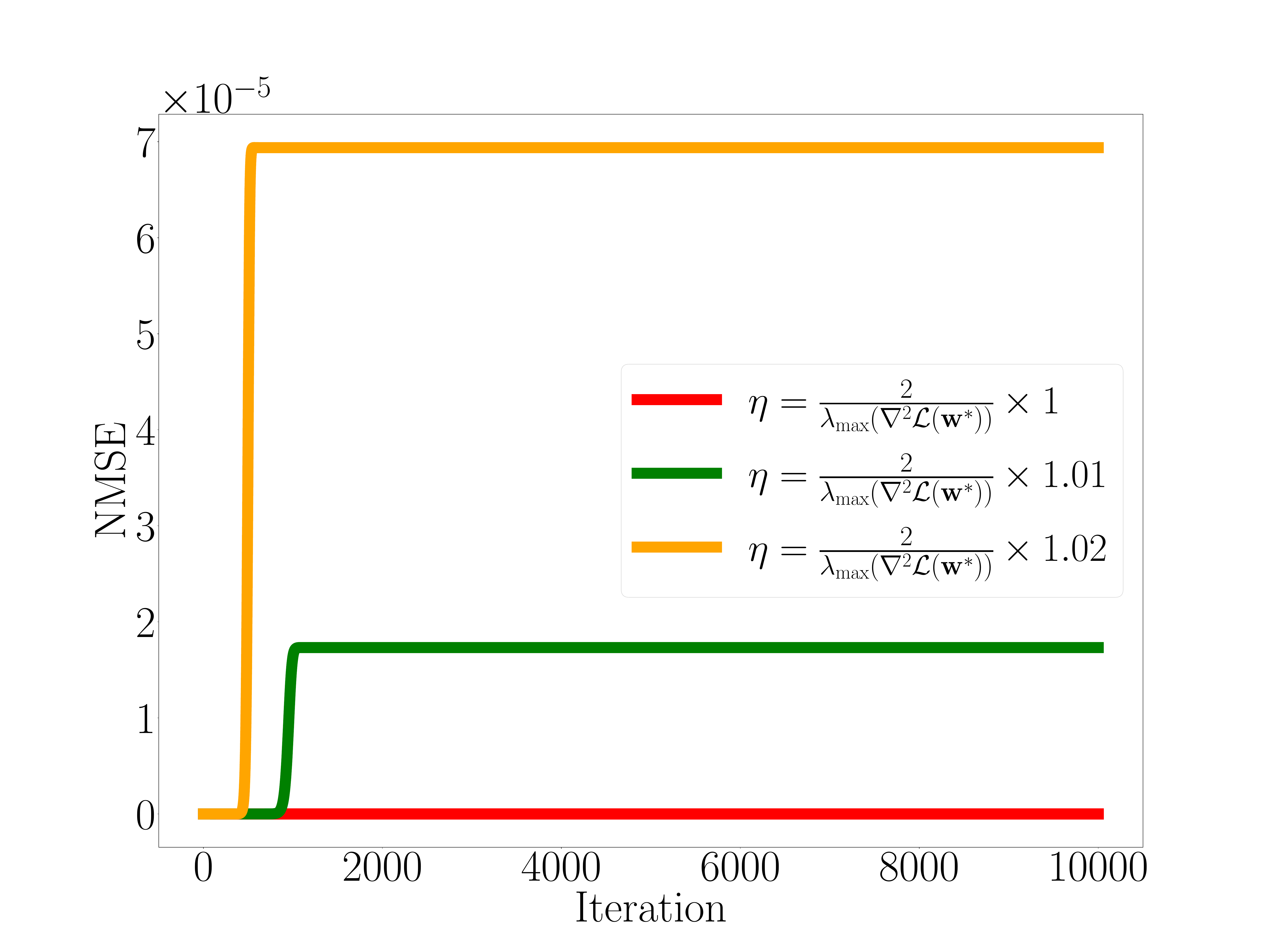}
        \caption{Normalized squared $\ell^2$ distance of $\mathbf{w}_{t}$ from the minimum $\mathbf{w}^*$.}
        \label{fig:figure4b}
    \end{subfigure}
    \hfill
    \begin{subfigure}[t]{0.32\textwidth}
        \centering
        \includegraphics[width=\textwidth]{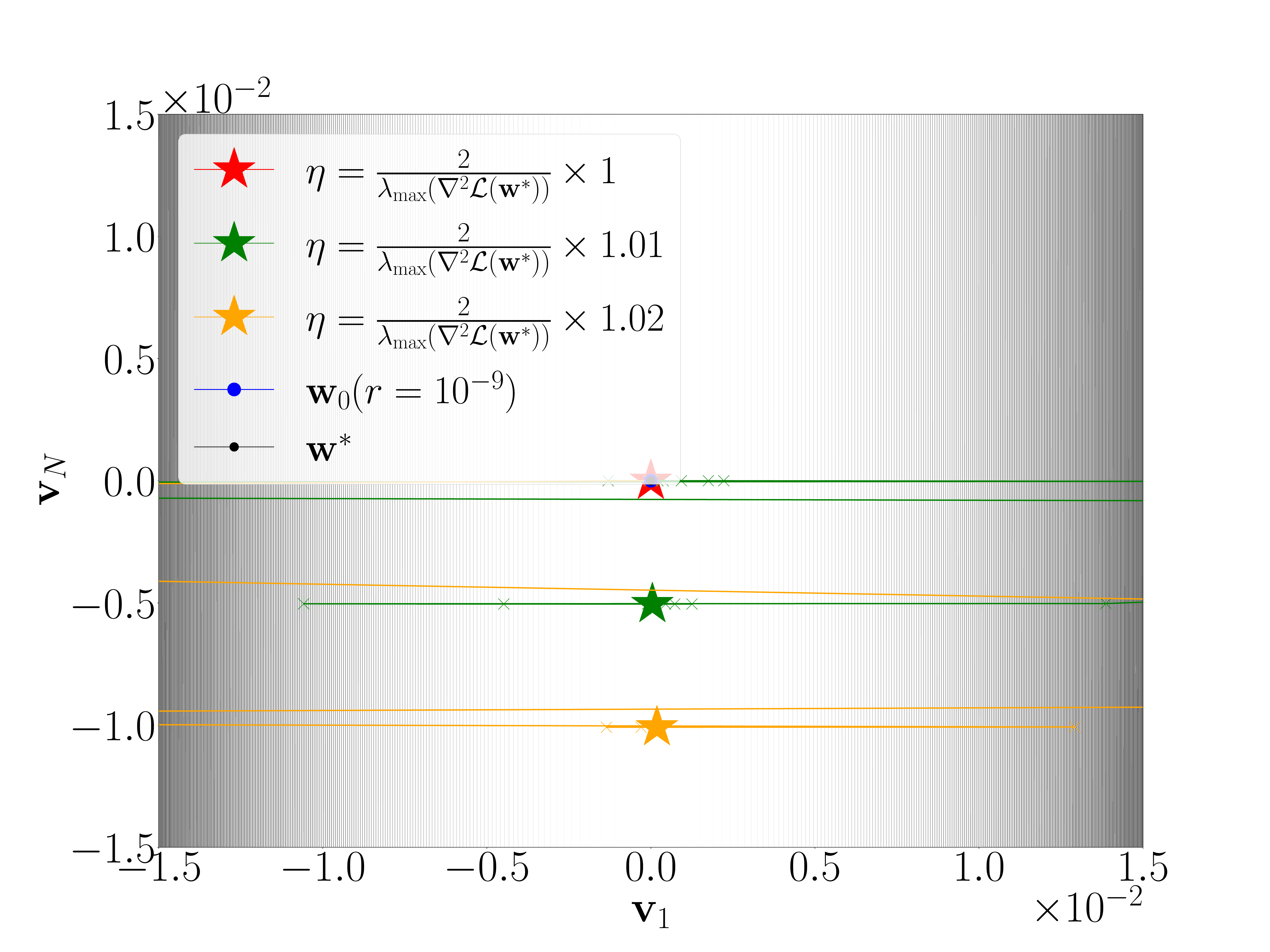}
        \caption{Trajectories of GD on the contour map of the loss landscape.}
        \label{fig:figure4c}
    \end{subfigure}

    \vspace{0.5em}

    \begin{subfigure}[t]{0.32\textwidth}
        \centering
        \includegraphics[width=\textwidth]{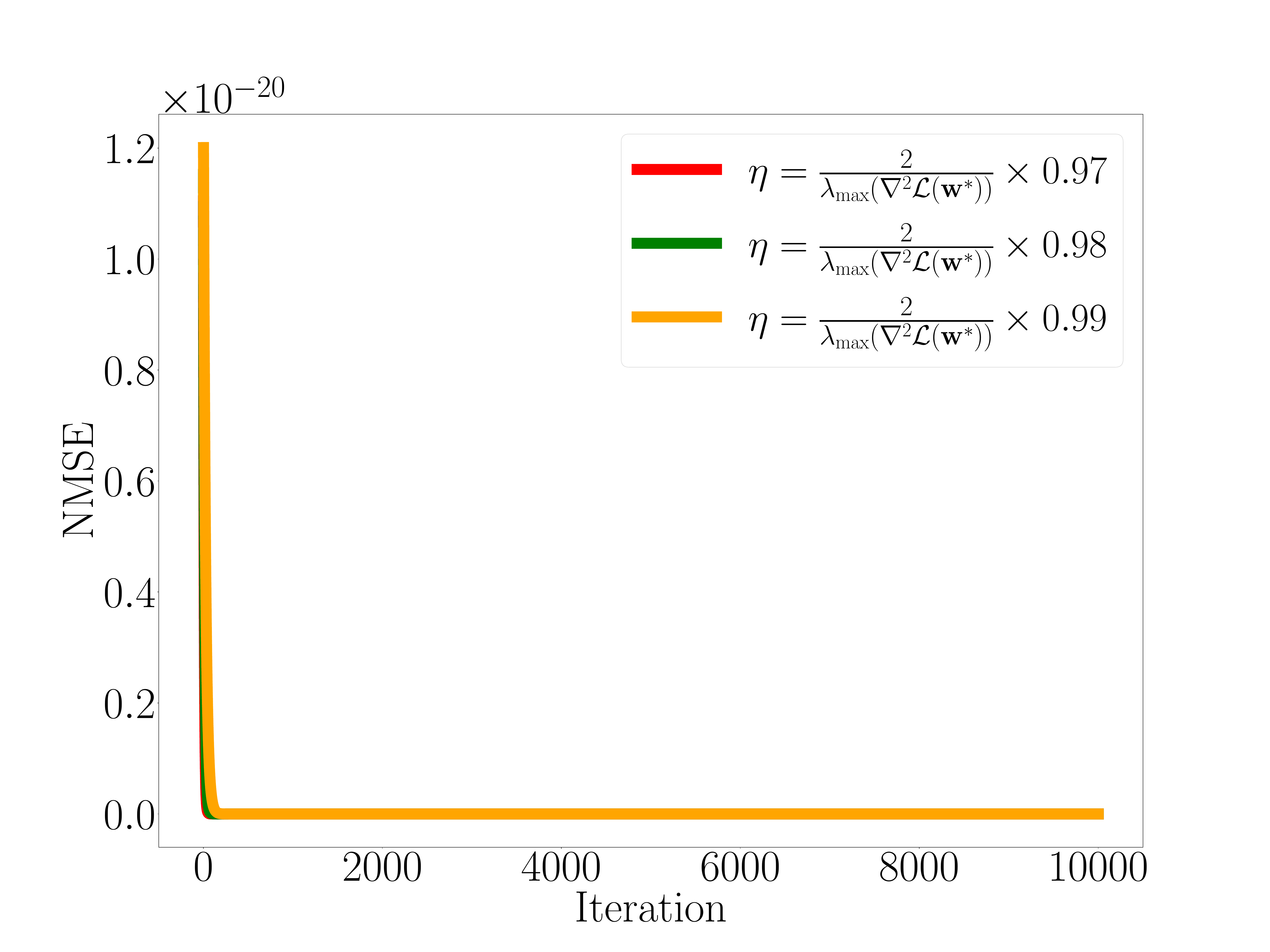}
        \caption{Normalized training error across iterations, i.e., $\mathcal{L}(\mathbf{w}_{t})/\|\mathbf{M}\|_F^2$.}
        \label{fig:figure4d}
    \end{subfigure}
    \hfill
    \begin{subfigure}[t]{0.32\textwidth}
        \centering
        \includegraphics[width=\textwidth]{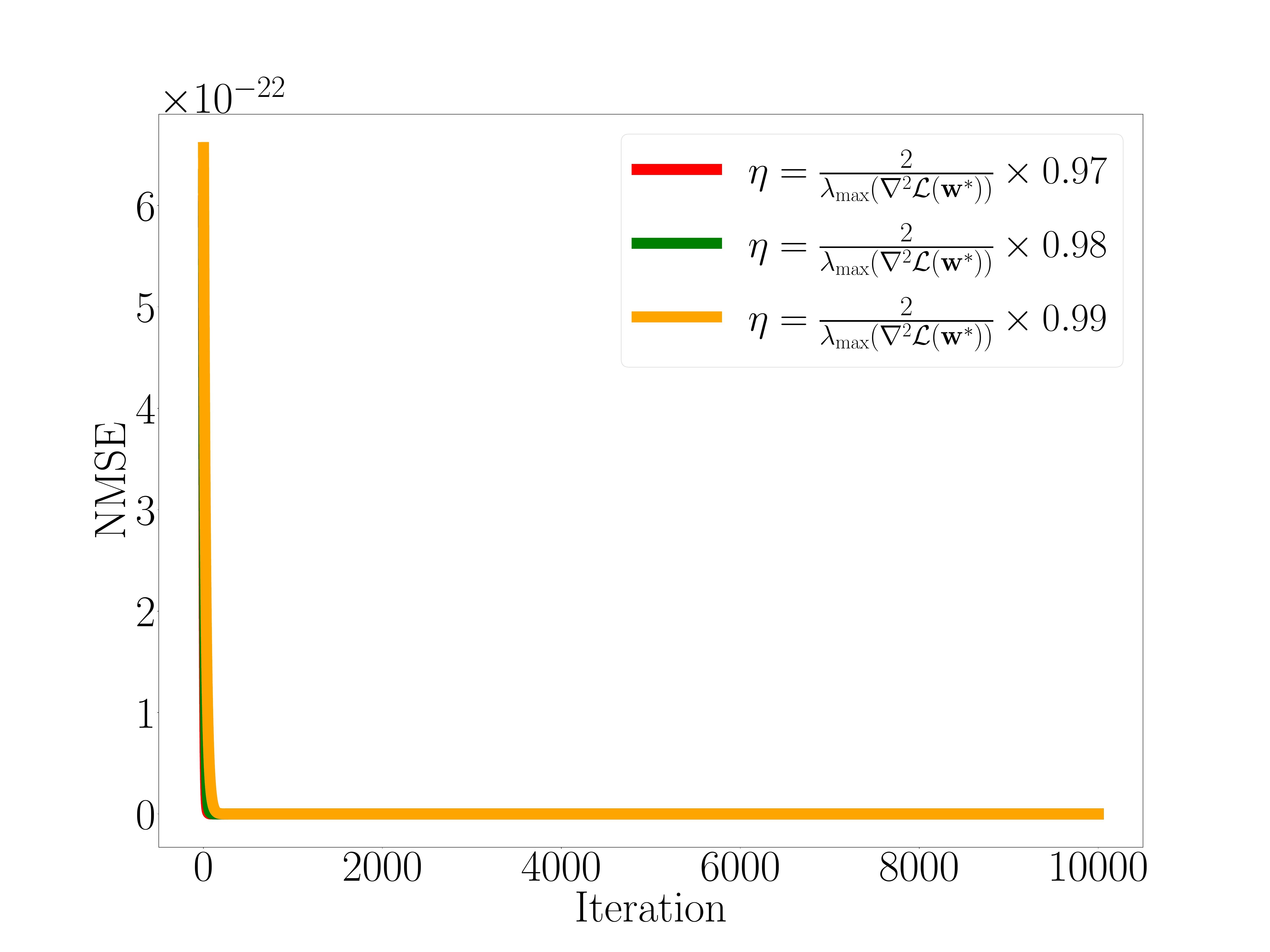}
        \caption{Normalized squared $\ell^2$ distance of $\mathbf{w}_{t}$ from the minimum $\mathbf{w}^*$.}
        \label{fig:figure4e}
    \end{subfigure}
    \hfill
    \begin{subfigure}[t]{0.32\textwidth}
        \centering
        \includegraphics[width=\textwidth]{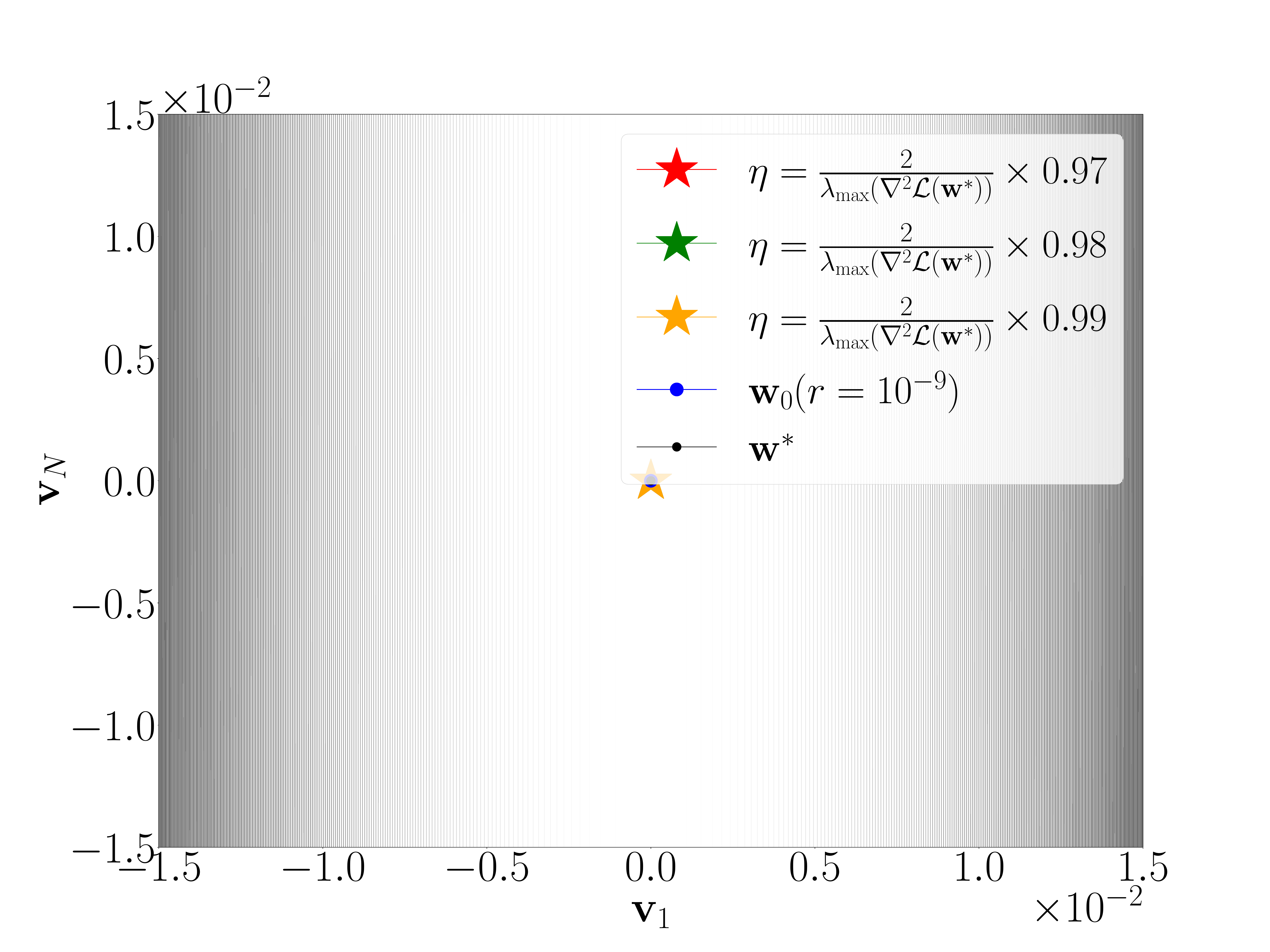}
        \caption{Trajectories of GD on the contour map of the loss landscape.}
        \label{fig:figure4f}
    \end{subfigure}
    \caption{The setting is a depth-2 matrix factorization of a random $25 \times 20$ target matrix, where $\mathbf{L} \in \mathbb{R}^{25 \times 30}$ and $\mathbf{R} \in \mathbb{R}^{20 \times 30}$ are Gaussian factors. GD dynamics with different step sizes, indicated by different colors, are initialized within a radius of $10^{-9}$ from the minimizer $\mathbf{w}^*$. The initialization direction is along the eigenvector of $\nabla^2 \mathcal{L}(\mathbf{w}^*)$ that corresponds to its maximum eigenvalue. The value of $\lambda_{\max}(\nabla^2 \mathcal{L}(\mathbf{w}^*))$ is computed using the closed-form expression derived in Corollary~\ref{remark:twolayer}.}
    \label{fig:figure4}
\end{figure}


\end{document}